\documentclass{article}

\PassOptionsToPackage{numbers, compress, sort}{natbib}
\usepackage[preprint]{neurips_2026}

\usepackage[utf8]{inputenc} 
\usepackage[T1]{fontenc}    
\usepackage[backref]{hyperref}       
\usepackage{url}            
\usepackage{enumitem}
\usepackage{booktabs}       
\usepackage{multirow}
\usepackage{amsfonts}       
\usepackage{nicefrac}       
\usepackage{microtype}      
\usepackage[dvipsnames]{xcolor}         
\hypersetup{
    colorlinks = true,
    linkcolor  = BrickRed,
    urlcolor   = teal,
    citecolor  = RoyalBlue
}
\usepackage{amsmath, amsfonts, bm, amsthm, amssymb, relsize, bbm, graphicx}
\usepackage[capitalize]{cleveref}
\usepackage{algorithm, algorithmic}
\usepackage{cancel}
\usepackage{subcaption}
\usepackage{wrapfig}

\def\d{{\mathrm{d}}}
\newcommand{\bb}[1]{\mathbb{#1}}
\newcommand{\x}[1]{\boldsymbol{#1}}

\newcommand{\cond}[1]{\math}
\def\id{{\bb{I}}}
\def\R{{\bb{R}}}
\def\E{{\bb{E}}}
\def\d{{\mathrm{d}}}

\def\xx{{\x{x}}}

\def\ttheta{{\x{\theta}}}

\def\L{{\mathcal{L}}}
\def\O{{\mathcal{O}}}
\def\Z{{\mathcal{Z}}}
\def\N{{\mathcal{N}}}
\def\K{{\mathcal{K}}}
\def\I{{\mathcal{I}}}
\def\J{{\mathcal{J}}}
\def\C{{\mathcal{C}}}
\def\Q{{\mathcal{Q}}}
\def\T{{\mathcal{T}}}
\def\H{{\mathcal{H}}}
\def\y{{\x{y}}}
\def\1{{\mathbbm{1}}}

\def\d{{\mathrm{d}}}

\newcommand{\abs}[1]{\left\vert#1\right\vert}
\newcommand{\norm}[1]{\left\Vert#1\right\Vert}
\DeclareMathOperator{\tr}{Tr}

\providecommand{\rank}{\operatorname{rank}}
\providecommand{\Cov}{\operatorname{Cov}}
\providecommand{\intr}{\operatorname{int}}

\providecommand{\asto}{\xrightarrow{\ \mathrm{a.s.}\ }}
\providecommand{\dto}{\xrightarrow{\ d\ }}

\providecommand{\diag}{\operatorname{diag}}

\theoremstyle{plain}
\newtheorem{theorem}{Theorem}[section]
\newtheorem{lemma}{Lemma}[section]
\newtheorem{proposition}{Proposition}[section]
\newtheorem{corollary}{Corollary}[section]

\theoremstyle{definition}

\theoremstyle{remark}
\newtheorem{remark}{Remark}[section]

\title{Generalised Score Matching on Convex Domains}

\author{%
  Nishanth Shetty\thanks{Equal contribution.} \qquad
  Saisuchith Mahajan\footnotemark[1] \qquad
  Chandra Sekhar Seelamantula \\
  Department of Electrical Engineering\\
  Indian Institute of Science, Bengaluru 560012 \\
  \texttt{\{nishanths, saisuchithm, css\}@iisc.ac.in}
}

\begin{document}
\maketitle

\begin{abstract}
    Score matching avoids computing the normalising constant that maximum-likelihood estimation requires. On constrained domains, its generalised variants weight the Fisher divergence so that boundary terms vanish. We derive generalised score matching on open convex subsets of $\mathbb{R}^{d}$ as the small-neighbourhood limit of minimum probability flow, in which the geometry of the neighbourhoods determines the weight. Every $C^{2}$ positive definite weight arises in this way, including those of classical score matching on $\mathbb{R}^{d}$ and of its variants for non-negative data on $\mathbb{R}_{+}^{d}$. For exponential families, we extend the standard convexity, consistency and asymptotic normality results to every such weight and show that the estimator converges to the true parameter under certain boundary conditions. For a truncated Gaussian on a polytope and a Dirichlet distribution on the simplex, proposed estimators attain the lowest median error of all methods compared, in at least 42 of 50 ground-truth configurations.
\end{abstract}

\section{Introduction} \label{sec:intro}
Learning unnormalised probabilistic models is a central challenge in modern machine learning and statistics. Several expressive model architectures, including energy-based models (EBMs)~\citep{doi:10.1073/pnas.79.8.2554, cd, pmlr-v9-gutmann10a, 10.5555/3104482.3104568, NEURIPS2019_378a063b, Grathwohl2020Your, Nijkamp_Hill_Han_Zhu_Wu_2020}, restricted Boltzmann machines (RBMs)~\citep{PhysRevLett.35.1792, DBNs, RBMs, hinton2012}, and Markov random fields (MRFs)~\citep{MRF, MRF-2, MRF-4, MRF-3}, are specified by unnormalised densities. A variety of partition-function-free estimation methods have been developed over the past several years that allow for parameter estimation without the need for explicit normalisation.

\emph{Score matching} \citep{Hyvarinen05a} is one such technique that circumvents the intractability of estimating the partition function by fitting the score function, defined as the gradient of the log density. Through integration by parts, the objective can be formulated solely in terms of the derivatives of the model density, independent of the normalising constant and the score function of the target data density. While score matching was originally formulated for densities supported on $\mathbb{R}^{d}$~\citep{Hyvarinen05a}, subsequent work~\citep{Hyvaerinen2007, Lyu_score_09, generalized_score, gsm_general_domain, truncated_sm, gsm-on-compositional-data} has extended the framework to distributions defined on subsets of $\mathbb{R}^{d}$.

A direct application of the original score matching formulation is not straightforward in such settings, because the boundary term produced by integration by parts need not vanish when the density or its derivatives are discontinuous at the boundary of the support. Existing extensions address this by weighting the Fisher divergence with a function that vanishes on the boundary, designed for specific domains such as $\R^{d}_{+}$~\citep{Hyvaerinen2007, generalized_score}, truncated sets~\citep{truncated_sm}, general domains~\citep{gsm_general_domain} and the simplex~\citep{gsm-on-compositional-data}. These objectives fit within the operator-theoretic framework of \citet{Lyu_score_09}, and we review them in Section \ref{sec:related-work}. In these works, the weighted Fisher divergence is the starting point: a weight is postulated for the domain at hand, and the tractable objective follows by integration by parts once the weight is shown to suppress the boundary terms. This raises the question of whether these objectives can be derived from a more unifying estimation principle, one that also explains where the weight comes from.

\subsection{Contributions}
We address this question for open convex domains. Our main contributions are the following.
\begin{enumerate}[leftmargin=*]
    \item \textbf{A constructive derivation.} \citet{mpf} showed that, with cube-shaped neighbourhoods on $\R^{d}$, the small-neighbourhood limit of Minimum Probability Flow (MPF) is classical score matching. Under regularity conditions, we extend this limit to symmetric connectivity functions on open convex domains (\autoref{theorem:sm-convex}). The limit is a Generalised score matching (GSM) objective that does not involve the data score, obtained without first postulating a weighted Fisher divergence. The weight is therefore an output of the construction, determined by the size and shape of the neighbourhoods over which likelihoods are compared. The same limit holds for a family of conditional noise-contrastive objectives that contains MPF~\citep{unified-nce} (Appendix~\ref{appendix:generalization-of-main-theorems}).

    \item \textbf{A common framework for smooth weights.} An explicit family of connectivity functions realises every $C^{2}$ positive-definite weight (\autoref{prop:midpoint-connectivity-checks-main-theorem} and the remarks after it). Constructing the connectivity from a convex function with an invertible Hessian yields the inverse Hessian of that function as the weight, and specific choices of the convex function recover score matching objectives proposed in prior works~\citep{Hyvarinen05a, Hyvaerinen2007, generalized_score} (\autoref{cor:special-cases}). In contrast, weights built from distances to the boundary~\citep{truncated_sm, gsm_general_domain, gsm-on-compositional-data} are not differentiable in general and fall outside this class.

    \item \textbf{Experiments with barrier-induced weights.} Choosing the convex function as a barrier function of the domain attenuates the weight at the boundary (Section~\ref{sec:experiments}). On a truncated Gaussian over a simplex polytope and a Dirichlet distribution on the simplex, we compare the entropic and logarithmic barriers with Truncated score matching~\citep{truncated_sm} and the estimators of \citet{gsm_general_domain} and \citet{gsm-on-compositional-data} (Section~\ref{sec:experiments}). In each model, one of the two barriers attains the lowest median error at $N = 800$, and does so in $84$--$100\%$ of $50$ ground-truth configurations (Tables~\ref{tab:trunc-gm-polytope-10d-ablation-results} and~\ref{tab:dirichlet-simplex-10d-ablation-results}).
\end{enumerate}

Two supporting results rely on standard arguments. Boundary conditions are required when establishing the equivalence between the limiting objective and a weighted Fisher divergence, for which we provide sufficient conditions on bounded and unbounded convex domains (\autoref{prop:main-objective-in-gsm-form}). Together with the completeness of the associated operator (\autoref{prop:linear-operator-is-complete}), these conditions ensure that, for a well-specified model, the population objective is minimised exactly by the parameters at which the model equals the data distribution. Furthermore, we establish that the per-sample loss is a strictly proper local scoring rule of order two~\citep{parry2012proper} relative to the densities satisfying these conditions (Appendix~\ref{appendix:proper-scoring-rules}). Second, for exponential families, the guarantees known on $\R^{d}$~\citep{min_stein_disc_est} and for specific weights on constrained domains~\citep{generalized_score, gsm_general_domain} hold for every weight in the proposed framework: the empirical objective is convex with a closed-form minimiser, and the estimator is consistent and asymptotically normal, with a consistently estimable covariance (Section~\ref{sec:exp-family}). Under the boundary condition, its limit is the true parameter (\autoref{cor:exp-fam-true-parameter}). Proofs are deferred to the appendix.

\section{Notation} \label{appendix:sec-notation}
Scalars are represented using normal font (e.g., $x, \theta$), whereas vectors are represented using  boldface (e.g. $\x{x}, \x{\theta}$), and matrices are represented in uppercase letters (e.g. $G, H$). For vectors $\x{x}, \x{y} \in \R^{d}$, $x_{k}$ denotes its $k$th element, $\norm{x} = \sqrt{\sum_{i=1}^{d}x_{i}^{2}}$ is the standard Euclidean norm and $\text{diag}(\x{x}) \in \R^{d \times d}$ constructs a diagonal matrix with diagonal entries as elements of $\x{x}$. The element-wise product is $\x{x} \circ \x{y}$. For a matrix $G \in \R^{d \times d}$, $G_{ij}$ denotes the element in the $i$th row and $j$th column, $\text{diag}(G) \in \R^{d}$ extracts its diagonal and $\norm{G}$ denotes the Frobenius norm. We denote $\mathbb{I}$ as identity matrix and $\mathbbm{1}_{A}$ as indicator function for set $A$. For a function $f: \R^{d} \to \R$, denote $H_{f}(\x{x})$ as the Hessian of $f$ evaluated at point $\x{x}$. For a vector field $g : \R^{d} \to \R^{d}$, $g_{k}$ is the $k$th output and the divergence is given by $(\nabla \cdot g)(\x{x}) = \sum_{k=1}^{d} \frac{\partial g_{k}(\x{x})}{\partial x_k}$. For a matrix field $D : \R^{d} \rightarrow \R^{d \times d}$, $(\nabla . D)(\x{x}) \in \R^{d}$ denotes the column-wise divergence operation where $(\nabla \cdot D)(\x{x})_{j} = \sum_{i=1}^{d}\frac{\partial D_{ij}(\xx)}{\partial x_i}$ for $j \in \{1,\dots,d\}$. For any $B \subset \R^{d}$, $\overline{B}$ denotes its closure, $\partial B$ denotes its boundary and $\intr B$ denotes the interior of the set. For a positive-definite matrix $G$, $\norm{\x{v}}^{2}_{G} = \x{v}^{\top}G\x{v}$. We denote by $C^{k}$ the space of functions (real-, vector- or matrix-valued) whose derivatives up to order $k$ are continuous. Throughout this paper, we assume that the distributions under consideration always admit a density function with respect to Lebesgue measure. The true density or data density is denoted as $p$. We use an energy based model to denote the model density as $ p_{\x{\theta}}(\x{x}) = \frac{1}{\Z(\x{\theta})} \exp(-E_{\x{\theta}}(\x{x}))$ where $E_{\ttheta}:\R^d \to \R$ is the energy function with the partition function given by $\Z(\x{\theta}) = \int_{\R^{d}} \exp(-E_{\ttheta}(\xx)) \d \xx$. For densities, when we say $p \in C^{k}$, then we mean $p$ and its derivatives up to order $k$ are continuous on the support of $p$.

\section{Background and Related Work} \label{sec:related-work}

Score matching \cite{Hyvarinen05a} was originally proposed as a parameter estimation method that matches the score function of the model distribution to the data distribution without computing partition functions. The score function of a probability density $p_{\x{\theta}}$ is defined as
\begin{align*}
    \x{s}_{\x{\theta}}(\x{x}) = \nabla \log p_{\x{\theta}}(\x{x}) = -\nabla E_{\x{\theta}}(\x{x}) - \nabla \log \Z(\x{\theta})
\end{align*}
Remarkably, the gradient of the log-partition function vanishes, thus, enabling partition-free estimation. The score matching objective minimises the expected squared distance between the model and data scores:
\begin{align*}
    \J(\x{\theta}) = \frac{1}{2} \underset{\x{x} \sim p}{\E}\left[\norm{\nabla \log p_{\x{\theta}}(\x{x}) - \nabla \log p(\x{x})}^2\right]
\end{align*}
Since the data score is unknown, $\J(\x{\theta})$ can be rewritten using integration by parts \citep[Theorem 1]{Hyvarinen05a} as
\begin{align*}
    \J(\x{\theta}) =& \underset{\x{x} \sim p}{\E}\left[\frac{1}{2}\norm{\nabla \log p_{\x{\theta}}(\x{x})}^2 + \tr( H_{\log p_{\x{\theta}}}(\x{x}))\right] + C
\end{align*}
where $C$ is a data-dependent constant. This form depends only on the model, but the boundary term discarded by integration by parts need not vanish when $p$ is supported on a proper subset of $\R^{d}$. This is the issue that the domain-adapted variants below address.

\citet{Lyu_score_09} proposed a general operator-theoretic extension of score matching, which we refer to generalised score matching (GSM). The key idea is to replace the gradient in Fisher divergence with linear operator $\O$ yielding the objective
\begin{align*}
    \J_{\O}(\x{\theta}) &= \frac{1}{2}\underset{\x{x} \sim p}{\E}\left[\norm{\frac{(\O p_{\x{\theta}})(\x{x})}{p_{\x{\theta}}(\x{x})} - \frac{(\O p)(\x{x})}{p(\x{x})}}^{2}\right]
\end{align*}
The GSM framework is highly general, and its effectiveness for parameter depends critically on the notion of \emph{completeness} of the operator $\O$ \cite[Definition 2]{Lyu_score_09}. Choosing $(\O p)(\x{x}) = \nabla p(\x{x})$ retrieves original score matching \cite{Hyvarinen05a}. Moreover, analogous to classical score matching, the objective $\J_{\O}(\x{\theta})$ can also be rewritten in a form depending only on the model density under suitable boundary conditions (cf. Eq. $13$ of \citet{Lyu_score_09}). A natural and general choice is $(\O p)(\xx) = M(\xx)\nabla p(\xx)$ for a matrix-valued function $M$, which gives the weighted Fisher divergence
\begin{align*}
    \J_{D}(\ttheta) = \frac{1}{2}\underset{\xx \sim p}{\E}\left[\norm{\nabla \log p_{\ttheta}(\xx) - \nabla \log p(\xx)}^{2}_{D(\xx)}\right]
\end{align*}
where the \emph{weighting function} $D(\xx) = M(\xx)^{\top}M(\xx)$. On $\R^{d}$, the same objective appears as diffusion score matching in the minimum Stein discrepancy framework of \citet{min_stein_disc_est}.

\citet{Hyvaerinen2007} proposed an extension for non-negative data with $D(\x{x}) = \text{diag}(\x{x}^{2})$. \citet{generalized_score} generalized this to elementwise weights $D(\x{x}) = \diag(\x{h}(\x{x}))$ with $\x{h}: \R^{d}_{+} \to \R^{d}_{+}$, subsequently extending the framework to general domains~\citep{gsm_general_domain} and simplex~\citep{gsm-on-compositional-data} via boundary-truncated distances. For densities truncated to a set $V$, \citet{truncated_sm} weight the Fisher divergence by the distance to the boundary, $D(\xx) = \inf_{\x{z} \in \partial V}\norm{\x{z} - \xx}\,\id$.

The works above propose a weighted objective $\J_{D}$ as their starting point, choose $D$ for the domain at hand, and then verify the boundary conditions needed for integration by parts. In contrast, we derive the objective as the infinitesimal limit of MPF on a convex domain. The weighting $D$ is then determined by the second moments of the connectivity function; any symmetric positive-definite $D$ can be obtained this way, and the choice $D_{\phi} = H_{\phi}^{-1}$ for a convex $\phi$ recovers the objectives of prior works\cite{Hyvarinen05a, Hyvaerinen2007, generalized_score} as special cases.

\paragraph{Further related work.}
The statistical efficiency of score matching has been related to isoperimetric properties of the target \citep{koehler2023statistical} and that of generalised objectives to the mixing of the diffusion associated with the chosen operator \cite{qin24a-fitlikeyousample}. The weighting $D_{\phi} = H_{\phi}^{-1}$ that arises in our construction is the inverse Hessian metric of a convex barrier function, the same geometry that underlies Riemannian and mirror Langevin samplers \citep{girolami2011riemann, hsieh2018mirrored, zhang2020wasserstein, ahn2021efficient} and the Dikin walk on polytopes \citep{kannan2009random}.

\section{GSM Construction}
In this section, we motivate and formally state the key results to show how GSM objectives can be constructed starting from Minimum Probability Flow (MPF) \citep{icml_mpf, mpf}. For a data point $\x{x}$, and a perturbed state $\x{y}$, the MPF objective is defined as (cf. Eq. $C-1$ in \cite{mpf})
\begin{align} \label{eq:mpf-objective}
    \K(\x{\theta}) &= \underset{\x{x} \sim p}{\E}\left[\int_{\x{y}} g(\x{y}, \x{x}) \exp\left(\frac{1}{2}\left(E_{\x{\theta}}(\x{x}) - E_{\x{\theta}}(\x{y})\right)\right)\d \x{y}\right]
\end{align}
where $g(\x{y}, \x{x})$ denotes a symmetric connectivity function. \citet{icml_mpf, mpf} consider a binary connectivity function $g$, defined as $g(\x{y}, \x{x}) = 1$ if $\x{y} \in \mathcal{R}(\x{x})$ and $g(\x{y}, \x{x}) = 0$ otherwise, where $\mathcal{R}(\x{x})$ denotes a neighbourhood of $\x{x}$. In particular, \citet{mpf} show that when $\mathcal{R}(\x{x})$ is chosen to be a cube centered at $\x{x}$ with side length $\varepsilon$ (cf. Eq. $C-2$ in \cite{mpf}), the MPF objective recovers the score matching objective on $\mathbb{R}^{d}$ in the limit as $\varepsilon \to 0$.

We now describe, at a high level, the approach for our construction. We consider a family of symmetric connectivity functions $g_{r}$, indexed by a scale $r \in (0, 1]$, whose support concentrates on the diagonal $\{\x{y} = \x{x}\}$ as $r \to 0$, and study the MPF objective in this limit. For a fixed scale, define
\begin{align*}
    \I(\x{x}, \x{\theta}; r) &= \int_{\x{y}} g_{r}(\x{y}, \x{x}) \exp\left(\frac{1}{2}\left(E_{\x{\theta}}(\x{x}) - E_{\x{\theta}}(\x{y})\right)\right)\d \x{y}
\end{align*}
so that $\K(\ttheta) = \E_{\xx \sim p}\left[\I(\xx, \ttheta; r)\right]$. Define
\begin{align} \label{eq:constant-ind-of-theta}
    \C(\xx; r) &= \int_{\x{y}} g_{r}(\x{y}, \x{x})\d\x{y}
\end{align}
which does not depend on $\ttheta$. Removing $\C$ and rescaling by a normalisation $\Q(r)$ isolates the leading $\ttheta$ dependent behaviour of $\I$. Accordingly, we define
\begin{align*}
    \Psi(\xx, \ttheta; r) &= \frac{\I(\xx, \ttheta; r) - \C(\xx; r)}{\Q(r)}
\end{align*}
The normalisation $\Q(r)$ is chosen to cancel the rate at which $\I(\xx, \ttheta; r) - \C(\xx; r)$ vanishes as $r \to 0$, so that $\Psi(\xx, \ttheta; r)$ converges to a nontrivial, $\ttheta$ dependent limit for each fixed $\xx$. We refer to
\begin{align*}
    \mathcal{L}(\x{\theta}) &= \lim_{r \to 0}\ \E_{\xx \sim p}\left[\Psi(\xx, \ttheta; r)\right]
\end{align*}
as the \emph{normalised objective}, to distinguish it from $\K(\ttheta)$ itself. We assume the following throughout.
\begin{enumerate}[label=(A\arabic*)]
    \item \label{ass:dominating-function-for-psi} There exist $r_{0} > 0$ and a function $\mathcal{V}(\cdot\,; \ttheta)$ with $\E_{\xx \sim p}[\mathcal{V}(\xx; \ttheta)] < \infty$ for all $\ttheta$, such that $\abs{\Psi(\xx, \ttheta; r)} \leq \mathcal{V}(\xx; \ttheta)$ for all $r \in (0, r_{0}]$ and $p$-almost every $\xx$.
\end{enumerate}
In the following, we show that the normalised objective recovers GSM. Let the support of $p$ be a non-empty open convex set $\Omega \subset \R^{d}$, and let $g_{r} : \Omega \times \Omega \to [0, \infty)$ be a family of non-negative connectivity functions. For $\x{x} \in \Omega$, we define
\begin{align*}
    \x{m}_{1, r}(\x{x}) = \int_{\Omega} g_{r}(\x{y}, \x{x})(\x{y} - \x{x})\d \y \quad \text{and} \quad M_{2, r}(\x{x}) = \int_{\Omega} g_{r}(\x{y}, \x{x})(\x{y} - \x{x})(\x{y} - \x{x})^{\top} \d \x{y}
\end{align*}
to be the local first and second moments respectively. The next theorem shows that, for symmetric connectivity, the limit of $\Psi$ is determined by the second moments.
\begin{theorem} \label{theorem:sm-convex}
    Let $E_{\x{\theta}} : \Omega \to \R$ be in $C^{2}$ and let $q : (0, 1] \to (0, \infty)$. Suppose that the connectivity functions $g_{r}$ are measurable and satisfy the following for $r \in (0, 1]$
    \begin{enumerate}[label=(A\arabic*), start=2]
        \item \label{cond:symmetric-g-and-c-bounded} $g_{r}(\x{y}, \x{x}) = g_{r}(\x{x}, \x{y})$ for all $\x{x}, \x{y} \in \Omega$. Further, for every compact $K \subset \Omega$, there exists a $r_{K} > 0$ such that function $\C(\cdot; r)$ is bounded on $K$ for all $0 < r < r_{K}$.
        \item \label{cond:delta-go-to-0-and-g-0} For every compact $K \subset \Omega$, there are radii $\delta_{r}(K) > 0$ such that $\delta_{r}(K) \to 0$ as $r \to 0$ and $g_{r}(\x{y}, \x{x}) = 0$ whenever $\x{x} \in K$ and $\norm{\y - \xx} > \delta_{r}(K)$.
        \item \label{cond:uni-conv-second-moment} There is a $C^{1}$ function $D: \Omega \to \R^{d \times d}$, with $D(\x{x})$ symmetric positive-definite, such that for every compact $K \subset \Omega$, $\sup_{\x{x} \in K} \norm{\frac{1}{q(r)}M_{2, r}(\x{x}) - D(\x{x})} \to 0$ as $r \to 0$.
        \item \label{cond:uni-conv-first-moment} There is a continuous function $\x{b} : \Omega \to \R^{d}$ such that for every compact $K \subset \Omega$, $\sup_{\x{x} \in K}\norm{\frac{1}{q(r)}\x{m}_{1, r}(\x{x}) - \x{b}(\x{x})} \to 0$ as $r \to 0$.
    \end{enumerate}
    Take $\Q(r) = \frac{q(r)}{4}$, then for every $\x{x} \in \Omega$, we have
    \begin{align} \label{eq:psi-pointwise-limit}
        \lim_{r \to 0} \Psi(\x{x}, \x{\theta}; r) = \frac{1}{2}\nabla E_{\ttheta}(\xx)^{\top}D(\xx)\nabla E_{\ttheta}(\xx) - \nabla \cdot \left(D(\xx)\nabla E_{\ttheta}(\xx)\right)
    \end{align}
    Moreover if \ref{ass:dominating-function-for-psi} holds, the normalised objective is
    \begin{align} \label{eq:main-objective-in-compact-form}
        \L(\x{\theta}) = \underset{\x{x} \sim p}{\E}\left[\frac{1}{2}\nabla E_{\x{\theta}}(\x{x})^{\top}D(\x{x})\nabla E_{\x{\theta}}(\x{x}) - \nabla \cdot \left(D(\x{x})\nabla E_{\x{\theta}}(\x{x})\right)\right]
    \end{align}
\end{theorem}

As stated below, \autoref{eq:main-objective-in-compact-form} admits an equivalent formulation in terms of weighted Fisher Divergence.

\begin{proposition} \label{prop:main-objective-in-gsm-form}
    Suppose the assumptions of \autoref{theorem:sm-convex} hold, let $p \in C^{1}$, and $\x{n}(\x{z})$ denote the outward unit normal at $\x{z} \in \partial\Omega$. Assume that $\E_{\x{x} \sim p}\left[\nabla \log p(\x{x})^{\top}D(\x{x})\nabla \log p(\x{x})\right] < \infty$ and that, for each $\x{\theta}$, the following hold.
    \begin{enumerate}[label=(A\arabic*), start=6]
        \item \label{cond:bounded-quad-form-model} $\E_{\x{x} \sim p}\left[\nabla \log p_{\x{\theta}}(\x{x})^{\top}D(\x{x})\nabla \log p_{\x{\theta}}(\x{x})\right] < \infty$.
        \item \label{cond:regularity-conditions-convex-set} For all bounded sequences $\{\x{z}_{j}\}$ in $\partial \Omega$ at which $\x{n}(\x{z}_{j})$ exists and $\{\x{x}_{j}\}$ in $\Omega$ with $\norm{\x{x}_{j} - \x{z}_{j}} \to 0$,
        \begin{align} \label{eq:regularity-conditions-convex-set}
            \lim_{j \to \infty} p(\x{x}_{j})\x{n}(\x{z}_{j})^{\top}D(\x{x}_{j})\nabla \log p_{\x{\theta}}(\x{x}_{j}) = 0
        \end{align}
        \item \label{cond:finite-norm-for-unbounded-sets} $\E_{\x{x} \sim p}\left[\norm{D(\x{x}) \nabla \log p_{\x{\theta}}(\x{x})}\right] < \infty$ if $\Omega$ is unbounded.
    \end{enumerate}
    Then, minimising $\L(\x{\theta})$ in \autoref{eq:main-objective-in-compact-form} with respect to $\x{\theta}$ is equivalent to minimising
    \begin{align} \label{eq:main-objective-in-gsm-form}
        \L_{\text{GSM}}(\x{\theta}) &= \frac{1}{2}\underset{\x{x} \sim p}{\E}\left[\norm{\nabla \log p_{\x{\theta}}(\x{x}) - \nabla \log p(\x{x})}^{2}_{D(\x{x})}\right]
    \end{align}
\end{proposition}

An immediate consequence of \autoref{prop:main-objective-in-gsm-form} is that the score matching objective in \autoref{eq:main-objective-in-gsm-form} corresponds to linear operator $\O$ defined as
\begin{align} \label{eq:definition-linear-operator-gsm-framework}
    (\O g)(\x{x}) = D(\x{x})^{\frac{1}{2}}\nabla g(\x{x})
\end{align}
in the framework of \citet{Lyu_score_09}. This holds because $\norm{\x{v}}^{2}_{G} = \norm{G^{\frac{1}{2}}\x{v}}^{2}$ for a positive-definite matrix $G$. As stated below, operator $\O$ is complete.
\begin{proposition} \label{prop:linear-operator-is-complete}
    Let $\O$ be the operator defined in \autoref{eq:definition-linear-operator-gsm-framework}, and let $p_{1}, p_{2} \in C^{1}$ be densities with support $\Omega$. Then $\O$ is complete, that is, if $\frac{(\O p_{1})(\x{x})}{p_{1}(\x{x})} = \frac{(\O p_{2})(\x{x})}{p_{2}(\x{x})}$ almost everywhere, then $p_{1} = p_{2}$ almost everywhere.
\end{proposition}

\autoref{theorem:sm-convex} applies to any family of connectivity functions satisfying \ref{cond:symmetric-g-and-c-bounded} to \ref{cond:uni-conv-first-moment}. The limiting weight $D$ is determined entirely by the second moments. We now show that such a family exists for every prescribed weight under mild regularity conditions of such weight. For $\x{x}, \x{y} \in \Omega$ write $\bar{\x{x}} = \frac{1}{2}(\x{x} + \x{y})$, which lies in $\Omega$ since $\Omega$ is convex. Let $A : \Omega \to \R^{d \times d}$ be $C^{2}$ with $A(\x{x})$ being symmetric positive-definite, $c : \Omega \to (0, \infty)$ be $C^{2}$, and  $\varepsilon : (0, 1] \to (0, 1]$ satisfy
\begin{align*}
    \varepsilon(r) \to 0 \quad \text{and} \quad \frac{\varepsilon(r)}{r} \to \infty \qquad \text{as } r \to 0
\end{align*}
An example of such $\varepsilon$ would be $\varepsilon(r) = \sqrt{r}$. Define $\mathcal{B}_{\varepsilon}(\x{x}) = \{\x{y} : \norm{\x{y} - \x{x}} \leq \varepsilon\}$ and $g_{r} : \Omega \times \Omega \to [0, \infty)$ by
\begin{align} \label{eq:midpoint-symmetric-connectivity}
    g_{r}(\x{y}, \x{x}) = c(\bar{\x{x}})\exp\left(-\frac{\norm{\x{y} - \x{x}}^{2}_{A(\bar{\x{x}})}}{2r^{2}}\right)\mathbbm{1}_{\mathcal{B}_{\varepsilon(r)}(\x{x})}(\x{y})
\end{align}

\begin{proposition} \label{prop:midpoint-connectivity-checks-main-theorem}
    $g_{r}$ defined in \autoref{eq:midpoint-symmetric-connectivity} satisfies \ref{cond:symmetric-g-and-c-bounded} to \ref{cond:uni-conv-first-moment} of \autoref{theorem:sm-convex} with $q(r) = (2\pi)^{\frac{d}{2}}r^{d+2}$ and $D(\x{x}) = c(\x{x})\det\left(A(\x{x})\right)^{-\frac{1}{2}}A(\x{x})^{-1}$.
\end{proposition}

\begin{remark}
    \autoref{prop:midpoint-connectivity-checks-main-theorem} places no condition on the behaviour of $A$ or $D$ near $\partial\Omega$. Whether the resulting objective is meaningful on a particular domain is a separate question, answered by \autoref{prop:main-objective-in-gsm-form}.
\end{remark}

As an important special case, let $\phi : \Omega \to \R$ be a strictly convex function in $C^4$ with positive-definite Hessian $H_{\phi}$. Choosing $A = H_{\phi}$ yields $D = (\det H_{\phi})^{-\frac{1}{2}} H_{\phi}^{-1}$ when $c = 1$, and $D = H_{\phi}^{-1}$ when $c = (\det H_{\phi})^{\frac{1}{2}}$. This Hessian induced construction connects our framework to domain-adapted score matching from prior work and serves as the primary specification in our empirical evaluations, where we adopt the choice $c = (\det H_{\phi})^{\frac{1}{2}}$ unless stated otherwise. To highlight the dependence of $D$ on $\phi$, we represent $D_{\phi} = H_{\phi}^{-1}$. We now state the following corollary with the setup described.

\begin{corollary} \label{cor:special-cases}
    \begin{enumerate}[label=(\alph*)]
        \item Let $\Omega = \R^{d}$, then taking $\phi(\x{x}) = \frac{1}{2}\norm{\x{x}}^{2}$, recovers original score matching \cite{Hyvarinen05a}.
        \item Let $\Omega = \R^{d}_{+}$, then taking $\phi(\x{x}) = -\sum_{i=1}^{d}\log x_{i}$, recovers non-negative score matching \cite{Hyvaerinen2007}.
        \item Let $\Omega = \R^{d}_{+}$, then with $\phi(\x{x}) = \sum_{i=1}^{d}\phi_{i}(x_{i})$ where $\phi_{i} : \R_{+} \to \R$ is convex with $\phi_{i} \in C^{4}$ and $\phi_{i}'' > 0$. Then the score matching objective is the same as the objective proposed in \citet{generalized_score} with $\x{h} = [\frac{1}{\phi''_{1}},\dots,\frac{1}{\phi''_{d}}]$.
    \end{enumerate}
\end{corollary}

\begin{remark}
    Extending the above discussion, for any matrix-valued function $D \in C^2$ taking symmetric positive-definite values, choosing $A = D^{-1}$ and $c = (\det D)^{-1/2}$ yields GSM objective with weighting $D$.
\end{remark}

\begin{remark}
    The construction does not depend on MPF specifically. In Appendix~\ref{appendix:generalization-of-main-theorems}, we show that every $f$-CondNCE objective~\citep{unified-nce} with $f \in C^{3}$ and $f''(1) > 0$, which includes MPF for $f(z) = -\sqrt{z}$, yields the same limit (\autoref{eq:main-objective-in-compact-form}). The choice of $f$ enters only through the normalisation $\Q(r) = f''(1)\,q(r)$, while the weighting $D$ is determined by $g_{r}$ alone.
\end{remark}
\begin{remark}
    In Appendix~\ref{appendix:proper-scoring-rules}, we show that the resulting per-sample loss is a strictly proper local scoring rule of order two~\citep{parry2012proper} with generating matrix $D$.
\end{remark}

\section{Analysis for the Exponential Family} \label{sec:exp-family}
Consider model density $p_{\x{\theta}}$ from exponential family on $\Omega$ with $\ttheta \in \Theta$, given by
\begin{align} \label{eq:exp-fam}
    \log p_{\ttheta}(\xx) = \ttheta^{\top}\x{t}(\xx) - \psi(\ttheta) + \beta(\xx)
\end{align}
where the sufficient statistic $\x{t}:\Omega \to \R^{r}$ and the log base measure $\beta : \Omega \to \R$ are of class $C^{2}$, $\psi(\ttheta)$ is the log-partition function, and $\Theta \subseteq \{\ttheta \in \R^{r} : \psi(\ttheta) < \infty\}$.

Let $J_{\x{t}}(\x{x}) \in \R^{r \times d}$ denote the Jacobian of $\x{t}$, so that $\nabla \log p_{\ttheta} = J_{\x{t}}^{\top}\ttheta + \nabla \beta$. For $\x{\theta} \in \R^{r} \setminus \Theta$, we define $\nabla \log p_{\ttheta} = J_{\x{t}}^{\top}\ttheta + \nabla \beta$ as it doesn't involve $\psi$. Consequently, the per-sample loss is given by
\begin{align} \label{eq:exp-fam-per-sample-loss}
    S(\xx; \ttheta) = \frac{1}{2}\nabla \log p_{\ttheta}(\xx)^{\top} D(\xx) \nabla \log p_{\ttheta}(\xx) + \nabla \cdot \big(D(\xx)\nabla \log p_{\ttheta}(\xx)\big)
\end{align}
is defined for every $\ttheta \in \R^{r}$. For i.i.d. $\xx_{1}, \dots, \xx_{N} \in \Omega$ drawn from $p$, the finite-sample objective is $\hat{\L}_{N}(\ttheta) = \frac{1}{N}\sum_{i=1}^{N} S(\xx_{i}; \ttheta)$. For $\xx \in \Omega$ and $\ell = 1, \dots, r$, define
\begin{align} \label{eq:exp-fam-Gamma-g-c}
    \Gamma(\xx) &= J_{\x{t}}(\xx) D(\xx) J_{\x{t}}(\xx)^{\top}, \nonumber \\
    g_{\ell}(\xx) &= \nabla \cdot \big(D \nabla t_{\ell}\big)(\xx) + \nabla t_{\ell}(\xx)^{\top} D(\xx) \nabla \beta(\xx), \\
    c(\xx) &= \frac{1}{2}\nabla \beta(\xx)^{\top} D(\xx) \nabla \beta(\xx) + \nabla \cdot (D \nabla \beta)(\xx) \nonumber
\end{align}
where $\x{g}(\xx) = (g_{1}(\xx), \dots, g_{r}(\xx))^{\top}$, and $\Gamma_{N}$, $\x{g}_{N}$ and $c_{N}$ denote sample averages of $\Gamma$, $\x{g}$ and $c$, respectively.

\begin{proposition} \label{theorem:convexity}
    For all $\xx_{1}, \dots, \xx_{N} \in \Omega$, the following hold.
    \begin{enumerate}[label=(\alph*)]
        \item For every $\ttheta \in \R^{r}$,
        \begin{align} \label{eq-convex-gsm-loss-exp}
            \hat{\L}_{N}(\ttheta) = \frac{1}{2}\ttheta^{\top}\Gamma_{N}\ttheta + \x{g}_{N}^{\top}\ttheta + c_{N}
        \end{align}
        where $\x{g}_{N}$ can be written as
        \begin{align*}
            \x{g}_{N} = \frac{1}{N}\sum_{i=1}^{N}\Big[\x{\tau}(\xx_{i}) + J_{\x{t}}(\xx_{i})(\nabla \cdot D)(\xx_{i}) + J_{\x{t}}(\xx_{i})D(\xx_{i})\nabla \beta(\xx_{i})\Big]
        \end{align*}
        where $\tau_{\ell}(\xx) = \tr\big(D(\xx)H_{t_{\ell}}(\xx)\big)$.
        \item $\Gamma_{N}$ is symmetric positive-semidefinite. Furthermore, $\Gamma_{N} \succ 0$ if and only if the matrix $M_{N} = \begin{bmatrix} J_{\x{t}}(\xx_{1}) & \cdots & J_{\x{t}}(\xx_{N}) \end{bmatrix}^{\top} \in \R^{Nd \times r}$ has rank $r$.
        \item $\hat{\L}_{N}$ is convex on $\R^{r}$. If $\Gamma_{N} \succ 0$, then $\hat{\L}_{N}$ is strictly convex and $\hat{\ttheta}_{N} = -\Gamma_{N}^{-1}\x{g}_{N}$ is its unique minimiser over $\R^{r}$.
    \end{enumerate}
\end{proposition}

When $\E\left[\norm{\Gamma(\x{x})}\right]$, $\E\left[\norm{\x{g}(\x{x})}\right]$ and $\E\left[\abs{c(\x{x})}\right]$ are finite, set $\Gamma_{0} = \E[\Gamma(\x{x})]$, $\x{g}_{0} = \E[\x{g}(\x{x})]$ and $c_{0} = \E[c(\x{x})]$. The ensemble counterpart of $\hat{\L}_{N}$ with $\ttheta \in \R^{r}$ is
\begin{align} \label{eq:exp-fam-population}
    \L(\ttheta) = \E_{\x{x} \sim p}\big[S(\x{x}; \ttheta)\big] = \frac{1}{2}\ttheta^{\top}\Gamma_{0}\ttheta + \x{g}_{0}^{\top}\ttheta + c_{0}
\end{align}

\begin{theorem}[Existence, consistency and asymptotic normality] \label{theorem:consistency}
    Let $\x{x}_{1}, \x{x}_{2}, \dots$ be i.i.d.\ with density $p$ on $\Omega$, and assume
    \begin{enumerate}[label=(A\arabic*), start=9]
        \item \label{ass:exp-fam-moments} $\E\left[\norm{\Gamma(\x{x})}\right]$, $\E\left[\norm{\x{g}(\x{x})}\right]$ and $\E\left[\abs{c(\x{x})}\right]$ are finite.
        \item \label{ass:exp-fam-identifiable} $\Gamma_{0} \succ 0$.
    \end{enumerate}
    Let $\ttheta_{0} = -\Gamma_{0}^{-1}\x{g}_{0}$ be the unique minimiser of $\L$ over $\R^{r}$. Let $\Gamma_{N}$ and $\x{g}_{N}$ be computed from $\x{x}_{1}, \dots, \x{x}_{N}$. Set $\hat{\ttheta}_{N} = -\Gamma_{N}^{-1}\x{g}_{N}$ on the event $\{\Gamma_{N} \succ 0\}$, and $\hat{\ttheta}_{N} = \x{0}$ otherwise. Then:
    \begin{enumerate}[label=(\alph*)]
        \item almost surely, $\Gamma_{N} \succ 0$ for all sufficiently large $N$, and whenever $\Gamma_{N} \succ 0$, $\hat{\ttheta}_{N}$ is the unique minimiser of $\hat{\L}_{N}$ over $\R^{r}$;
        \item $\hat{\ttheta}_{N} \asto \ttheta_{0}$;
        \item if $\Sigma_{0} = \E\big[(\Gamma(\x{x})\ttheta_{0} + \x{g}(\x{x}))(\Gamma(\x{x})\ttheta_{0} + \x{g}(\x{x}))^{\top}\big]$ is finite, then $\sqrt{N}(\hat{\ttheta}_{N} - \ttheta_{0}) \dto \N(\x{0}, \Gamma_{0}^{-1}\Sigma_{0}\Gamma_{0}^{-1})$;
        \item if $\ttheta_{0} \in \intr\Theta$, then almost surely $\hat{\ttheta}_{N} \in \Theta$ for all sufficiently large $N$;
        \item if $\E\left[\norm{\Gamma(\x{x})}^{2}\right]$ and $\E\left[\norm{\x{g}(\x{x})}^{2}\right]$ are finite, then $\Sigma_{0}$ is finite and
        \begin{align*}
            \Gamma_{N}^{-1}\hat{\Sigma}_{N}\Gamma_{N}^{-1} &\asto \Gamma_{0}^{-1}\Sigma_{0}\Gamma_{0}^{-1}, \quad \text{where} \\
            \hat{\Sigma}_{N} &= \frac{1}{N}\sum_{i=1}^{N}\big(\Gamma(\x{x}_{i})\hat{\ttheta}_{N} + \x{g}(\x{x}_{i})\big)\big(\Gamma(\x{x}_{i})\hat{\ttheta}_{N} + \x{g}(\x{x}_{i})\big)^{\top}
        \end{align*}
    \end{enumerate}
\end{theorem}

\autoref{theorem:consistency} concerns the population minimiser $\ttheta_{0}$. When the model is well specified, \autoref{prop:main-objective-in-gsm-form} identifies $\ttheta_{0}$ with the true parameter.

\begin{corollary}[Consistency for the true parameter] \label{cor:exp-fam-true-parameter}
    Assume \ref{ass:exp-fam-moments}, \ref{ass:exp-fam-identifiable}, and that the model is well specified, that is, $p = p_{\ttheta^{*}}$ for some $\ttheta^{*} \in \Theta$. Suppose that the assumptions of \autoref{prop:main-objective-in-gsm-form} hold for every $\ttheta$ in a neighbourhood of $\ttheta^{*}$. Then $\ttheta^{*}$ is the unique minimiser of $\L$, that is, $\ttheta_{0} = \ttheta^{*}$. Consequently, the conclusions of \autoref{theorem:consistency} hold with $\ttheta_{0} = \ttheta^{*}$; in particular, $\hat{\ttheta}_{N} \asto \ttheta^{*}$.
\end{corollary}

For $\Omega = \R^{d}$, $\hat{\ttheta}_{N}$ is a diffusion score matching estimator, for which \citet{min_stein_disc_est} established consistency and asymptotic normality under certain assumptions; what a constrained domain such as formulation considered here adds is the boundary condition that \autoref{cor:exp-fam-true-parameter} needs.

\section{Experiments} \label{sec:experiments}
In this section, we evaluate the proposed generalised score matching estimators against existing methods for parameter estimation on constrained domains. Specifically, we focus on distributions supported on $(d-1)$-simplex, and the standard simplex polytope $\mathcal{S}^{d} = \{\x{x} \in \R^{d} : x_{i} > 0, \sum_{i=1}^{d}x_{i} < 1\}$, where MLE is intractable or computationally expensive.

\textbf{Barrier functions.} Assumption \ref{cond:regularity-conditions-convex-set} of \autoref{prop:main-objective-in-gsm-form} establishes that the boundary surface integral vanishes, $p(\x{x}) \x{n}(\x{x})^\top D(\x{x}) \nabla \log p_\theta(\x{x})$, tends to zero as $\x{x} \to \partial\Omega$. With $D_{\phi} = H_{\phi}^{-1}$, a natural way to satisfy this condition is to choose $\phi$ whose Hessian diverges as $\xx \to \partial\Omega$. We call such a $\phi$ a \emph{barrier function} for $\Omega$, as in interior-point methods, although only its Hessian needs to diverge. For a convex polytope $\Omega = \{\x{x} \in \R^{d} : \x{a}_{k}^{\top}\x{x} < b_{k},\ 1 \le k \le m\}$, which covers $\mathcal{S}^{d}$ and, in reduced coordinates, $\Delta^{d-1}$, barriers can be built from the slacks $s_{k}(\x{x}) = b_{k} - \x{a}_{k}^{\top}\x{x}$ and need only $(\x{a}_{k}, b_{k})$. Our experiments use the entropic barrier $\phi_{1} = \sum_{k} s_{k}\log s_{k}$ and the logarithmic barrier $\phi_{2} = -\sum_{k}\log s_{k}$. Near a facet, the resulting weight vanishes across the facet but not along it, whereas a scalar weight, such as the distance weighting of Truncated score matching \cite{truncated_sm}, vanishes in every direction. A weight that does not vanish on part of $\partial\Omega$ satisfies Assumption~\ref{cond:regularity-conditions-convex-set} there only if the data density or the weighted normal score of the model vanishes, which is not the case for the truncated Gaussian. For this reason, we do not compare against the elementwise weights of \citet{generalized_score}, designed for $\R^{d}_{+}$, which do not vanish on the facet $\{\sum_{i} x_{i} = 1\}$ of $\mathcal{S}^{d}$.

\subsection{Truncated Gaussian Model} \label{subsec:truncated-gaussian-polytope-exp}
We consider a truncated Gaussian density of the form
\begin{align*}
    p_{\x{\mu}, K}(\x{x}) \propto \exp\left(-\frac{1}{2}\norm{\x{x} - \x{\mu}}^{2}_{K}\right)\mathbbm{1}_{\Omega}(\x{x})
\end{align*}
where $\mathbbm{1}_{\Omega}$ denotes the indicator function restricting the support to $\Omega$, $\x{\mu} \in \R^{d}$, and $K \in \R^{d \times d}$ is the symmetric positive-definite matrix. We take $\Omega = \mathcal{S}^{d}$ with $d = 10$. Experiments are evaluated across multiple sample sizes $N \in \{200, 500, 800\}$, with performance aggregated over $50$ independent trials. We compare our proposed GSM estimators against Truncated Score Matching (Truncated SM) \cite{truncated_sm} and Generalised Score Matching on General Domains (GSM-GD) \cite{gsm_general_domain} evaluated under weight functions $\x{h}_{1}(\x{x}) = \x{x}$, $\x{h}_{2}(\x{x}) = \x{x}^2$, and $\x{h}_{3}(\x{x}) = \sqrt{\x{x}}$. For our method, we use the barrier $\phi_{1}$ and $\phi_{2}$ defined above.

\autoref{fig:trunc-gm-polytope-10d-results} shows the error across sample sizes, and \autoref{tab:trunc-gm-polytope-10d-results} in the appendix reports the values at $N = 800$. The entropic barrier $\phi_{1}$ is the most accurate method for both parameters. For $\x{\mu}$, its median MSE is $0.0045$, against $0.0070$ for the strongest baseline, GSM-GD($\x{h}_{1}$). For $K$, $\phi_{1}$ attains a median MSE of $1.65 \times 10^{4}$, against $1.98 \times 10^{4}$ for GSM-GD($\x{h}_{1}$).

\begin{figure}[ht]
    \centering
    \includegraphics[width=0.9\linewidth]{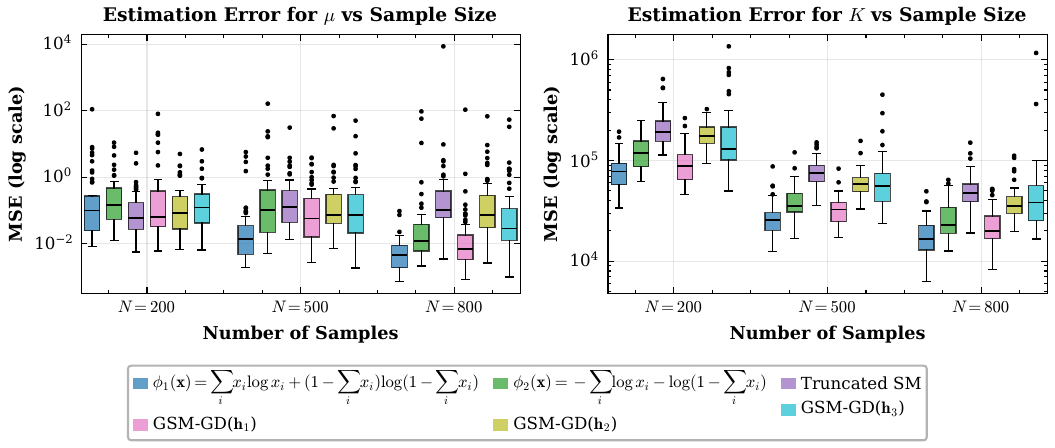}
    \caption{Comparison of parameter estimation error (MSE) for $\x{\mu}$ and $K$ on the simplex polytope $\mathcal{S}^{10}$ across sample sizes $N \in \{200, 500, 800\}$, evaluated over 50 independent trials.}
    \label{fig:trunc-gm-polytope-10d-results}
\end{figure}

\subsection{Dirichlet Model} \label{appendix:subsec-dirichlet-simplex-exp}
We consider a Dirichlet distribution with concentration parameters $\x{\alpha}$ in $\R^{d}_{+}$, whose density is given by $p_{\x{\alpha}}(\x{x}) \propto \prod_{i=1}^{d}x_{i}^{\alpha_{i} - 1} \mathbbm{1}_{\x{x} \in \Delta^{d-1}}$ where $\Delta^{d-1} = \{\x{x} \in \R^{d}: x_{i} > 0, \sum_{i=1}^{d}x_{i} = 1\}$ denotes the $d-1$ simplex. Since $\Delta^{d-1}$ has an empty interior in $\R^{d}$, \autoref{theorem:sm-convex} does not apply directly. We therefore drop the last coordinate and work with $\x{y} = (x_{1}, \dots, x_{d-1}) \in \mathcal{S}^{d-1}$ and $x_{d} = 1 - \sum_{i} y_{i}$. In these coordinates, the density would be
\begin{align*}
    \tilde{p}_{\x{\alpha}}(\x{y}) \propto (1 - \sum_{i} y_{i})^{\alpha_{d} - 1}\prod_{i < d} y_{i}^{\alpha_{i} - 1}
\end{align*}
which is supported on the polytope $\mathcal{S}^{d-1}$, so the same barriers apply directly.

\begin{figure}
    \centering
    \includegraphics[width=0.63\linewidth]{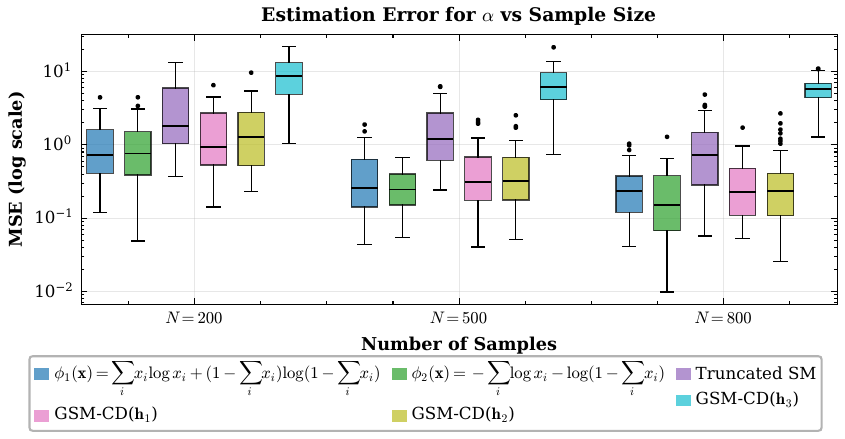}
    \caption{Estimation error (MSE) for $\x{\alpha}$ of the Dirichlet distribution on $\Delta^{9}$ across sample sizes, over 50 independent trials. GSM-CD~\citep{gsm-on-compositional-data} adapts GSM-GD~\citep{gsm_general_domain} to the simplex.}
    \label{fig:dirichlet-simplex-10d-results}
\end{figure}

We evaluate on $\Delta^{9}$ ($d = 10$) for $N \in \{200, 500, 800\}$, over $50$ independent trials. We compare against Truncated SM \cite{truncated_sm} and GSM-CD~\citep{gsm-on-compositional-data}, the adaptation of GSM-GD \citep{gsm_general_domain} to the simplex, with the same weights $\x{h}_{1}, \x{h}_{2}, \x{h}_{3}$. The logarithmic barrier $\phi_{2}$ attains the lowest median MSE at $N = 800$ ($0.153$; \autoref{tab:dirichlet-simplex-10d-results} in the appendix), ahead of GSM-CD($\x{h}_{1}$) ($0.230$), the entropic barrier ($0.236$) and GSM-CD($\x{h}_{2}$) ($0.237$). Truncated SM ($0.731$) and GSM-CD($\x{h}_{3}$) ($5.74$) are clearly worse. The errors of both barriers decrease steadily from $N = 200$ to $N = 800$ (\autoref{fig:dirichlet-simplex-10d-results}).

\subsection{Ablations and Further Experiments} \label{subsec:ablations}

\textbf{Robustness to the ground truth.} The comparisons above use a single ground-truth parameter per model. To check that the conclusions do not depend on this choice, we repeat both experiments for $50$ different ground-truth configurations at $N = 800$ (Appendix~\ref{appendix:ablation-robustness}). The ranking is stable. For the truncated Gaussian, the entropic barrier is the most accurate method for $\x{\mu}$ in $84\%$ of the configurations and for $K$ in all of them. For the Dirichlet, the logarithmic barrier is the most accurate in $86\%$ of them.

\textbf{How fast should the weight vanish?} Near a facet, $D_{\phi}$ decays roughly like $s_{k}$ for the entropic barrier and like $s_{k}^{2}$ for the logarithmic barrier. The power barrier, with Hessian $\sum_{k} s_{k}^{p-2}\x{a}_{k}\x{a}_{k}^{\top}$, interpolates between the two: it becomes the logarithmic barrier as $p \to 0$ and the entropic barrier as $p \to 1$, and larger $p$ means slower decay (Appendix~\ref{appendix:ablation-subsec-power-barrier}). On the truncated Gaussian, the error is smallest for intermediate exponents, around $p = 1.2$, where the power barrier is more accurate than every baseline for both parameters. Decaying too fast ($p$ near $0$) raises the error moderately, consistent with discarding information carried by samples near the boundary. Decaying too slowly is far more harmful: at $p = 1.8$, the median error for $K$ is about ten times its minimum.

\textbf{When attenuation is unnecessary.} Attenuation is a sufficient way to satisfy~\ref{cond:regularity-conditions-convex-set}, not a necessary one. If the model score has zero normal component on $\partial\Omega$ (a Neumann condition), the boundary term vanishes, and original score matching is valid even on a bounded domain. For two such models on $(0,1)$ and $(0,1)^{2}$ (Appendix~\ref{subsec:neumann_ablation}), original score matching generally outperforms both barrier estimators across sample sizes, closely trailing only entropic barrier in $d = 2$ at $N = 800$. Moreover, entropic barrier consistently achieves lower error than the logarithmic barrier. Attenuation therefore has a statistical cost: it down-weights samples near the boundary even when they carry reliable information.

\textbf{Non-convex support.} Our theory assumes a convex domain. As a stress test, we fit a truncated Gaussian on the union of a ball and a surrounding spherical shell in $\R^{3}$, using a barrier for the convex hull (Appendix~\ref{appendix:ablation-subsec-non-convex-support}). This weight does not vanish on the two interior boundaries, so~\ref{cond:regularity-conditions-convex-set} fails there, and, as expected, the error in $K$ does not decrease with $N$. In general, though, the convex hull is not a safe substitute for the true domain.

\textbf{Beyond parameter estimation.} Beyond parameter estimation, the objective can be used in place of sliced score matching to train the score network of an implicit VAE~\citep{song2019sliced}. This gives comparable results on MNIST~\citep{lecun2010mnist} and CelebA~\citep{liu2015faceattributes}, but is sensitive to the choice of weight (Appendix~\ref{appendix:subsec-implicit-vae}). We view this as a proof of concept rather than an improvement.

\subsection{Practical Guidance} \label{subsec:discussion-and-practical-considerations}

These observations suggest three steps. First, choose $\phi$ to be a barrier for $\Omega$. For example, on polytopes, the slack barriers above provide such choices directly. Attenuation alone does not guarantee~\ref{cond:regularity-conditions-convex-set}, which also depends on the model and on the behaviour of $p$ near $\partial\Omega$. Second, avoid the extremes of the attenuation rate. Third, rank the remaining candidates by the estimated total variance, that is, the trace of the covariance of estimator in \autoref{theorem:consistency}(e). Because it measures the spread of $\hat{\ttheta}_{N}$ about the population minimiser, it is designed for candidates for which this minimiser is the true parameter. Among six candidate barriers on the truncated Gaussian, this criterion picked the most accurate one in about half of the trials, roughly three times the rate of a random choice, and its ranking agreed well with the ranking by true error (Spearman correlation $0.88$ for $K$ and $0.66$ for $\x{\mu}$). On the Dirichlet, it picked the single best candidate less often ($30\%$), but it rank the unstable barrier last in every trial (Appendix~\ref{appendix:subsec-data-driven-ranking}).

\section{Conclusions} \label{sec:conclusions}
We have shown that generalised score matching on an open convex domain arises as the small-neighbourhood limit of MPF, and more generally of $f$-CondNCE, with symmetric connectivity. In this construction the weight is an output rather than an assumption: every $C^{2}$ positive-definite weight can be obtained, and inverse Hessians of convex functions recover the objectives of prior work~\citep{Hyvarinen05a, Hyvaerinen2007, generalized_score}. For exponential families, the resulting estimator has a closed form, is consistent and asymptotically normal, and converges to the true parameter when the boundary condition holds. In our experiments, barrier-induced weights attained the lowest median error among the score matching estimators compared, although which barrier was better depended on the model.

An important direction for future work is to develop a systematic approach for choosing $\phi$. Characterising the MSE-optimal choice of $\phi$, studying the existence of $\phi$ that recovers the maximum likelihood estimate, establishing non-asymptotic guarantees for the corresponding estimators are all interesting directions for future work.

\subsection*{AI use statement}
The problem formulation, the derivation of the generalised score matching objective as the infinitesimal limit of minimum probability flow, the barrier-Hessian construction and the resulting unification of existing weightings, the statements of the theoretical results, the design of all experiments, and all code implementing the proposed method, the baselines, and the experiments are the authors' own. We used Claude Opus 5.5 (Anthropic) in three supporting roles. First, as an adversarial proof checker: we supplied our proofs and asked the model to identify gaps, invalid steps, and unstated assumptions; the authors assessed each issue raised and resolved those that were valid. This made several implicit assumptions explicit, for example restricting the parameter set $\Theta$ in the exponential-family results to parameters for which the log-partition function is finite. Second, to check algebraic identities and worked examples in the paper. Third, for editing and \LaTeX{} formatting. We used Gemini 3.8 Flash (Google) to write plotting code for the figures and, after the authors had implemented the baseline methods, to cross-check these implementations against their original descriptions for fidelity. The authors reviewed all model outputs, checked every proof line by line, and take full responsibility for the content of this work, including text, claims, and code produced with the aid of generative AI.

\newpage

\bibliography{refs}

@inproceedings{min_stein_disc_est,
 author = {Barp, Alessandro and Briol, Francois-Xavier and Duncan, Andrew and Girolami, Mark and Mackey, Lester},
 booktitle = {Advances in Neural Information Processing Systems},
 editor = {H. Wallach and H. Larochelle and A. Beygelzimer and F. d\textquotesingle Alch\'{e}-Buc and E. Fox and R. Garnett},
 pages = {},
 publisher = {Curran Associates, Inc.},
 title = {Minimum Stein Discrepancy Estimators},
 url = {https://proceedings.neurips.cc/paper_files/paper/2019/file/ba7609ee5789cc4dff171045a693a65f-Paper.pdf},
 volume = {32},
 year = {2019}
}

@article{local-proper-scoring-rule-ehm,
author = {Werner Ehm and Tilmann Gneiting},
title = {{Local proper scoring rules of order two}},
volume = {40},
journal = {The Annals of Statistics},
number = {1},
publisher = {Institute of Mathematical Statistics},
pages = {609 -- 637},
year = {2012},
doi = {10.1214/12-AOS973},
URL = {https://doi.org/10.1214/12-AOS973}
}

@article{forecaster-lerch,
 ISSN = {08834237, 21688745},
 URL = {http://www.jstor.org/stable/26408123},
 author = {Sebastian Lerch and Thordis L. Thorarinsdottir and Francesco Ravazzolo and Tilmann Gneiting},
 journal = {Statistical Science},
 number = {1},
 pages = {106--127},
 publisher = {Institute of Mathematical Statistics},
 title = {Forecaster's Dilemma: Extreme Events and Forecast Evaluation},
 urldate = {2026-05-06},
 volume = {32},
 year = {2017}
}

@article{Hyvarinen05a,
  title      = {Estimation of Non-Normalized Statistical Models by Score Matching},
  author     = {A. Hyv{{\"a}}rinen},
  year       = 2005,
  journal    = {Journal of Machine Learning Research},
  volume     = 6,
  number     = 24,
  url        = {http://jmlr.org/papers/v6/hyvarinen05a.html}
}

@article{Hyvaerinen2007,
  title    = {Some extensions of score matching},
  journal  = {Computational Statistics \& Data Analysis},
  volume   = {51},
  number   = {5},
  pages    = {2499-2512},
  year     = {2007},
  issn     = {0167-9473},
  doi      = {https://doi.org/10.1016/j.csda.2006.09.003},
  url      = {https://www.sciencedirect.com/science/article/pii/S0167947306003264},
  author   = {Aapo Hyvärinen}
}

@article{truncated_sm,
author = {Liu, Song and Kanamori, Takafumi and Williams, Daniel J.},
title = {Estimating density models with truncation boundaries using score matching},
year = {2022},
issue_date = {January 2022},
publisher = {JMLR.org},
volume = {23},
number = {1},
issn = {1532-4435},
journal = {J. Mach. Learn. Res.},
month = jan,
articleno = {186},
numpages = {38}
}

@ARTICLE{MRF-2,
  author={Geman, Stuart and Geman, Donald},
  journal={IEEE Transactions on Pattern Analysis and Machine Intelligence}, 
  title={Stochastic Relaxation, Gibbs Distributions, and the Bayesian Restoration of Images}, 
  year={1984},
  volume={PAMI-6},
  number={6},
  pages={721-741},
  doi={10.1109/TPAMI.1984.4767596}}

@inbook{MRF-3,
author = {Yedidia, Jonathan S. and Freeman, William T. and Weiss, Yair},
title = {Understanding belief propagation and its generalizations},
year = {2003},
isbn = {1558608117},
publisher = {Morgan Kaufmann Publishers Inc.},
address = {San Francisco, CA, USA},
booktitle = {Exploring Artificial Intelligence in the New Millennium},
pages = {239–269},
numpages = {31}
}

@ARTICLE{MRF-4,
  author={Boykov, Y. and Veksler, O. and Zabih, R.},
  journal={IEEE Transactions on Pattern Analysis and Machine Intelligence}, 
  title={Fast approximate energy minimization via graph cuts}, 
  year={2001},
  volume={23},
  number={11},
  pages={1222-1239},
  doi={10.1109/34.969114}}

@article{MRF,
author = {Besag, Julian},
title = {Spatial Interaction and the Statistical Analysis of Lattice Systems},
journal = {Journal of the Royal Statistical Society: Series B (Methodological)},
volume = {36},
number = {2},
pages = {192-225},
doi = {https://doi.org/10.1111/j.2517-6161.1974.tb00999.x},
url = {https://rss.onlinelibrary.wiley.com/doi/abs/10.1111/j.2517-6161.1974.tb00999.x},
eprint = {https://rss.onlinelibrary.wiley.com/doi/pdf/10.1111/j.2517-6161.1974.tb00999.x},
year = {1974}
}

@book{hinton2012,
author="Hinton, Geoffrey E.",
title="A Practical Guide to Training Restricted Boltzmann Machines",
bookTitle="Neural Networks: Tricks of the Trade: Second Edition",
year="2012",
publisher="Springer Berlin Heidelberg",
address="Berlin, Heidelberg",
pages="599--619",
isbn="978-3-642-35289-8",
doi="10.1007/978-3-642-35289-8_32",
url="https://doi.org/10.1007/978-3-642-35289-8_32"
}

@inproceedings{RBMs,
author = {Tieleman, Tijmen},
title = {Training restricted Boltzmann machines using approximations to the likelihood gradient},
year = {2008},
isbn = {9781605582054},
publisher = {Association for Computing Machinery},
address = {New York, NY, USA},
url = {https://doi.org/10.1145/1390156.1390290},
doi = {10.1145/1390156.1390290},
booktitle = {Proceedings of the 25th International Conference on Machine Learning},
pages = {1064–1071},
numpages = {8},
location = {Helsinki, Finland},
series = {ICML '08}
}

@ARTICLE{DBNs,
  author={Hinton, Geoffrey E. and Osindero, Simon and Teh, Yee-Whye},
  journal={Neural Computation}, 
  title={A Fast Learning Algorithm for Deep Belief Nets}, 
  year={2006},
  volume={18},
  number={7},
  pages={1527-1554},
  doi={10.1162/neco.2006.18.7.1527}}

@inproceedings{10.5555/3104482.3104568,
author = {Welling, Max and Teh, Yee Whye},
title = {Bayesian learning via stochastic gradient langevin dynamics},
year = {2011},
isbn = {9781450306195},
publisher = {Omnipress},
address = {Madison, WI, USA},
booktitle = {Proceedings of the 28th International Conference on International Conference on Machine Learning},
pages = {681–688},
numpages = {8},
location = {Bellevue, Washington, USA},
series = {ICML'11}
}

@inproceedings{
Grathwohl2020Your,
title={Your classifier is secretly an energy based model and you should treat it like one},
author={Will Grathwohl and Kuan-Chieh Wang and Joern-Henrik Jacobsen and David Duvenaud and Mohammad Norouzi and Kevin Swersky},
booktitle={International Conference on Learning Representations},
year={2020},
url={https://openreview.net/forum?id=Hkxzx0NtDB}
}

@inproceedings{NEURIPS2019_378a063b,
 author = {Du, Yilun and Mordatch, Igor},
 booktitle = {Advances in Neural Information Processing Systems},
 editor = {H. Wallach and H. Larochelle and A. Beygelzimer and F. d\textquotesingle Alch\'{e}-Buc and E. Fox and R. Garnett},
 pages = {},
 publisher = {Curran Associates, Inc.},
 title = {Implicit Generation and Modeling with Energy Based Models},
 url = {https://proceedings.neurips.cc/paper_files/paper/2019/file/378a063b8fdb1db941e34f4bde584c7d-Paper.pdf},
 volume = {32},
 year = {2019}
}

@InProceedings{pmlr-v9-gutmann10a,
  title = 	 {Noise-contrastive estimation: A new estimation principle for unnormalized statistical models},
  author = 	 {Gutmann, Michael and Hyvärinen, Aapo},
  booktitle = 	 {Proceedings of the Thirteenth International Conference on Artificial Intelligence and Statistics},
  pages = 	 {297--304},
  year = 	 {2010},
  editor = 	 {Teh, Yee Whye and Titterington, Mike},
  volume = 	 {9},
  series = 	 {Proceedings of Machine Learning Research},
  address = 	 {Chia Laguna Resort, Sardinia, Italy},
  month = 	 {13--15 May},
  publisher =    {PMLR},
  url = 	 {https://proceedings.mlr.press/v9/gutmann10a.html}
}

@ARTICLE{cd,
  author={Hinton, Geoffrey E.},
  journal={Neural Computation}, 
  title={Training Products of Experts by Minimizing Contrastive Divergence}, 
  year={2002},
  volume={14},
  number={8},
  pages={1771-1800},
  doi={10.1162/089976602760128018}}

@article{
doi:10.1073/pnas.79.8.2554,
author = {J J Hopfield },
title = {Neural networks and physical systems with emergent collective computational abilities.},
journal = {Proceedings of the National Academy of Sciences},
volume = {79},
number = {8},
pages = {2554-2558},
year = {1982},
doi = {10.1073/pnas.79.8.2554},
URL = {https://www.pnas.org/doi/abs/10.1073/pnas.79.8.2554},
eprint = {https://www.pnas.org/doi/pdf/10.1073/pnas.79.8.2554}}

@article{PhysRevLett.35.1792,
  title = {Solvable Model of a Spin-Glass},
  author = {Sherrington, David and Kirkpatrick, Scott},
  journal = {Phys. Rev. Lett.},
  volume = {35},
  issue = {26},
  pages = {1792--1796},
  numpages = {0},
  year = {1975},
  month = {Dec},
  publisher = {American Physical Society},
  doi = {10.1103/PhysRevLett.35.1792},
  url = {https://link.aps.org/doi/10.1103/PhysRevLett.35.1792}
}

@article{Nijkamp_Hill_Han_Zhu_Wu_2020, 
title={On the Anatomy of MCMC-Based Maximum Likelihood Learning of Energy-Based Models}, volume={34}, url={https://ojs.aaai.org/index.php/AAAI/article/view/5973}, DOI={10.1609/aaai.v34i04.5973}, abstractNote={&lt;p&gt;This study investigates the effects of Markov chain Monte Carlo (MCMC) sampling in unsupervised Maximum Likelihood (ML) learning. Our attention is restricted to the family of unnormalized probability densities for which the negative log density (or energy function) is a ConvNet. We find that many of the techniques used to stabilize training in previous studies are not necessary. ML learning with a ConvNet potential requires only a few hyper-parameters and no regularization. Using this minimal framework, we identify a variety of ML learning outcomes that depend solely on the implementation of MCMC sampling.&lt;/p&gt;&lt;p&gt;On one hand, we show that it is easy to train an energy-based model which can sample realistic images with short-run Langevin. ML can be effective and stable even when MCMC samples have much higher energy than true steady-state samples throughout training. Based on this insight, we introduce an ML method with purely noise-initialized MCMC, high-quality short-run synthesis, and the same budget as ML with informative MCMC initialization such as CD or PCD. Unlike previous models, our energy model can obtain realistic high-diversity samples from a noise signal after training.&lt;/p&gt;&lt;p&gt;On the other hand, ConvNet potentials learned with non-convergent MCMC do not have a valid steady-state and cannot be considered approximate unnormalized densities of the training data because long-run MCMC samples differ greatly from observed images. We show that it is much harder to train a ConvNet potential to learn a steady-state over realistic images. To our knowledge, long-run MCMC samples of all previous models lose the realism of short-run samples. With correct tuning of Langevin noise, we train the first ConvNet potentials for which long-run and steady-state MCMC samples are realistic images.&lt;/p&gt;}, number={04}, journal={Proceedings of the AAAI Conference on Artificial Intelligence}, author={Nijkamp, Erik and Hill, Mitch and Han, Tian and Zhu, Song-Chun and Wu, Ying Nian}, year={2020}, month={Apr.}, pages={5272-5280} }

@inproceedings{Lyu_score_09,
  title     = {Interpretation and Generalization of Score Matching},
  author    = {Lyu, S.},
  year      = 2009,
  booktitle = {Proceedings of the Twenty-Fifth Conference on Uncertainty in Artificial Intelligence},
  location  = {Montreal, Quebec, Canada},
  isbn      = 9780974903958
}

@article{extensive_scoring_rule,
	author = {Matthew Parry},
	doi = {10.1214/16-EJS1132},
	journal = {Electronic Journal of Statistics},
	number = {1},
	pages = {1098 -- 1108},
	publisher = {Institute of Mathematical Statistics and Bernoulli Society},
	title = {{Extensive scoring rules}},
	url = {https://doi.org/10.1214/16-EJS1132},
	volume = {10},
	year = {2016}
}

@article{parry2012proper,
author = {Matthew Parry and A. Philip Dawid and Steffen Lauritzen},
title = {{Proper local scoring rules}},
volume = {40},
journal = {The Annals of Statistics},
number = {1},
publisher = {Institute of Mathematical Statistics},
pages = {561 -- 592},
year = {2012},
doi = {10.1214/12-AOS971},
URL = {https://doi.org/10.1214/12-AOS971}
}

@InProceedings{qin24a-fitlikeyousample,
  title = 	 {Fit Like You Sample: Sample-Efficient Generalized Score Matching from Fast Mixing Diffusions},
  author =       {Qin, Yilong and Risteski, Andrej},
  booktitle = 	 {Proceedings of Thirty Seventh Conference on Learning Theory},
  pages = 	 {4413--4457},
  year = 	 {2024},
  editor = 	 {Agrawal, Shipra and Roth, Aaron},
  volume = 	 {247},
  series = 	 {Proceedings of Machine Learning Research},
  month = 	 {30 Jun--03 Jul},
  publisher =    {PMLR},
  url = 	 {https://proceedings.mlr.press/v247/qin24a.html}
}

@inproceedings{fdsm,
author = {Pang, Tianyu and Xu, Kun and Li, Chongxuan and Song, Yang and Ermon, Stefano and Zhu, Jun},
title = {Efficient learning of generative models via finite-difference score matching},
year = {2020},
isbn = {9781713829546},
publisher = {Curran Associates Inc.},
address = {Red Hook, NY, USA},
booktitle = {Proceedings of the 34th International Conference on Neural Information Processing Systems},
articleno = {1609},
numpages = {14},
location = {Vancouver, BC, Canada},
series = {NIPS '20}
}

@inproceedings{icml_mpf,
author = {Sohl-Dickstein, Jascha and Battaglino, Peter and DeWeese, Michael R.},
title = {Minimum probability flow learning},
year = {2011},
isbn = {9781450306195},
publisher = {Omnipress},
address = {Madison, WI, USA},
booktitle = {Proceedings of the 28th International Conference on International Conference on Machine Learning},
pages = {905–912},
numpages = {8},
location = {Bellevue, Washington, USA},
series = {ICML'11}
}

@article{mpf,
  title = {New Method for Parameter Estimation in Probabilistic Models: Minimum Probability Flow},
  author = {Sohl-Dickstein, Jascha and Battaglino, Peter B. and DeWeese, Michael R.},
  journal = {Phys. Rev. Lett.},
  volume = {107},
  issue = {22},
  pages = {220601},
  numpages = {4},
  year = {2011},
  month = {Nov},
  publisher = {American Physical Society},
  doi = {10.1103/PhysRevLett.107.220601},
  url = {https://link.aps.org/doi/10.1103/PhysRevLett.107.220601}
}

@inproceedings{song2019sliced,
	title        = {Sliced Score Matching: {A} Scalable Approach to Density and Score Estimation},
	author       = {Y. Song and S. Garg and J. Shi and S. Ermon},
	year         = 2019,
	booktitle    = {Proceedings of the Thirty-Fifth Conference on Uncertainty in Artificial
                  Intelligence, {UAI}},
	url          = {http://auai.org/uai2019/proceedings/papers/204.pdf}
}

@article{generalized_score,
  author  = {Shiqing Yu and Mathias Drton and Ali Shojaie},
  title   = {Generalized Score Matching for Non-Negative Data},
  journal = {Journal of Machine Learning Research},
  year    = {2019},
  volume  = {20},
  number  = {76},
  pages   = {1--70},
  url     = {http://jmlr.org/papers/v20/18-278.html}
}

@article{gsm_general_domain,
    author = {Yu, Shiqing and Drton, Mathias and Shojaie, Ali},
    title = {Generalized score matching for general domains},
    journal = {Information and Inference: A Journal of the IMA},
    volume = {11},
    number = {2},
    pages = {739-780},
    year = {2021},
    month = {01},
    issn = {2049-8772},
    doi = {10.1093/imaiai/iaaa041},
    url = {https://doi.org/10.1093/imaiai/iaaa041},
    eprint = {https://academic.oup.com/imaiai/article-pdf/11/2/739/44020696/iaaa041.pdf},
}

@inproceedings{energy_discrepancy,
 author = {Schr\"{o}der, Tobias and Ou, Zijing and Lim, Jen and Li, Yingzhen and Vollmer, Sebastian and Duncan, Andrew},
 booktitle = {Advances in Neural Information Processing Systems},
 editor = {A. Oh and T. Naumann and A. Globerson and K. Saenko and M. Hardt and S. Levine},
 pages = {45300--45338},
 publisher = {Curran Associates, Inc.},
 title = {Energy Discrepancies: A Score-Independent Loss for Energy-Based Models},
 url = {},
 volume = {36},
 year = {2023}
}

@inproceedings{
koehler2023statistical,
title={Statistical Efficiency of Score Matching: The View from Isoperimetry},
author={Frederic Koehler and Alexander Heckett and Andrej Risteski},
booktitle={The Eleventh International Conference on Learning Representations },
year={2023},
url={https://openreview.net/forum?id=TD7AnQjNzR6}
}

@inproceedings{
unified-nce,
title={A Unified View on Learning Unnormalized Distributions via Noise-Contrastive Estimation},
author={Jongha Jon Ryu and Abhin Shah and Gregory W. Wornell},
booktitle={Forty-second International Conference on Machine Learning},
year={2025},
url={https://openreview.net/forum?id=Wwj6jjxZet}
}

@InProceedings{CondNCE,
  title = 	 {Conditional Noise-Contrastive Estimation of Unnormalised Models},
  author =       {Ceylan, Ciwan and Gutmann, Michael U.},
  booktitle = 	 {Proceedings of the 35th International Conference on Machine Learning},
  pages = 	 {726--734},
  year = 	 {2018},
  editor = 	 {Dy, Jennifer and Krause, Andreas},
  volume = 	 {80},
  series = 	 {Proceedings of Machine Learning Research},
  month = 	 {10--15 Jul},
  publisher =    {PMLR},
  url = 	 {https://proceedings.mlr.press/v80/ceylan18a.html}
}

@book{durrett2019,
  author    = {Durrett, Rick},
  title     = {Probability: Theory and Examples},
  edition   = {5th},
  year      = {2019},
  publisher = {Cambridge University Press},
  isbn      = {978-1108473682}
}

@book{vanderVaart1998,
  title={Asymptotic Statistics},
  author={van der Vaart, A. W. },
  year={1998},
  publisher={Cambridge University Press }
}

@article{gsm_xu,
title = {Generalized score matching},
journal = {Journal of Multivariate Analysis},
volume = {210},
pages = {105473},
year = {2025},
issn = {0047-259X},
doi = {https://doi.org/10.1016/j.jmva.2025.105473},
url = {https://www.sciencedirect.com/science/article/pii/S0047259X25000685},
author = {Jiazhen Xu and Janice L. Scealy and Andrew T.A. Wood and Tao Zou}
}

@article{simon2014introduction,
  title={Introduction to geometric measure theory},
  author={Simon, Leon},
  journal={Tsinghua Lectures},
  volume={2},
  number={2},
  pages={3--1},
  year={2014},
  url={https://math.stanford.edu/~lms/ntu-gmt-text.pdf}
}

@article{lecun2010mnist,
  title={MNIST handwritten digit database},
  author={LeCun, Yann and Cortes, Corinna and Burges, CJ},
  journal={ATT Labs [Online]. Available: http://yann.lecun.com/exdb/mnist},
  volume={2},
  year={2010}
}

@inproceedings{liu2015faceattributes,
  title = {Deep Learning Face Attributes in the Wild},
  author = {Liu, Ziwei and Luo, Ping and Wang, Xiaogang and Tang, Xiaoou},
  booktitle = {Proceedings of International Conference on Computer Vision (ICCV)},
  month = {December},
  year = {2015} 
}

@INPROCEEDINGS{Shetty-mcsm,
  author={Shetty, Nishanth and Seelamantula, Chandra Sekhar},
  booktitle={ICASSP 2025 - 2025 IEEE International Conference on Acoustics, Speech and Signal Processing (ICASSP)}, 
  title={Monte Carlo Score Matching for Image Generation}, 
  year={2025},
  volume={},
  number={},
  pages={1-5},
  doi={10.1109/ICASSP49660.2025.10889041}}

@article{Epperlyetal2024,
  author  = {Epperly, Ethan N. and Tropp, Joel A. and Webber, Robert J.},
  title   = {{XT}race: Making the most of every sample in stochastic trace estimation},
  journal = {SIAM Journal on Matrix Analysis and Applications},
  year    = {2024},
  volume  = {45},
  number  = {1},
  pages   = {1--23},
  doi     = {10.1137/23M1548323}
}

@article{Hutchinson01011990,
author = {M.F. Hutchinson},
title = {A stochastic estimator of the trace of the influence matrix for laplacian smoothing splines},
journal = {Communications in Statistics - Simulation and Computation},
volume = {19},
number = {2},
pages = {433--450},
year = {1990},
publisher = {Taylor \& Francis},
doi = {10.1080/03610919008812866},
URL = {https://doi.org/10.1080/03610919008812866}
}

@inproceedings{
bortoli2025distributional,
title={Distributional Diffusion Models with Scoring Rules},
author={Valentin De Bortoli and Alexandre Galashov and J Swaroop Guntupalli and Guangyao Zhou and Kevin Patrick Murphy and Arthur Gretton and Arnaud Doucet},
booktitle={Forty-second International Conference on Machine Learning},
year={2025},
url={https://openreview.net/forum?id=N82967FcVK}
}

@InProceedings{gsm-on-compositional-data,
  title = 	 {Interaction Models and Generalized Score Matching for Compositional Data},
  author =       {Yu, Shiqing and Drton, Mathias and Shojaie, Ali},
  booktitle = 	 {Proceedings of the Second Learning on Graphs Conference},
  pages = 	 {20:1--20:25},
  year = 	 {2024},
  editor = 	 {Villar, Soledad and Chamberlain, Benjamin},
  volume = 	 {231},
  series = 	 {Proceedings of Machine Learning Research},
  month = 	 {27--30 Nov},
  publisher =    {PMLR},
  url = 	 {https://proceedings.mlr.press/v231/yu24a.html}
}

@InProceedings{neumann_boundary,
  title = 	 { Minimax Optimal Regression over Sobolev Spaces via Laplacian Regularization on Neighborhood Graphs },
  author =       {Green, Alden and Balakrishnan, Sivaraman and Tibshirani, Ryan},
  booktitle = 	 {Proceedings of The 24th International Conference on Artificial Intelligence and Statistics},
  pages = 	 {2602--2610},
  year = 	 {2021},
  editor = 	 {Banerjee, Arindam and Fukumizu, Kenji},
  volume = 	 {130},
  series = 	 {Proceedings of Machine Learning Research},
  month = 	 {13--15 Apr},
  publisher =    {PMLR},
  url = 	 {https://proceedings.mlr.press/v130/green21a.html}
}

@book{rockafellar1997convex,
  title={Convex analysis},
  author={Rockafellar, R Tyrrell},
  volume={28},
  year={1997},
  publisher={Princeton university press}
}

@article{hsieh2018mirrored,
  title={Mirrored langevin dynamics},
  author={Hsieh, Ya-Ping and Kavis, Ali and Rolland, Paul and Cevher, Volkan},
  journal={Advances in Neural Information Processing Systems},
  volume={31},
  year={2018}
}

@article{girolami2011riemann,
  title={Riemann manifold langevin and hamiltonian monte carlo methods},
  author={Girolami, Mark and Calderhead, Ben},
  journal={Journal of the Royal Statistical Society Series B: Statistical Methodology},
  volume={73},
  number={2},
  pages={123--214},
  year={2011},
  publisher={Oxford University Press}
}

@inproceedings{zhang2020wasserstein,
  title={Wasserstein control of mirror langevin monte carlo},
  author={Zhang, Kelvin Shuangjian and Peyr{\'e}, Gabriel and Fadili, Jalal and Pereyra, Marcelo},
  booktitle={Conference on learning theory},
  pages={3814--3841},
  year={2020},
  organization={PMLR}
}

@article{ahn2021efficient,
  title={Efficient constrained sampling via the mirror-Langevin algorithm},
  author={Ahn, Kwangjun and Chewi, Sinho},
  journal={Advances in Neural Information Processing Systems},
  volume={34},
  pages={28405--28418},
  year={2021}
}

@inproceedings{kannan2009random,
  title={Random walks on polytopes and an affine interior point method for linear programming},
  author={Kannan, Ravi and Narayanan, Hariharan},
  booktitle={Proceedings of the forty-first annual ACM symposium on Theory of computing},
  pages={561--570},
  year={2009}
}
\bibliographystyle{plainnat}

\clearpage
\appendix

\section{Connections to Related Objectives} \label{appendix:sec-connection-related-obj}
In this section, we discuss the connections between the initial objective in \autoref{eq:mpf-objective} and related formulations that have appeared in the literature. Since we started with Minimum Probability Flow (MPF), we begin with a brief overview of MPF and discuss its connection to other related objectives.

\subsection{Minimum Probability Flow Learning}
Minimum Probability Flow (MPF) \cite{icml_mpf, mpf} is a parameter estimation framework that avoids the computation of intractable partition function. MPF introduced the continuous time Markov process with dynamics that transport probability mass between states. Given transition rates $\Gamma(\x{y}, \x{x})$ and the symmetric connectivity function $g(\x{y}, \x{x})$ between states $\x{y}$ and $\x{x}$, the probability flow from state $\x{x}$ to state $\x{y}$ is
\begin{align*}
    \Gamma_{\x{\theta}}(\x{y}, \x{x}) &= g(\x{y}, \x{x})\exp\left(\frac{1}{2}\left[E_{\x{\theta}}(\x{x}) - E_{\x{\theta}}(\x{y})\right]\right)
\end{align*}
The MPF objective function measures the expected probability flow from the empirical data distribution $p$ to non-data states over infinitesimal time $\epsilon$. MPF starts from a KL divergence; a first-order Taylor expansion yields the objective in \autoref{eq:mpf-objective} (Equation $C-1$ from \citet{mpf}). Intuitively, we are interested in finding the parameter $\x{\theta}$ that minimises the probability flow from data states to non data states. The consistency of the MPF estimator under suitable regularity conditions was established by \citet{icml_mpf}, to which we refer the reader for a detailed treatment.

\subsection{Conditional Noise Contrastive Estimation} \label{appendix:sec-cond-nce-discussion}
The central idea of Noise Contrastive Estimation (NCE) \cite{pmlr-v9-gutmann10a} is to learn a classifier that distinguishes samples from the data distribution $p$ from samples of noise distribution $p_{n}$. It is known that the noise distribution $p_{n}$ must be carefully chosen to guarantee good convergence of the resulting estimator, generally considered hard in practice. To address this limitation, \citet{CondNCE} introduced conditional NCE (CondNCE) where noisy samples are generated conditionally on the observed data samples. This framework was further generalised by \citet{unified-nce} through the introduction of the $f$-CondNCE framework, based on a general convex function $f$. We consider the objective proposed by \citet{unified-nce} (cf. Eq. $4$)
\begin{align} \label{eq:nce-formulation}
    \mathcal{H}_{f}(\x{\theta}) &= \underset{\substack{\x{y} \sim p, \, \\ \x{x} \sim q(\cdot \mid \x{y})}}{\E}\left[D_{f}\left(\frac{p(\x{x})q(\x{y} \mid \x{x})}{p(\x{y})q(\x{x} \mid \x{y})}, \frac{p_{\x{\theta}}(\x{x})q(\x{y} \mid \x{x})}{p_{\x{\theta}}(\x{y})q(\x{x} \mid \x{y})}\right)\right] - \underset{\substack{\x{x} \sim p, \, \\ \x{y} \sim q(\cdot \mid \x{x})}}{\E}\left[f\left(\frac{p(\x{y})q(\x{x} \mid \x{y})}{p(\x{x})q(\x{y} \mid \x{x})}\right)\right] \nonumber \\
    &= \underset{\substack{\x{x} \sim p, \, \\ \x{y} \sim q(\cdot \mid \x{x})}}{\E}\left[-f'(\rho_{\x{\theta}}(\x{x}, \x{y})) + \rho_{\x{\theta}}(\x{y}, \x{x})f'(\rho_{\x{\theta}}(\x{y}, \x{x})) - f(\rho_{\x{\theta}}(\x{y}, \x{x}))\right]
\end{align}
where $f: \R_{\geq 0} \to \R$ is strictly convex function, $D_{f}$ denotes the Bregman divergence defined as
\begin{align*}
    D_{f}(y, x) = f(y) - f(x) - f'(x)(y - x)
\end{align*}
and $\rho_{\x{\theta}}(\x{x}, \x{y}) = \frac{p_{\x{\theta}}(\x{x})q(\x{y} \mid \x{x})}{p_{\x{\theta}}(\x{y})q(\x{x} \mid \x{y})}$. Following \citet{unified-nce}, we consider symmetric channel $q(\x{y} \mid \x{x}) = q(\x{x} \mid \x{y})$, then the objective reduces to
\begin{align} \label{eq:nce-formulation-symmetric-channel}
    \mathcal{H}^{S}_{f}(\x{\theta}) = \underset{\substack{\x{x} \sim p, \, \\ \x{y} \sim q(\cdot \mid \x{x})}}{\E}\left[-f'\left(\frac{p_{\x{\theta}}(\x{x})}{p_{\x{\theta}}(\x{y})}\right) + \frac{p_{\x{\theta}}(\x{y})}{p_{\x{\theta}}(\x{x})}f'\left(\frac{p_{\x{\theta}}(\x{y})}{p_{\x{\theta}}(\x{x})}\right) - f\left(\frac{p_{\x{\theta}}(\x{y})}{p_{\x{\theta}}(\x{x})}\right)\right]
\end{align}
Consider the above objective for $f(x) = -\sqrt{x}$, then we have
\begin{align*}
    \mathcal{H}^{S}_{f}(\x{\theta}) &= \underset{\substack{\x{x} \sim p, \, \\ \x{y} \sim q(\cdot \mid \x{x})}}{\E}\left[\sqrt{\frac{p_{\x{\theta}}(\x{y})}{p_{\x{\theta}}(\x{x})}}\right]
\end{align*}
which coincides with \autoref{eq:mpf-objective} when the connectivity function $g(\x{y}, \x{x})$ is chosen as conditional density $q(\x{y} \mid \x{x})$.
\subsection{Energy Discrepancy}
\citet{energy_discrepancy} proposed Energy Discrepancy (ED), a framework that learns the energy function directly without relying on score derivatives. ED constructs a loss based on the energy difference between data points and perturbed samples. Formally, let $\x{x}$ be a data point and $\x{y}$ be a perturbed sample sampled from conditional density $q(\cdot \mid \x{x})$. The contrastive potential $E_{q}(\x{y})$ is defined as (cf. Equation $4$ in \cite{energy_discrepancy})
\begin{align*}
    E_{q}(\x{y}) &= -\log \int_{\x{x}} q(\x{y} \mid \x{x}) \exp(-E_{\x{\theta}}(\x{x})) \d\x{x}
\end{align*}
The Energy Discrepancy objective is then defined as the expected difference between the model energy at the data point and the contrastive potential at the perturbed point
\begin{align*}
    \text{ED}_{q}(p, E_{\x{\theta}}) = \frac{1}{2}\left(\underset{{\x{x} \sim p}}{\E}\left[E_{\x{\theta}}(\x{x})\right] - \underset{\substack{\x{x} \sim p,\, \\\x{y} \sim q(\cdot \mid \x{x})}}{\E}[E_{q}(\x{y})]\right)
\end{align*}
We propose a simplifying modeling assumption by replacing the energy function corresponding to the contrastive potential induced by $q$, i.e., $E_{q}(\x{y})$, with the model energy, $E_{\x{\theta}}(\x{y})$. While this substitution is not exact, it is intuitively justified when the perturbation density $q(\x{y} \mid \x{x})$ is sufficiently {\it{localised}} and $E_{\x{\theta}}$ varies smoothly. In this regime, the log-sum-exp integral defining $E_{q}(\x{y})$ is dominated by contributions from $\x{x}$ in the immediate vicinity of $\x{y}$. Consequently, $E_{q}(\x{y})$ can be locally approximated by $E_{\x{\theta}}(\x{y})$, i.e., $E_{q}(\x{y}) \approx E_{\x{\theta}}(\x{y})$. We define
\begin{align*}
    \hat{\text{ED}}_{q}(p, E_{\x{\theta}}) = \underset{\substack{\x{x} \sim p,\, \\\x{y} \sim q(\cdot \mid \x{x})}}{\E} \left[\frac{1}{2} \big(E_{\x{\theta}}(\x{x}) - E_{\x{\theta}}(\x{y})\big)\right]
\end{align*}
By invoking the convexity of the exponential function and Jensen's inequality, we obtain an upper bound on $\hat{\text{ED}}_q$, that is
\begin{align*}
    \text{exp}\left(\hat{\text{ED}}_{q}(p, E_{\x{\theta}})\right) \leq \underset{\substack{\x{x} \sim p, \, \\ \x{y} \sim q(\cdot \mid \x{x})}}{\E}\left[\exp\left(\frac{1}{2}[E_{\x{\theta}}(\x{x}) - E_{\x{\theta}}(\x{y})]\right)\right]
\end{align*}
Rather than minimising the left-hand side directly, we consider minimising this upper bound as a surrogate objective. Observe that this surrogate coincides with \autoref{eq:mpf-objective} when the connectivity function $g(\x{y}, \x{x})$ is chosen as the conditional density $q(\x{y} \mid \x{x})$, corresponding to the soft neighbourhood case. A factor of $\frac{1}{2}$ was introduced in the definition of energy discrepancy to show the connection to MPF objective explicit.

\section{Supporting Lemmas}
We will first establish supporting lemmas that we will use in the proof of stated theorems and propositions.
\begin{lemma} \label{lemma:compact-cover-of-S}
    Suppose $\Omega \subset \R^{d}$ is open and $S \subset \Omega$ is compact. Then there exists $\eta > 0$ such that the set
    \begin{align*}
        K = \left\{\x{z} \in \R^{d} : \inf_{\x{x} \in S} \norm{\x{z} - \x{x}} \leq \eta\right\}
    \end{align*}
    is a subset of $\Omega$. Furthermore $K$ is compact and $S \subset K$.
\end{lemma}
\begin{proof}
    As $\Omega$ is open, for each point $\x{x} \in S$, there exists $r({\x{x}}) > 0$ such that $N_{r({\x{x}})}(\x{x}) = \{\x{z} \in \R^{d} : \norm{\x{z} - \x{x}} < r({\x{x}})\} \subset \Omega$. Observe,
    \begin{align*}
        S \subset \bigcup_{\x{x} \in S} N_{\frac{r(\x{x})}{2}}(\x{x}) \subset \Omega
    \end{align*}
    As $S$ is compact, we have
    \begin{align*}
        \exists \: n \in \mathbb{N}, \x{x}_{1},\dots,\x{x}_{n} \in S \text{ such that } S \subset \bigcup_{i=1}^{n} N_{\frac{r(\x{x}_{i})}{2}}(\x{x}_{i})
    \end{align*}
    Define
    \begin{align*}
        \eta = \frac{1}{4}\min\left\{r(\x{x}_{1}),\dots,r(\x{x}_{n})\right\} > 0
    \end{align*}
    and the set
    \begin{align*}
        K = \left\{\x{z} \in \R^{d} : \inf_{\x{x} \in S} \norm{\x{z} - \x{x}} \leq \eta\right\}
    \end{align*}
    Take any $\x{z} \in K$, then $\inf_{\x{x} \in S} \norm{\x{z} - \x{x}} \leq \eta$. As $S$ is compact, there exists $\x{x}' \in S$ such that $\norm{\x{z} - \x{x}'} = \inf_{\x{x} \in S} \norm{\x{z} - \x{x}}$ (infimum is attained). As $\x{x}' \in S$, by above
    \begin{align*}
        \exists \: 1 \leq i \leq n : \norm{\x{x}' - \x{x}_{i}} < \frac{r(\x{x}_{i})}{2}
    \end{align*}
    Now,
    \begin{align*}
        \norm{\x{z} - \x{x}_{i}} &\leq \norm{\x{z} - \x{x}'} + \norm{\x{x}'  - \x{x}_{i}} \\
        &< \eta + \frac{r(\x{x}_{i})}{2} \\
        &\leq \frac{r(\x{x}_{i})}{4} + \frac{r(\x{x}_{i})}{2} \\
        &= \frac{3r(\x{x}_{i})}{4} \\
        &< r(\x{x}_{i})
    \end{align*}
    $\implies \x{z} \in N_{r(\x{x}_{i})}(\x{x}_{i}) \implies \x{z} \in \Omega$. So $K \subset \Omega$. To show that $K$ is compact, define $\rho_{S} : \R^{d} \to \R$ as
    \begin{align*}
        \rho_{S}(\x{z}) = \inf_{\x{x} \in S} \norm{\x{z} - \x{x}}
    \end{align*}
    Observe that $K = \rho_{S}^{-1}([0, \eta])$ (follows from definition of $K$). As $\rho_{S}$ is continuous and $[0, \eta]$ is closed, we conclude that $K$ is closed. For any $\x{z} \in K$, there exists $\x{x}' \in S$ such that
    \begin{align*}
        \norm{\x{z} - \x{x}'} = \inf_{\x{x} \in S} \norm{\x{z} - \x{x}} \leq \eta
    \end{align*}
    Since $S$ is compact, it is bounded, so $\norm{\x{x}'} \leq b$ for some $b > 0$. Observe that
    \begin{align*}
        \norm{\x{z}} &\leq \norm{\x{z} - \x{x}'} + \norm{\x{x}'} \\
        &\leq \eta + b
    \end{align*}
    which shows that $K$ is bounded. As $K$ is closed and bounded, by Heine-Borel, compactness of $K$ follows. It is easy to see that $S \subset K$.
\end{proof}
\begin{lemma} \label{lemma:smooth-function-integral-0-func-0}
    Let $\Omega \subset \R^{d}$ be open, and write $C^{\infty}_{c}(\Omega)$ for the set of functions in $C^{\infty}$ defined on $\Omega$ whose support is a compact subset of $\Omega$. If $h : \Omega \to \R$ is continuous and
    \begin{align*}
        \int_{\Omega}\psi(\xx)h(\xx)\d\xx = 0 \qquad \text{for every } \psi \in C^{\infty}_{c}(\Omega)
    \end{align*}
    then $h(\xx) = 0$ for all $\xx \in \Omega$.
\end{lemma}
\begin{proof}
    We first construct a smooth bump. Let $v : \R \to \R$ be given by $v(t) = \exp(-1/t)$ for $t > 0$ and $v(t) = 0$ for $t \leq 0$. It is standard that $v \in C^{\infty}(\R)$: for $t > 0$ every derivative of $v$ has the form $P(1/t)\exp(-1/t)$ for a polynomial $P$, which tends to $0$ as $t \downarrow 0$. Define
    \begin{align*}
        \varphi(\x{z}) = v\left(1 - \norm{\x{z}}^{2}\right), \qquad \x{z} \in \R^{d}
    \end{align*}
    As the composition of $v$ with a polynomial, $\varphi$ is $C^{\infty}$ on $\R^{d}$. Moreover $\varphi \geq 0$, $\varphi(\x{z}) > 0$ when $\norm{\x{z}} < 1$, and $\varphi(\x{z}) = 0$ when $\norm{\x{z}} \geq 1$.

    Now suppose, for contradiction, that $h(\x{x}_{0}) \neq 0$ for some $\x{x}_{0} \in \Omega$. Replacing $h$ by $-h$ if necessary, we may assume $h(\x{x}_{0}) > 0$. Since $\Omega$ is open and $h$ is continuous at $\x{x}_{0}$, there is $\varepsilon > 0$ such that the closed ball $\overline{B}(\x{x}_{0}, \varepsilon) = \{\x{z} : \norm{\x{z} - \x{x}_{0}} \leq \varepsilon\}$ is contained in $\Omega$ and $h(\x{z}) > \frac{1}{2}h(\x{x}_{0})$ for every $\x{z}$ in it. Put
    \begin{align*}
        \psi(\x{z}) = \varphi\left(\frac{\x{z} - \x{x}_{0}}{\varepsilon}\right)
    \end{align*}
    Then $\psi$ is $C^{\infty}$, nonnegative, and vanishes outside $\overline{B}(\x{x}_{0}, \varepsilon)$, which is a compact subset of $\Omega$; so $\psi \in C^{\infty}_{c}(\Omega)$. Since $\psi \geq 0$ and $h > \frac{1}{2}h(\x{x}_{0})$ wherever $\psi \neq 0$,
    \begin{align*}
        \int_{\Omega}\psi(\x{z})h(\x{z})\d\x{z} \geq \frac{1}{2}h(\x{x}_{0})\int_{\Omega}\psi(\x{z})\d\x{z} = \frac{1}{2}h(\x{x}_{0})\,\varepsilon^{d}\int_{\R^{d}}\varphi(\x{u})\d\x{u} > 0
    \end{align*}
    where we substituted $\x{u} = (\x{z} - \x{x}_{0})/\varepsilon$, and the last integral is positive because $\varphi$ is continuous, nonnegative, and positive on the open unit ball. This contradicts the hypothesis, so $h = 0$ on $\Omega$.
\end{proof}
\begin{lemma} \label{lemma:weak-divergence}
    Let $\Omega \subset \R^{d}$ be open, let $D : \Omega \to \R^{d \times d}$ be $C^{1}$ with $D(\xx)$ symmetric, and let $\x{b} : \Omega \to \R^{d}$ be continuous. If
    \begin{align} \label{eq:weak-divergence-hypothesis}
        \int_{\Omega}\psi(\xx)\x{b}(\xx)\d\xx = -\frac{1}{2}\int_{\Omega}D(\xx)\nabla\psi(\xx)\d\xx \qquad \text{for every } \psi \in C^{\infty}_{c}(\Omega)
    \end{align}
    then $\x{b} = \frac{1}{2}\nabla \cdot D$ on $\Omega$.
\end{lemma}
\begin{proof}
    Fix $\psi \in C^{\infty}_{c}(\Omega)$ with support $S \subset \Omega$, and fix $i, j \in \{1, \dots, d\}$. We first show that
    \begin{align} \label{eq:ibp-compact-support}
        \int_{\Omega}D_{ij}(\xx)\frac{\partial\psi}{\partial x_{j}}(\xx)\d\xx = -\int_{\Omega}\psi(\xx)\frac{\partial D_{ij}}{\partial x_{j}}(\xx)\d\xx
    \end{align}
    Define $F : \R^{d} \to \R$ as follows
    \begin{align*}
        F(\x{x}) = \begin{cases}
            D_{ij}(\x{x})\psi(\x{x}) & \text{if } \x{x} \in \Omega \\
            0 & \text{otherwise}
        \end{cases}
    \end{align*}
    Then $F$ is $C^{1}$ on $\R^{d}$. To see this, observe the following
    \begin{enumerate}
        \item Take any $\x{x} \in \Omega$ and as $\Omega$ is open, there exists a neighbourhood $N$ of $\x{x}$ such that $\x{x} \in N \subset \Omega$. We can write
        \begin{align*}
            F(\x{y}) = D_{ij}(\x{y})\psi(\x{y}) \quad \forall \: \x{y} \in N
        \end{align*}
        Now $D_{ij}$ and $\psi$ are differentiable at every point of $N$ which implies $F$ is differentiable on $N$. By product rule, we have
        \begin{align*}
            \nabla F(\x{y}) = D_{ij}(\x{y})\nabla \psi(\x{y}) + \nabla D_{ij}(\x{y})\psi (\x{y}) \quad \forall \: \x{y} \in N
        \end{align*}
        Because $D_{ij},\nabla D_{ij}, \psi, \nabla \psi$ are continuous on $\Omega$, they are continuous on $N$ as well. This shows that $\nabla F$ is continuous on $N$.
        \item Suppose $\x{x} \notin \Omega \implies \x{x} \notin S$. Since $S^{c}$ is open, there exists a neighbourhood $N'$ of $\x{x}$ such that $\x{x} \in N' \subset S^{c}$. For any $\x{y} \in N'$, either $\x{y} \in \Omega$ or $\x{y} \notin \Omega$. If $\x{y} \in \Omega$, then $F(\x{y}) = 0$ because $\psi(\x{y}) = 0$. For $\x{y} \notin \Omega$, $F(\x{y}) = 0$ by definition. Summarising, we have
        \begin{align*}
            F(\x{y}) = 0 \quad \forall \: \x{y} \in N'
        \end{align*}
        As $F$ is constant on an open ball, $F$ is differentiable on $N'$. Also
        \begin{align*}
            \nabla F(\x{y}) = 0 \quad \forall \: \x{y} \in N'
        \end{align*}
        implying $\nabla F$ is continuous on $N'$ (because it is constant).
    \end{enumerate}
    Since $S$ is bounded, we can choose $R > 0$ with $S \subset (-R, R)^{d}$, then $F$ and $\frac{\partial F}{\partial x_{j}}$ vanish outside $(-R, R)^{d}$ (as seen above, they vanish outside $S$). Write $\xx = (x_{j}, \hat{\xx})$, where $\hat{\xx} \in \R^{d-1}$ collects the remaining coordinates. For each fixed $\hat{\xx}$, the fundamental theorem of calculus gives
    \begin{align*}
        \int_{-R}^{R}\frac{\partial F}{\partial x_{j}}(x_{j}, \hat{\xx})\d x_{j} = F(R, \hat{\xx}) - F(-R, \hat{\xx}) = 0
    \end{align*}
    since both endpoints lie outside $(-R, R)^{d}$. As $\frac{\partial F}{\partial x_{j}}$ is continuous with compact support, it is integrable, and Fubini's theorem gives $\int_{\R^{d}}\frac{\partial F}{\partial x_{j}}\d\xx = 0$. On $\Omega$ the product rule gives $\frac{\partial F}{\partial x_{j}} = \psi\frac{\partial D_{ij}}{\partial x_{j}} + D_{ij}\frac{\partial\psi}{\partial x_{j}}$, and outside $\Omega$ both $F$ and its derivative vanish, so
    \begin{align*}
        0 = \int_{\Omega}\psi(\xx)\frac{\partial D_{ij}}{\partial x_{j}}(\xx)\d\xx + \int_{\Omega}D_{ij}(\xx)\frac{\partial\psi}{\partial x_{j}}(\xx)\d\xx
    \end{align*}
    Both integrals are finite, their integrands being continuous and vanishing outside $S$. This is \autoref{eq:ibp-compact-support}.

    Now take the $i$-th component of \autoref{eq:weak-divergence-hypothesis}, whose right-hand side is $-\frac{1}{2}\sum_{j}\int_{\Omega}D_{ij}\frac{\partial\psi}{\partial x_{j}}\d\xx$, and apply \autoref{eq:ibp-compact-support} to each term:
    \begin{align*}
        \int_{\Omega}\psi(\xx)\x{b}_{i}(\xx)\d\xx = \frac{1}{2}\int_{\Omega}\psi(\xx)\sum_{j=1}^{d}\frac{\partial D_{ij}}{\partial x_{j}}(\xx)\d\xx = \frac{1}{2}\int_{\Omega}\psi(\xx)(\nabla \cdot D)_{i}(\xx)\d\xx
    \end{align*}
    where the last step uses the symmetry of $D$, which gives $\sum_{j}\frac{\partial D_{ij}}{\partial x_{j}} = \sum_{j}\frac{\partial D_{ji}}{\partial x_{j}} = (\nabla \cdot D)_{i}$. Hence
    \begin{align*}
        \int_{\Omega}\psi(\xx)\left(\x{b}_{i}(\xx) - \frac{1}{2}(\nabla \cdot D)_{i}(\xx)\right)\d\xx = 0 \qquad \text{for every } \psi \in C^{\infty}_{c}(\Omega)
    \end{align*}
    The function $\x{b}_{i} - \frac{1}{2}(\nabla \cdot D)_{i}$ is continuous, since $\x{b}$ is continuous and $D$ is $C^{1}$, so \autoref{lemma:smooth-function-integral-0-func-0} shows it vanishes identically. As this holds for each $i$, $\x{b} = \frac{1}{2}\nabla \cdot D$ on $\Omega$.
\end{proof}

\begin{lemma} \label{lemma:b-is-divergence}
    Suppose assumptions of \autoref{theorem:sm-convex} hold for all $\x{x} \in \Omega$, then
    \begin{align*}
        \x{b}(\x{x}) = \frac{1}{2}(\nabla \cdot D)(\x{x})
    \end{align*}
\end{lemma}
\begin{proof}
    In this proof, we will use the fact that $g$ is symmetric extensively. Let $\psi : \Omega \to \R$ be any function such that $\psi \in C^{\infty}$ and it has compact support $S \subset \Omega$. Put
    \begin{align*}
        l_{\psi} = \sup_{\x{z} \in \Omega} \abs{\psi(\x{z})}, k_{\psi} = \sup_{\x{z} \in \Omega} \norm{\nabla \psi(\x{z})}, t_{\psi} = \sup_{\x{z} \in \Omega} \norm{H_{\psi}(\x{z})}
    \end{align*}
    which are finite since $\psi$ and its derivatives become $0$ outside $S$ and because they are continuous on compact set $S$. From \autoref{lemma:compact-cover-of-S}, there exists $\eta > 0$ such that the set defined as follows
    \begin{align*}
        K = \{\x{z} \in \R^{d} : \inf_{\x{s} \in S} \norm{\x{z} - \x{s}} \leq \eta\}
    \end{align*}
    satisfies $S \subset K \subset \Omega$ and $K$ is compact. Applying \ref{cond:delta-go-to-0-and-g-0} on $S$, there exists $\delta_{r}(S) > 0$ such that $\lim_{r \to 0} \delta_{r}(S) = 0$. By definition of limit,
    \begin{align*}
        \exists \: r_{1} > 0 \text{ such that } \forall \: 0 < r < r_{1} \text{ we have } \delta_{r}(S) < \eta
    \end{align*}
    Further, applying \ref{cond:symmetric-g-and-c-bounded} on $S$, there exists $r_{S}$ such that $\C(\cdot; r)$ is bounded for $0 < r < r_{S}$. From now on, we take $r$ such that $0 < r < \min(r_{1}, r_{S})$. We observe the following
    \begin{quote}
        \textbf{Claim - 1} : If $g_{r}(\x{y}, \x{x}) > 0$ and if $\x{x}$ lies in $S$, then $\norm{\x{y} - \x{x}} \leq \delta_{r}(S)$ and $\x{y}$ lies in $K$.
        \begin{proof}
            If $\norm{\x{y} - \x{x}} > \delta_{r}(S)$, then by \ref{cond:delta-go-to-0-and-g-0}, we have $g_{r}(\x{y}, \x{x}) = 0$ which is a contradiction. So $\norm{\x{y} - \x{x}} \leq \delta_{r}(S)$. From above, $\delta_{r}(S) < \eta$ implies $\norm{\x{y} - \x{x}} \leq \eta$. As $\x{x} \in S$, this means
            \begin{align*}
                \inf_{\x{z} \in S} \norm{\x{y} - \x{z}} \leq \eta \implies y \in K
            \end{align*}
        \end{proof}
        \textbf{Claim - 2} : If $g_{r}(\x{y}, \x{x}) > 0$ and if $\x{y}$ lies in $S$, then $\norm{\x{y} - \x{x}} \leq \delta_{r}(S)$ and $\x{x}$ lies in $K$.
        \begin{proof}
            This follows from \textbf{Claim - 1} by interchanging $\x{x}$ and $\x{y}$ and using the fact that $g_{r}(\x{x}, \x{y}) = g_{r}(\x{y}, \x{x})$.
        \end{proof}
    \end{quote}
    Define the integral
    \begin{align*}
        \T_{r} &= \int_{\x{x} \in \Omega}\psi(\x{x})\x{m}_{1, r}(\x{x}) \d \x{x} \\
        &= \int_{\x{x} \in \Omega}\int_{\x{y} \in \Omega} \psi(\x{x})g_{r}(\x{y}, \x{x})(\x{y} - \x{x}) \d \x{y} \d \x{x}
    \end{align*}
    We will first show that integrand in $\T_{r}$ is absolutely integrable.
    \begin{quote}
        From the first equality, the integrand of $\T_{r}$ is $0$ when $\x{x} \notin S$ since $\psi$ has support $S$. From the second equality, only $\x{y}$ such that $g_{r}(\x{y}, \x{x}) > 0$ contribute to $\T_{r}$. From \textbf{Claim - 1} above, it means $\norm{\x{y} - \x{x}} \leq \delta_{r}(S)$. So
        \begin{align*}
            \int_{\x{x} \in \Omega}\int_{\x{y} \in \Omega} g_{r}(\x{y}, \x{x})\abs{\psi(\x{x})}\norm{\x{y} - \x{x}} \d \x{y} \d \x{x} &= \underset{\x{x} \in S}{\int} \: \underset{\x{y}: \norm{\x{y} - \x{x}} \leq \delta_{r}(S)}{\int} g_{r}(\x{y}, \x{x})\abs{\psi(\x{x})}\abs{\norm{y - x}} \d \x{y} \d \x{x} \\
            &\leq l_{\psi}\delta_{r}(S)\int_{\x{x} \in S}\C(\x{x};r) \d \x{x}
        \end{align*}
        As $r < r_{S}$, $\mathcal{C}(\x{x}; r)$ is bounded implying that the integrand in $\T_{r}$ is absolutely integrable.
    \end{quote}
    We rewrite $\T_{r}$ as follows
    \begin{quote}
        Interchanging variables in definition $\T_{r}$, we have
        \begin{align*}
            \T_{r} = \int_{\x{y} \in \Omega}\int_{\x{x} \in \Omega} \psi(\x{y})g_{r}(\x{x}, \x{y})(\x{x} - \x{y}) \d \x{x} \d \x{y}
        \end{align*}
        As $\T_{r}$ is absolutely integrable, by Fubini's theorem, we can change the order of integral. Also using the fact that $g$ is symmetric, we have
        \begin{align*}
            \T_{r} = -\int_{\x{x} \in \Omega}\int_{\x{y} \in \Omega} \psi(\x{y})g_{r}(\x{y}, \x{x})(\x{y} - \x{x}) \d \x{y} \d \x{x}
        \end{align*}
    \end{quote}
    Adding two different expressions of $\T_{r}$ and dividing by $2$ gives
    \begin{align*}
        \T_{r} = \frac{1}{2}\int_{\x{x} \in \Omega}\int_{\x{y} \in \Omega} [\psi(\x{x}) - \psi(\x{y})]g_{r}(\x{y}, \x{x})(\x{y} - \x{x}) \d \x{y} \d \x{x}
    \end{align*}
    As $\Omega$ is a convex set, the line segment joining $\x{x}, \x{y} \in \Omega$ lies in $\Omega$. So applying Taylor's theorem in integral form, we have
    \begin{align*}
        \psi(\x{y}) = \psi(\x{x}) + \nabla \psi(\x{x})^{\top}(\x{y} - \x{x}) + \rho(\x{x}, \x{y})
    \end{align*}
    where
    \begin{align*}
        \rho(\x{x}, \x{y}) = \int_{0}^{1} (\x{y} - \x{x})^{\top}H_{\psi}(\x{x} + t(\x{y} - \x{x}))(\x{y} - \x{x}) (1 - t) \d t
    \end{align*}
    Putting this back in $\T_{r}$,
    \begin{align*}
        \T_{r} &= -\frac{1}{2}\int_{\x{x} \in \Omega}\int_{\x{y} \in \Omega}\left[\nabla \psi(\x{x})^{\top}(\x{y} - \x{x})\right]g_{r}(\x{y}, \x{x})(\x{y} - \x{x})\d \x{y} \d \x{x} \\
        &\quad -\frac{1}{2}\int_{\x{x} \in \Omega}\int_{\x{y} \in \Omega}\rho(\x{x}, \x{y})g_{r}(\x{y}, \x{x})(\x{y} - \x{x}) \d \x{y} \d \x{x} \\
        &= -\frac{1}{2}\int_{\x{x} \in \Omega}M_{2, r}(\x{x}) \nabla \psi(\x
        x) \d \x{x} - \frac{1}{2}\underbrace{\int_{\x{x} \in \Omega}\int_{\x{y} \in \Omega}\rho(\x{x}, \x{y})g_{r}(\x{y}, \x{x})(\x{y} - \x{x}) \d \x{y} \d \x{x}}_{\mathcal{E}_{r}}
    \end{align*}
    The idea is now to divide by $q(r)$ and take the limit $r \to 0$. We will first show that second term goes to $0$ when we take the limit.
    \begin{quote}
        We can write $\rho(\x{x}, \x{y}) = \psi(\x{y}) - \psi(\x{x}) - \nabla \psi(\x{x})^{\top}(\x{y} - \x{x})$. Suppose $\x{x}, \x{y} \notin S$, then $\psi(\x{x}) = \psi(\x{y}) = 0$ and $\nabla \psi(\x{x}) = 0$ implying $\rho(\x{x}, \x{y}) = 0$. So the integrand of $\mathcal{E}_{r}$ is $0$ unless $\x{x} \in S$ or $\x{y} \in S$. Consider the cases
        \begin{enumerate}
            \item If $\x{x} \in S$, we need $\x{y}$ such that $g_{r}(\x{y}, \x{x}) > 0$ which implies $\norm{\x{y} - \x{x}} \leq \delta_{r}(S)$ by \textbf{Claim - 1}. As $S \subset K$, we have $\x{x} \in K$.
            \item If $\x{x} \notin S$, then we need $\x{y} \in S$. Similar to case $1$, we also need $g_{r}(\x{y}, \x{x}) > 0$. By \textbf{Claim - 2}, we have $\x{x} \in K$ and $\norm{\x{y} - \x{x}} \leq \delta_{r}(S)$.
        \end{enumerate}
        From both cases, we have $\x{x} \in K$ and $\x{y}$ such that $\norm{\x{y} - \x{x}} \leq \delta_{r}(S)$. Now,
        \begin{align*}
            \mathcal{E}_{r} = \int_{\x{x} \in K} \int_{\x{y} : \norm{\x{y} - \x{x}} \leq \delta_{r}(S)} \rho(\x{x}, \x{y})g_{r}(\x{y}, \x{x})(\x{y} - \x{x}) \d \x{y} \d \x{x}
        \end{align*}
        Observe that $\abs{\rho(\x{x}, \x{y})} \leq \frac{1}{2}t_{\psi}\norm{\x{y} - \x{x}}^{2}$. So
        \begin{align*}
            \frac{1}{q(r)}\norm{\mathcal{E}_{r}} &\leq \frac{1}{2}t_{\psi}\delta_{r}(S)\int_{\x{x} \in K}\int_{\x{y} : \norm{\x{y} - \x{x}} \leq \delta_{r}(S)}\frac{1}{q(r)}g_{r}(\x{y}, \x{x}) \norm{\x{y} - \x{x}}^{2} \d \x{y} \d \x{x} \\
            &\leq \frac{1}{2}t_{\psi}\delta_{r}(S)\int_{\x{x} \in K}\int_{\x{y} \in \Omega}\frac{1}{q(r)}g_{r}(\x{y}, \x{x}) \norm{\x{y} - \x{x}}^{2} \d \x{y} \d \x{x} \\
            &= \frac{1}{2}t_{\psi}\delta_{r}(S)\int_{\x{x} \in K}\frac{1}{q(r)}\tr(M_{2, r}(\x{x})) \d \x{x}
        \end{align*}
        Applying \ref{cond:uni-conv-second-moment} to $K$, we have
        \begin{align*}
            \exists \: c > 0 \text{ such that } \forall \: 0 < r < c, \x{x} \in K \text{ we have } \norm{\frac{1}{q(r)}M_{2, r}(\x{x}) - D(\x{x})} < 1
        \end{align*}
        For any $x \in K$ and $0 < r < c$, we have
        \begin{align*}
            \frac{1}{q(r)}\tr(M_{2, r}(\x{x})) &\leq \frac{d}{q(r)}\norm{M_{2, r}(\x{x})} \\
            &< d(1 + \norm{D(\x{x})})
        \end{align*}
        Let $\Delta = \sup_{\x{x} \in K}\norm{D(\x{x})}$, then $\Delta$ is finite because $D$ is continuous and $K$ is compact. Putting this back, we have
        \begin{align*}
            \frac{1}{q(r)}\norm{\mathcal{E}_{r}} \leq \frac{d}{2}t_{\psi}\delta_{r}(S)(1 + \Delta)\text{Vol}(K)
        \end{align*}
        where $\text{Vol}(K)$ denotes the volume of $K$ which is finite because $K$ is bounded. Now as $r \to 0$, $\delta_{r}(S) \to 0$ implying
        \begin{align*}
            \lim_{r \to 0} \frac{1}{q(r)}\mathcal{E}_{r} = \x{0}
        \end{align*}
    \end{quote}
    Coming back to $\T_{r}$ and dividing by $q(r)$,
    \begin{align*}
        \frac{1}{q(r)}\T_{r} = -\frac{1}{2}\int_{\x{x} \in S}\frac{1}{q(r)}M_{2,r}(\x{x})\nabla\psi(\x{x})\d\x{x} - \frac{1}{2}\cdot\frac{1}{q(r)}\mathcal{E}_{r}
    \end{align*}
    where we have used that $\nabla\psi(\x{x}) = \x{0}$ for $\x{x} \notin S$. We take the limit $r \to 0$ of each term. In RHS, for the first term, \ref{cond:uni-conv-second-moment} applied to $S$ gives
    \begin{align*}
        \norm{\int_{\x{x} \in S}\left(\frac{1}{q(r)}M_{2,r}(\x{x}) - D(\x{x})\right)\nabla\psi(\x{x})\d\x{x}} &\leq k_{\psi}\text{Vol}(S)\sup_{\x{x} \in S}\norm{\frac{1}{q(r)}M_{2,r}(\x{x}) - D(\x{x})} \\
        &\longrightarrow 0
    \end{align*}
    and the second term tends to $\x{0}$ as shown above. For the LHS, since $\psi$ vanishes outside $S$, $\frac{1}{q(r)}\T_{r} = \int_{\x{x} \in S}\psi(\x{x})\frac{1}{q(r)}\x{m}_{1,r}(\x{x})\d\x{x}$, and \ref{cond:uni-conv-first-moment} applied to $S$ gives
    \begin{align*}
        \norm{\int_{\x{x} \in S}\psi(\x{x})\left(\frac{1}{q(r)}\x{m}_{1,r}(\x{x}) - \x{b}(\x{x})\right)\d\x{x}} &\leq l_{\psi}\text{Vol}(S)\sup_{\x{x} \in S}\norm{\frac{1}{q(r)}\x{m}_{1,r}(\x{x}) - \x{b}(\x{x})}  \\
        &\longrightarrow 0
    \end{align*}
    Letting $r \to 0$, and extending both integrals to $\Omega$, which changes nothing since the integrands vanish outside $S$,
    \begin{align*}
        \int_{\x{x} \in \Omega}\psi(\x{x})\x{b}(\x{x})\d\x{x} = -\frac{1}{2}\int_{\x{x} \in \Omega}D(\x{x})\nabla\psi(\x{x})\d\x{x}
    \end{align*}
    Since $\psi \in C^{\infty}_{c}(\Omega)$ was arbitrary, \autoref{lemma:weak-divergence} applies, and $\x{b}(\x{x}) = \frac{1}{2}(\nabla \cdot D)(\x{x})$ for all $\x{x} \in \Omega$.
\end{proof}

\begin{lemma} \label{lemma:divergence-on-convex}
    Let $\Omega \subset \R^{d}$ be a nonempty bounded open convex set, and let $\x{f} : \Omega \to \R^{d}$ be $C^{1}$ with $\nabla \cdot \x{f}$ to be absolutely integrable. Denote $\x{n}(\x{z})$ the outward unit normal at $\x{z} \in \partial \Omega$. Suppose that for all sequences $\{\x{z}_{j}\}$ in $\partial\Omega$ at which $\x{n}(\x{z}_{j})$ exists and $\{\x{x}_{j}\}$ in $\Omega$ with $\norm{\x{x}_{j} - \x{z}_{j}} \to 0$, we have $\x{n}(\x{z}_{j})^{\top}\x{f}(\x{x}_{j}) \to 0$. Then
    \begin{align*}
        \int_{\Omega}(\nabla \cdot \x{f})(\x{x})\d\x{x} = 0
    \end{align*}
\end{lemma}
\begin{proof}
    Fix $\x{x}_{0} \in \Omega$, and for $s \in (0, 1)$ let
    \begin{align*}
        \x{h}_{s}(\x{z}) = \x{x}_{0} + s(\x{z} - \x{x}_{0}) \quad \forall \: \x{z} \in \R^{d}
    \end{align*}
    that is $\x{h}_{s}$ will give a point between $\x{x}_{0}$ and $\x{z}$. Let $\Omega_{s} = \x{h}_{s}(\Omega)$. We observe the following
    \begin{enumerate}
        \item It's easy to see that $\x{h}_{s}$ is homeomorphism. As $\Omega$ is open, $\Omega_{s}$ is also open.
        \item As $\overline{\Omega}$ is closed, $\x{h}_{s}(\overline{\Omega})$ is closed as well. Now, $\Omega \subset \overline{\Omega}$, implies $\Omega_{s} = \x{h}_{s}(\Omega) \subset \x{h}_{s}(\overline{\Omega})$. As $\x{h}_{s}(\overline{\Omega})$ is closed, we conclude that $\overline{\Omega_{s}} \subset \x{h}_{s}(\overline{\Omega})$. Similarly, we know that $\Omega_{s} \subset \overline{\Omega_{s}}$ implying
        \begin{align*}
            \x{h}_{s}^{-1}(\Omega_{s}) \subset \x{h}_{s}^{-1}\left(\overline{\Omega_{s}}\right) \implies \Omega \subset \x{h}_{s}^{-1}\left(\overline{\Omega_{s}}\right) \implies \overline{\Omega} \subset \x{h}_{s}^{-1}\left(\overline{\Omega_{s}}\right) \implies \x{h}_{s}(\overline{\Omega}) \subset \overline{\Omega_{s}}
        \end{align*}
        where we have used the fact that $\x{h}_{s}^{-1}(\overline{\Omega_{s}})$ is closed. So $\overline{\Omega_{s}} = \x{h}_{s}(\overline{\Omega})$.
        \item Since $\Omega$ and $\Omega_{s}$ are open, we have
        \begin{align*}
            \partial \Omega = \overline{\Omega} \setminus \Omega \quad \& \quad \partial \Omega_{s} = \overline{\Omega_{s}} \setminus \Omega_{s}
        \end{align*}
        Using the properties of bijection, we have
        \begin{align*}
            \x{h}_{s}(\partial \Omega) &= \x{h}_{s}\left(\overline{\Omega} \setminus \Omega\right) \\
            &=\x{h}_{s}\left(\overline{\Omega}\right) \setminus \x{h}_{s}(\Omega) \\
            &= \overline{\Omega_{s}} \setminus \Omega_{s} = \partial \Omega_{s}
        \end{align*}
        \item From \citet[Theorem 6.1]{rockafellar1997convex}, we have $\x{h}_{s}(\overline{\Omega}) \subset \Omega$ implying $\overline{\Omega_{s}} \subset \Omega$.
        \item Further, if $\x{x}_{1},\x{x}_{2} \in \Omega_{s}$, then there exists $\x{z}_{1},\x{z}_{2} \in \Omega$ such that $\x{x}_{1} = \x{h}_{s}(\x{z}_{1}), \x{x}_{2} = \x{h}_{s}(\x{z}_{2})$. Consider for any $\lambda \in [0, 1]$
        \begin{align*}
            \lambda \x{x}_{1} + (1-\lambda) \x{x}_{2} &= \x{x}_{0} + s(\lambda\x{z}_{1} + (1 - \lambda)\x{z}_{2} - \x{x}_{0}) \\
            &= \x{h}_{s}(\lambda\x{z}_{1} + (1 - \lambda)\x{z}_{2})
        \end{align*}
        As $\Omega$ is convex implies $\lambda\x{z}_{1} + (1 - \lambda)\x{z}_{2} \in \Omega$. Therefore, $\lambda \x{x}_{1} + (1-\lambda) \x{x}_{2} \in \Omega_{s}$ concluding that $\Omega_{s}$ is convex.
    \end{enumerate}
    Being $\Omega_{s} \subset \Omega$, bounded and convex, $\Omega_{s}$ has Lipschitz boundary (\citet[Lemma 1.13, Chapter 2]{simon2014introduction}), and $\x{f}$ is $C^{1}$ on $\overline{\Omega_{s}}$ (point 4 above), applying divergence theorem gives
    \begin{align} \label{eq:divergence-theorem-on-omega-s}
        \int_{\Omega_{s}}(\nabla \cdot \x{f})(\x{x})\d\x{x} = \int_{\partial\Omega_{s}}\x{n}_{s}(\x{w})^{\top}\x{f}(\x{w})\d\mathcal{A}^{d-1}(\x{w}) = s^{d-1}\int_{\partial\Omega}\x{n}(\x{z})^{\top}\x{f}(\x{h}_{s}(\x{z}))\d\mathcal{A}^{d-1}(\x{z})
    \end{align}
    where $\x{n}_{s}$ is outward normal of $\Omega_{s}$, $\mathcal{A}^{d-1}$ is surface element and the second equality uses that $\x{h}_{s}$ maps $\partial\Omega$ onto $\partial\Omega_{s}$ (point 3 above), preserves outward normals, and scales surface measure by $s^{d-1}$.

    We now show the right-hand side tends to $0$ as $s \uparrow 1$. Let
    \begin{align*}
        \sigma(s) := \sup\left\{\abs{\x{n}(\x{z})^{\top}\x{f}(\x{h}_{s}(\x{z}))} : \x{z} \in \partial\Omega,\ \x{n}(\x{z}) \text{ exists}\right\}
    \end{align*}
    which is finite since $\x{f} \circ \x{h}_{s}$ is continuous on the compact set $\partial\Omega$.

    Assume that $\sigma(s)$ doesn't converge to $0$, then there exists $\epsilon > 0$ and sequence $\{s_{j}\}$, $s_{j} \uparrow 1$, such that $\sigma(s_{j}) \geq \epsilon$. As $\sigma$ is supremum, we would have
    \begin{align*}
        \exists \: \x{z}_{j} \in \partial \Omega, \x{n}(\x{z}_{j}) \text{ exists } : \abs{\x{n}(\x{z}_{j})^{\top}\x{f}(\x{h}_{s_{j}}(\x{z}_{j}))} \geq \frac{\epsilon}{2}
    \end{align*}
    Put $\x{x}_{j} = \x{h}_{s_{j}}(\x{z}_{j})$, then by above $\x{x}_{j} \in \Omega$. Consider
    \begin{align*}
        \norm{\x{x}_{j} - \x{z}_{j}} &= (1 - s_{j})\norm{\x{z}_{j} - \x{x}_{0}}
    \end{align*}
    As $\x{z}_{j} \in \partial \Omega$ and $\x{x}_{0} \in \Omega$ and $\Omega$ is bounded, $\norm{\x{z}_{j} - \x{x}_{0}}$ is bounded for all $j$. So $\norm{\x{x}_{j} - \x{z}_{j}} \to 0$. By the hypothesis of the lemma, $\x{n}(\x{z}_{j})^{\top}\x{f}(\x{x}_{j}) \to 0$, which contradicts $\abs{\x{n}(\x{z}_{j})^{\top}\x{f}(\x{x}_{j})} \geq \epsilon/2$ for all $j$. Hence $\lim_{s \uparrow 1}\sigma(s) = 0$.

    The surface area $\mathcal{A}^{d-1}(\partial\Omega)$ of a bounded convex set is finite, being at most that of any ball containing it. Hence
    \begin{align*}
        \abs{\int_{\Omega_{s}}(\nabla \cdot \x{f})(\x{x})\d\x{x}} \leq s^{d-1}\sigma(s)\,\mathcal{A}^{d-1}(\partial\Omega) \longrightarrow 0 \qquad \text{as } s \uparrow 1
    \end{align*}
    Therefore, it follows that
    \begin{align} \label{eq:limit-of-integral-over-partial-omega-is-0}
        \lim_{s \uparrow 1} \int_{\partial\Omega}\x{n}(\x{z})^{\top}\x{f}(\x{h}_{s}(\x{z}))\d\mathcal{A}^{d-1}(\x{z}) = 0
    \end{align}
    We now show that
    \begin{align} \label{eq:limit-of-integral-over-omega-s-conv-to-omega}
        \lim_{s \uparrow 1} \int_{\Omega_{s}}(\nabla \cdot \x{f})(\x{x})\d\x{x} = \int_{\Omega} (\nabla \cdot \x{f})(\x{x}) \d \x{x}
    \end{align}
    Observe the following
    \begin{enumerate}
        \item We show that $\mathbbm{1}_{\Omega_{s}}(\x{x}) \to \mathbbm{1}_{\Omega}(\x{x})$ for all $\x{x} \in \Omega$. It suffices to show that for any $\x{x} \in \Omega$, $\x{x} \in \Omega_{s}$ eventually. Suppose $\x{x} = \x{x}_{0}$, then $\x{h}_{s}(\x{x}_{0}) = \x{x}_{0}$. Therefore $\x{x}_{0} \in \Omega_{s}$ for every $s$. So for this, we will consider $\x{x} \neq \x{x}_{0}$. Since $\Omega$ is open, there exists a $r > 0$ such that $N_{r}(\x{x}) = \{\x{y} : \norm{\x{y} - \x{x}} < r\} \subset \Omega$. Take any $\frac{\norm{\x{x} - \x{x}_{0}}}{\norm{\x{x} - \x{x}_{0}} + r} < s< 1$, then it is easy to see that
        \begin{align*}
            \x{x} = \x{h}_{s}\left(\x{x}_{0} + \frac{\x{x} - \x{x}_{0}}{s}\right)
        \end{align*}
        It remains to prove that $\x{x}_{0} + \frac{\x{x} - \x{x}_{0}}{s} \in \Omega$. This follows because
        \begin{align*}
            \norm{\x{x}_{0} + \frac{\x{x} - \x{x}_{0}}{s} - \x{x}} &= \norm{\x{x} - \x{x}_{0}}\left(\frac{1}{s} - 1\right) \\
            &< \norm{\x{x} - \x{x}_{0}}\left(\frac{\norm{\x{x} - \x{x}_{0}} + r}{\norm{\x{x} - \x{x}_{0}}} - 1\right) \\
            &= r
        \end{align*}
        So $\x{x}_{0} + \frac{\x{x} - \x{x}_{0}}{s} \in \Omega$ implying that $\x{x} \in \Omega_{s}$. We proved
        \begin{align*}
            \forall s, \: \frac{\norm{\x{x} - \x{x}_{0}}}{\norm{\x{x} - \x{x}_{0}} + r} < s< 1 \text{ we have } \x{x} \in \Omega_{s}
        \end{align*}
        This concludes that $\mathbbm{1}_{\Omega_{s}}(\x{x}) \to \mathbbm{1}_{\Omega}(\x{x})$. Therefore
        \begin{align*}
            \lim_{s \uparrow 1} \left(\nabla \cdot \x{f}\right)(\x{x})\mathbbm{1}_{\Omega_{s}}(\x{x}) &= \left(\nabla \cdot \x{f}\right)(\x{x})\mathbbm{1}_{\Omega}(\x{x})
        \end{align*}
        \item As $\Omega_{s} \subset \Omega$ and $\nabla \cdot \x{f}$ is absolutely integrable on $\Omega$, dominated convergence can be applied resulting in $\int_{\Omega_{s}} \nabla \cdot \x{f} \to \int_{\Omega} \nabla \cdot \x{f}$.
    \end{enumerate}
    Using Equations \ref{eq:divergence-theorem-on-omega-s}, \ref{eq:limit-of-integral-over-partial-omega-is-0}, \ref{eq:limit-of-integral-over-omega-s-conv-to-omega} it follows that
    \begin{align*}
        \int_{\Omega} (\nabla \cdot \x{f})(\x{x}) \d \x{x} = 0
    \end{align*}
\end{proof}

\section{Proof of \autoref{theorem:sm-convex}} \label{appendix:proof-sm-convex}
\begin{proof}
    Fix $\x{\theta}$ and put $f(\x{x}) = \exp(-E_{\x{\theta}}(\x{x}) / 2)$, which is positive and $C^{2}$ on $\Omega$ since $E_{\x{\theta}}$ is $C^{2}$. Observe that
    \begin{align*}
        \exp\left(\frac{1}{2}[E_{\x{\theta}}(\x{x}) - E_{\x{\theta}}(\x{y})]\right) = \frac{f(\x{y})}{f(\x{x})}
    \end{align*}
    So
    \begin{align} \label{eq:psi-numerator}
        \I(\x{x}, \x{\theta}; r) - \C(\x{x}; r) &= \int_{\x{y} \in \Omega} g_{r}(\x{y}, \x{x})\frac{f(\x{y})}{f(\x{x})} \d \x{y} - \int_{\x{y} \in \Omega}g_{r}(\x{y}, \x{x}) \d \x{y} \nonumber \\
        &= \frac{1}{f(\x{x})} \int_{\x{y} \in \Omega} g_{r}(\x{y}, \x{x})(f(\x{y}) - f(\x{x})) \d \x{y}
    \end{align}
    We have to compute the limit of $\Psi$ as $r \to 0$ for a fixed $\x{x}$. Since $\Omega$ is a convex set, the line segment joining $\x{x}, \x{y} \in \Omega$ lies in $\Omega$. So applying Taylor's theorem in integral form, we have
    \begin{align*}
        f(\x{y}) &= f(\x{x}) + \nabla f(\x{x})^{\top}(\x{y} - \x{x}) + \frac{1}{2}(\x{y} - \x{x})^{\top}H_{f}(\x{x})(\x{y} - \x{x}) + \nu(\x{y})
    \end{align*}
    where
    \begin{align*}
        \nu(\x{y}) = \int_{0}^{1}\left[(\x{y} - \x{x})^{\top}\left(H_{f}(\x{x} + t(\x{y} - \x{x})) - H_{f}(\x{x})\right)(\x{y} - \x{x})\right](1 - t) \d t
    \end{align*}
    Substituting the Taylor expansion into \autoref{eq:psi-numerator}, and using $(\x{y} - \x{x})^{\top}H_{f}(\x{x})(\x{y} - \x{x}) = \tr\left(H_{f}(\x{x})(\x{y} - \x{x})(\x{y} - \x{x})^{\top}\right)$, the first two terms produce the first and second moments of the connectivity:
    \begin{align*}
        \I(\x{x}, \x{\theta}; r) - \C(\x{x}; r) &= \frac{1}{f(\x{x})} \int_{\x{y} \in \Omega}g_{r}(\x{y}, \x{x})\left[\nabla f(\x{x})^{\top}(\x{y} - \x{x}) + \frac{1}{2}(\x{y} - \x{x})^{\top}H_{f}(\x{x})(\x{y} - \x{x}) + \nu(\x{y})\right] \d \x{y} \\
        &= \frac{1}{f(\x{x})}\left[\x{m}_{1,r}(\x{x})^{\top}\nabla f(\x{x}) + \frac{1}{2}\tr\left(M_{2,r}(\x{x})H_{f}(\x{x})\right) + \int_{\x{y} \in \Omega} g_{r}(\x{y}, \x{x})\nu(\x{y}) \d \x{y}\right]
    \end{align*}
    Since $\Q(r) = q(r)/4$, this gives
    \begin{align} \label{eq:psi-in-moments}
        \Psi(\x{x}, \x{\theta}; r) &= \frac{4}{f(\x{x})}\Bigg[\frac{1}{q(r)}\x{m}_{1,r}(\x{x})^{\top}\nabla f(\x{x}) + \frac{1}{2}\tr\left(\frac{1}{q(r)}M_{2,r}(\x{x})H_{f}(\x{x})\right) \nonumber \\
        &\quad + \int_{\x{y} \in \Omega} \frac{1}{q(r)}g_{r}(\x{y}, \x{x})\nu(\x{y}) \d \x{y}\Bigg]
    \end{align}
    We will show that the integral term goes to $0$ when $r \to 0$.
    \begin{quote}
        Applying \ref{cond:delta-go-to-0-and-g-0} to the compact set $K = \{\x{x}\}$, there are radii $\delta_{r}(K) > 0$ with $\delta_{r}(K) \to 0$ as $r \to 0$, such that $g_{r}(\x{y}, \x{x}) = 0$ whenever $\norm{\x{y} - \x{x}} > \delta_{r}(K)$. As $\Omega$ is open, there exists $z > 0$ such that $\{\x{y} : \norm{\x{y} - \x{x}} < z\} \subset \Omega$. Define $\beta : (0, z) \to [0, \infty)$ by
        \begin{align*}
            \beta(u) = \sup_{\x{y} : \norm{\x{y} - \x{x}} \leq u}\norm{H_{f}(\x{y}) - H_{f}(\x{x})}
        \end{align*}
        For $u < z$, the closed ball of radius $u$ is a compact subset of $\Omega$ on which $H_{f}$ is continuous, so $\beta(u)$ is finite. Moreover $\beta$ is nondecreasing, since the ball grows with $u$.

        Let $\epsilon > 0$. As $H_{f}$ is continuous at $\x{x}$,
        \begin{align*}
            \exists \: \eta > 0 \text{ such that } \forall \: \x{y} \in \Omega, \norm{\x{y} - \x{x}} < \eta \text{ we have } \norm{H_{f}(\x{y}) - H_{f}(\x{x})} < \epsilon
        \end{align*}
        Fix $s$ with $0 < s < \min(z, \eta)$. Every $\x{y}$ with $\norm{\x{y} - \x{x}} \leq s$ satisfies $\norm{\x{y} - \x{x}} < \eta$, so $\beta(s) \leq \epsilon$. Next, applying \ref{cond:uni-conv-second-moment} to $K = \{\x{x}\}$,
        \begin{align*}
            \exists \: c > 0 \text{ such that } \forall \: 0 < r < c \text{ we have } \norm{\frac{1}{q(r)}M_{2,r}(\x{x}) - D(\x{x})} < 1
        \end{align*}
        Finally, since $\delta_{r}(K) \to 0$
        \begin{align*}
            \exists \: a > 0 \text{ such that } \forall \: 0 < r < a \text{ we have } \delta_{r}(K) < s
        \end{align*}
        From now on, take $0 < r < \min(a, c)$. Then $\delta_{r}(K) < s < z$, so $\beta(\delta_{r}(K))$ is defined, and $\beta(\delta_{r}(K)) \leq \beta(s) \leq \epsilon$ by monotonicity. Let $\x{y}$ satisfy $\norm{\x{y} - \x{x}} \leq \delta_{r}(K)$. Every point $\x{x} + t(\x{y} - \x{x})$ with $t \in [0, 1]$ is within distance $\norm{\x{y} - \x{x}} \leq \delta_{r}(K)$ of $\x{x}$, so the matrix inside $\nu$ has norm at most $\beta(\delta_{r}(K))$, and
        \begin{align*}
            \abs{\nu(\x{y})} &\leq \beta(\delta_{r}(K))\norm{\x{y} - \x{x}}^{2}\int_{0}^{1}(1 - t)\d t \\
            &\leq \frac{1}{2}\epsilon\norm{\x{y} - \x{x}}^{2}
        \end{align*}
        Since $g_{r}(\x{y}, \x{x}) = 0$ whenever $\norm{\x{y} - \x{x}} > \delta_{r}(K)$. Also $g$ is non-negative and $q$ is positive, so
        \begin{align*}
            \abs{\int_{\x{y} \in \Omega}\frac{1}{q(r)}g_{r}(\x{y}, \x{x})\nu(\x{y})\d\x{y}} &\leq \int_{\x{y} : \norm{\x{y} - \x{x}} \leq \delta_{r}(K)}\frac{1}{q(r)}g_{r}(\x{y}, \x{x})\abs{\nu(\x{y})}\d\x{y} \\
            &\leq \frac{1}{2}\epsilon\,\frac{1}{q(r)}\int_{\x{y} \in \Omega}g_{r}(\x{y}, \x{x})\norm{\x{y} - \x{x}}^{2}\d\x{y} \\
            &= \frac{1}{2}\epsilon\,\frac{1}{q(r)}\tr M_{2,r}(\x{x}) \\
            &\leq \frac{1}{2}\epsilon d (1 + \norm{D(\x{x})})
        \end{align*}
        Since $\norm{D(\x{x})}$ does not depend on $\epsilon$ and $\epsilon > 0$ was arbitrary, this shows that
        \begin{align*}
            \lim_{r \to 0}\int_{\x{y} \in \Omega}\frac{1}{q(r)}g_{r}(\x{y}, \x{x})\nu(\x{y})\d\x{y} = 0
        \end{align*}
    \end{quote}
    Applying \ref{cond:uni-conv-second-moment} and \ref{cond:uni-conv-first-moment} on $K = \{\x{x}\}$ once again, we have
    \begin{align} \label{eq:generator-limit}
        \lim_{r \to 0} \Psi(\x{x}, \x{\theta}; r) &= \frac{4}{f(x)}\left[\x{b}(\x{x})^{\top}\nabla f(\x{x}) + \frac{1}{2}\tr\left(D(\x{x})H_{f}(\x{x})\right)\right]
    \end{align}
    From \autoref{lemma:b-is-divergence}, $\x{b}(\x{x}) = \frac{1}{2}(\nabla \cdot D)(\x{x})$. Using $f = \exp(-E_{\x{\theta}}(\x{x}) / 2)$, we have
    \begin{align*}
        \nabla f(\x{x}) = -\frac{1}{2}f(\x{x})\nabla E_{\x{\theta}}(\x{x}) \quad \& \quad H_{f}(\x{x}) = -\frac{1}{2}f(\x{x})H_{E_{\x{\theta}}}(\x{x}) + \frac{1}{4}f(\x{x})\nabla E_{\x{\theta}}(\x{x})\nabla E_{\x{\theta}}(\x{x})^{\top}
    \end{align*}
    Substituting gives
    \begin{align*}
        \lim_{r \to 0} \Psi(\x{x}, \x{\theta}; r) &= \frac{1}{2}\nabla E_{\ttheta}(\xx)^{\top}D(\xx)\nabla E_{\ttheta}(\xx) - (\nabla . D)(\x{x})^{\top}\nabla E_{\x{\theta}}(\x{x}) - \tr\left(D(\x{x})H_{E_{\x{\theta}}}(\x{x})\right) \\
        &= \frac{1}{2}\nabla E_{\ttheta}(\xx)^{\top}D(\xx)\nabla E_{\ttheta}(\xx) - \nabla \cdot \left(D(\xx)\nabla E_{\ttheta}(\xx)\right)
    \end{align*}
    which is exactly \autoref{eq:psi-pointwise-limit}. Finally, under \ref{ass:dominating-function-for-psi} we have $\abs{\Psi(\xx, \ttheta; r)} \leq \mathcal{V}(\xx; \ttheta)$ for all $r \in (0, r_{0}]$ with $\mathcal{V}$ integrable under $p$. The dominated convergence theorem, allows the limit to be taken inside the expectation, and \autoref{eq:main-objective-in-compact-form} follows from \autoref{eq:psi-pointwise-limit}.
\end{proof}

\section{Proof of \autoref{prop:main-objective-in-gsm-form}} \label{appendix:proof-prop-main-objective-in-gsm-form}
\begin{proof}
    Fix $\x{\theta}$. Using $\nabla \log p_{\x{\theta}}(\x{x}) = -\nabla E_{\x{\theta}}(\x{x})$ and $H_{\log p_{\x{\theta}}}(\x{x}) = -H_{E_{\x{\theta}}}(\x{x})$, rewriting \autoref{eq:main-objective-in-compact-form}, we have
    \begin{align*}
        \L(\x{\theta}) &= \underset{\x{x} \sim p}{\E}\left[\frac{1}{2}\nabla \log p_{\x{\theta}}(\x{x})^{\top}D(\x{x})\nabla \log p_{\x{\theta}}(\x{x}) + \nabla \cdot \left(D(\x{x})\nabla \log p_{\x{\theta}}(\x{x})\right)\right]
    \end{align*}
    The idea is to integrate the divergence by parts on the last term. The resulting boundary term involves $p(\x{x})D(\x{x})\nabla \log p_{\x{\theta}}(\x{x})$. So we analyze this first. Accordingly define $\x{f}_{\x{\theta}} : \Omega \to \R^{d}$ as follows
    \begin{align*}
        \x{f}_{\x{\theta}}(\x{x}) = p(\x{x})D(\x{x})\nabla \log p_{\x{\theta}}(\x{x})
    \end{align*}
    We first show that the divergence of $\x{f}_{\x{\theta}}$ is integrable. $\x{f}_{\x{\theta}}$ is $C^{1}$ on $\Omega$, since $p$ is $C^{1}$, $\nabla \log p_{\x{\theta}}$ is $C^{1}$, and $D$ is $C^{1}$. By the product rule, we have
    \begin{align} \label{eq:divergence-of-flux}
        \nabla \cdot \x{f}_{\x{\theta}}(\x{x}) &= \nabla p(\x{x})^{\top}D(\x{x})\nabla \log p_{\x{\theta}}(\x{x}) + p(\x{x})\nabla \cdot\left(D(\x{x})\nabla \log p_{\x{\theta}}(\x{x})\right) \nonumber \\
        &= p(\x{x})\nabla \log p(\x{x})^{\top}D(\x{x})\nabla \log p_{\x{\theta}}(\x{x}) + p(\x{x})\nabla \cdot\left(D(\x{x})\nabla \log p_{\x{\theta}}(\x{x})\right)
    \end{align}
    The second term is integrable because of \ref{ass:dominating-function-for-psi} from \autoref{theorem:sm-convex} and \ref{cond:bounded-quad-form-model}. The first term is integrable because $D$ is positive-definite and it induces inner product of the form $\x{u}^{\top}D\x{v}$. So by Cauchy–Schwarz, we would have
    \begin{align*}
        p(\x{x})\abs{\nabla \log p(\x{x})^{\top}D(\x{x})\nabla \log p_{\x{\theta}}(\x{x})} &\leq \frac{p(\x{x})}{2}\Bigg[\nabla \log p_{\x{\theta}}(\x{x})^{\top}D(\x{x})\nabla \log p_{\x{\theta}}(\x{x}) \\
        &\quad + \nabla \log p(\x{x})^{\top}D(\x{x})\nabla \log p(\x{x})\Bigg]
    \end{align*}
    The terms in right hand side is integrable because of \ref{cond:bounded-quad-form-model} and finiteness assumption on $\E_{\x{x} \sim p}[\nabla \log p(\x{x})^{\top}D(\x{x})\nabla \log p(\x{x})]$ concluding that $\nabla \cdot \x{f}_{\x{\theta}}$ is integrable.

    If $\Omega$ is bounded, \autoref{lemma:divergence-on-convex} can be applied to $\x{f} = \x{f}_{\x{\theta}}$, whose hypothesis are \ref{cond:regularity-conditions-convex-set}, $\x{f}$ to be $C^{1}$ and $\nabla \cdot \x{f}$ be absolutely integrable which we showed that $\x{f}_{\x{\theta}}$ satisfies. Therefore, it follows
    \begin{align*}
        \int_{\Omega} (\nabla \cdot \x{f}_{\x{\theta}})(\x{x}) \d \x{x}= 0
    \end{align*}

    If $\Omega$ is unbounded, the divergence theorem cannot be applied on $\Omega$ directly, so we truncate and apply it on bounded sets. Consider a function $u: \R \to [0, 1]$ defined as
    \begin{align*}
        u(x) = \begin{cases}
            1 & \text{if } x \leq 0 \\
            1 - \frac{\exp\left(-\frac{1}{x}\right)}{\exp\left(-\frac{1}{x}\right) + \exp\left(-\frac{1}{1 - x}\right)} & \text{if } 0 < x < 1 \\
            0 & \text{if } x \geq 1
        \end{cases}
    \end{align*}
    then $u \in C^{\infty}$ with $u = 1$ on $(-\infty, 0]$ and $u = 0$ on $[1, \infty)$. It's also easy to see that all derivatives of $u$ are bounded. Let $L$ be such that $\abs{u'} \leq L$. For $R \geq 1$, put
    \begin{align*}
        \rho_{R}(\x{x}) &= u(\norm{\x{x}} - R)
    \end{align*}
    $\rho_{R}$ is $C^{\infty}$, equals to $1$ when $\norm{\x{x}} \leq R$ and $0$ when $\norm{\x{x}} \geq R+1$. It follows that $\norm{\nabla \rho_{R}} \leq L$. Define
    \begin{align*}
        \x{g}_{R, \x{\theta}}(\x{x}) = \rho_{R}(\x{x})\x{f}_{\x{\theta}}(\x{x})
    \end{align*}
    We denote $B_{R} = \{\x{x} \in \R^{d} : \norm{\x{x}} < R\}$ to be open ball of radius $R$ around $\x{0}$. Consider the bounded open convex set
    \begin{align*}
        \Omega_{R} = \Omega \cap B_{R+2}
    \end{align*}
    which is nonempty for large $R$ (For ex. $R > \inf_{\x{x} \in \Omega}\norm{\x{x}}$). $\x{g}_{R, \x{\theta}}$ is $C^{1}$ on $\Omega_{R}$ because $\rho_{R}, \x{f}_{\x{\theta}}$ are $C^{1}$. By product rule
    \begin{align*}
        (\nabla \cdot \x{g}_{R, \x{\theta}})(\x{x}) &= \rho_{R}(\x{x})(\nabla \cdot \x{f}_{\x{\theta}})(\x{x}) + \x{f}_{\x{\theta}}(\x{x})^{\top}\nabla \rho_{R}(\x{x})
    \end{align*}
    As $\abs{\rho_{R} (\nabla \cdot \x{f}_{\x{\theta}})} \leq \abs{\nabla \cdot \x{f}_{\x{\theta}}}$ and $\nabla \cdot \x{f}_{\x{\theta}}$ is absolutely integrable as shown above, it follows that $\rho_{R}(\nabla \cdot \x{f}_{\x{\theta}})$ is absolutely integrable. Consider the second term
    \begin{align*}
        \int_{\x{x} \in \Omega} \abs{\x{f}_{\x{\theta}}(\x{x})^{\top}\nabla \rho_{R}(\x{x})} \d \x{x} &\leq \int_{\x{x} \in \Omega} \norm{\nabla \rho_{R}(\x{x})}\norm{\x{f}_{\x{\theta}}(\x{x})} \d \x{x} \\
        &\leq L\int_{\x{x} \in \Omega}p(\x{x})\norm{D(\x{x})\nabla \log p_{\x{\theta}}(\x{x})} \d \x{x} \\
        &= L\E_{\x{x} \sim p}[\norm{D(\x{x})\nabla \log p_{\x{\theta}}(\x{x})}] \\
        &< \infty
    \end{align*}
    where the last step is because of \ref{cond:finite-norm-for-unbounded-sets}. This concludes that $\nabla \cdot \x{g}_{R,\x{\theta}}$ is absolutely integrable on $\Omega$ implying it is absolutely integrable on $\Omega_{R}$ as $\Omega_{R} \subset \Omega$.

    \begin{quote}
        \textbf{Claim - 1}  Suppose $\x{z} \in \partial\Omega_{R}$ and $\norm{\x{z}} < R + 2$, then $\x{z} \in \partial\Omega$.
        \begin{proof}
            Since $\Omega_{R} = \Omega \cap B_{R+2}$ we have $\partial\Omega_{R} \subset \partial\Omega \cup \partial B_{R+2}$, where $\partial B_{R+2} = \{\x{x} : \norm{\x{x}} = R + 2\}$. As $\norm{\x{z}} < R + 2$, $\x{z} \notin \partial B_{R+2}$, and hence $\x{z} \in \partial\Omega$.
        \end{proof}
    \end{quote}

    \begin{quote}
        \textbf{Claim - 2} For any bounded sequence $\{\x{z}_{j}\}$ in $\partial\Omega_{R}$ at which $\x{n}(\x{z}_{j})$ exists and $\{\x{x}_{j}\}$ in $\Omega_{R}$ such that $\norm{\x{x}_{j} - \x{z}_{j}} \to 0$, we have $\lim_{j \to \infty} \x{n}_{R}(\x{z}_{j})^{\top}\x{g}_{R, \x{\theta}}(\x{x}_{j}) \to 0$.
        \begin{proof}
             Let $\{\x{z}_{j}\}$ in $\partial\Omega_{R}$ at which $\x{n}_{R}(\x{z}_{j})$, the outward normal of $\Omega_{R}$, exists, and $\{\x{x}_{j}\}$ in $\Omega_{R}$ with $\norm{\x{x}_{j} - \x{z}_{j}} \to 0$. A real sequence converges to $0$ if and only if each of its subsequences has a further subsequence converging to $0$, so fix a subsequence. Since $\partial\Omega_{R}$ is compact, it has a further subsequence, still denoted $\{\x{z}_{j_{k}}\}$, with $\x{z}_{j_{k}} \to \x{z} \in \partial\Omega_{R}$, and then $\x{x}_{j_{k}} \to \x{z}$ as well. Since $\partial\Omega_{R} \subset \partial\Omega \cup \partial B_{R+2}$, there are two cases
             \begin{enumerate}
                 \item $\x{z} \in \partial B_{R+2}$, then $\norm{\x{z}} = R+2$. Then $\norm{\x{x}_{j_{k}}} > R + 1$ eventually, so $\rho_{R}(\x{x}_{j_{k}}) = 0$ and $\x{g}_{R,\x{\theta}}(\x{x}_{j_{k}}) = \x{0}$ eventually. Therefor $\x{n}_{R}(\x{z}_{j_{k}})^{\top}g_{R,\x{\theta}}(\x{x}_{j_{k}}) = 0$.
                 \item $\x{z} \notin \partial B_{R+2}$, then $\x{z} \in \partial \Omega$ and $\norm{\x{z}} < R + 2$. Then $\norm{\x{z}_{j_{k}}} < R + 2$ eventually, so by \textbf{Claim - 1}, $\x{z}_{j_{k}} \in \partial\Omega$ eventually. Since $\Omega_{R}$ and $\Omega$ coincide on a neighbourhood of each such $\x{z}_{j_{k}}$, $\x{n}(\x{z}_{j_{k}})$ exists and equals $\x{n}_{R}(\x{z}_{j_{k}})$. Hence, after discarding finitely many terms, $\{\x{z}_{j_{k}}\}$ is a bounded sequence in $\partial\Omega$ at which $\x{n}$ exists, with $\x{x}_{j_{k}} \in \Omega$ and $\norm{\x{x}_{j_{k}} - \x{z}_{j_{k}}} \to 0$. By \ref{cond:regularity-conditions-convex-set}, $\x{n}(\x{z}_{j_{k}})^{\top}\x{f}_{\x{\theta}}(\x{x}_{j_{k}}) \to 0$, and since $0 \leq \rho_{R} \leq 1$,
                \begin{align*}
                    \abs{\x{n}_{R}(\x{z}_{j_{k}})^{\top}\x{g}_{R,\x{\theta}}(\x{x}_{j_{k}})} = \rho_{R}(\x{x}_{j_{k}})\abs{\x{n}(\x{z}_{j_{k}})^{\top}\x{f}_{\x{\theta}}(\x{x}_{j_{k}})} \leq \abs{\x{n}(\x{z}_{j_{k}})^{\top}\x{f}_{\x{\theta}}(\x{x}_{j_{k}})} \longrightarrow 0.
                \end{align*}
            \end{enumerate}
            In both cases $\x{n}_{R}(\x{z}_{j_{k}})^{\top}\x{g}_{R,\x{\theta}}(\x{x}_{j_{k}}) \to 0$ as $k \to \infty$. Since the initial subsequence was arbitrary, $\x{n}_{R}(\x{z}_{j})^{\top}\x{g}_{R,\x{\theta}}(\x{x}_{j}) \to 0$.
        \end{proof}
    \end{quote}
    From \autoref{lemma:divergence-on-convex}, we have
    \begin{align*}
        &\int_{\Omega_{R}}(\nabla \cdot \x{g}_{R, \x{\theta}})(\x{x}) \d \x{x} = 0 \\
        \implies & \int_{\Omega_{R}}\rho_{R}(\x{x})(\nabla \cdot \x{f}_{\x{\theta}})(\x{x}) = -\int_{\Omega_{R}}\x{f}_{\x{\theta}}(\x{x})^{\top}\nabla \rho_{R}(\x{x}) \d \x{x}
    \end{align*}
    Since both integrands vanish on $\Omega \setminus B_{R+1}$, and $\Omega \cap B_{R+1} \subset \Omega_{R}$, so the integrals over $\Omega_{R}$ equal those over $\Omega$ that is
    \begin{align*}
        \int_{\Omega}\rho_{R}(\x{x})(\nabla \cdot \x{f}_{\x{\theta}})(\x{x}) \d \x{x} = -\int_{\Omega}\x{f}_{\x{\theta}}(\x{x})^{\top}\nabla \rho_{R}(\x{x}) \d \x{x}
    \end{align*}
    We now let $R \to \infty$. Let $\{R_{n}\}$ be any sequence with $R_{n} \to \infty$. We treat each side using the dominated convergence theorem

    For right hand side, since $\norm{\nabla\rho_{R_{n}}} \leq L$, we have for every $n$
    \begin{align*}
        \abs{\x{f}_{\x{\theta}}(\x{x})^{\top}\nabla\rho_{R_{n}}(\x{x})} \leq L\norm{\x{f}_{\x{\theta}}(\x{x})},
    \end{align*}
    and $L\norm{\x{f}_{\x{\theta}}}$ is absolutely integrable on $\Omega$ by \ref{cond:finite-norm-for-unbounded-sets}. For fixed $\x{x} \in \Omega$, once $R_{n} > \norm{\x{x}}$, the point $\x{x}$ lies in the open ball $B_{R_{n}}$, on which $\rho_{R_{n}} = 1$; hence $\nabla\rho_{R_{n}}(\x{x}) = \x{0}$ and
    \begin{align*}
        \lim_{n \to \infty}\x{f}_{\x{\theta}}(\x{x})^{\top}\nabla\rho_{R_{n}}(\x{x}) = 0.
    \end{align*}
    By dominated convergence,
    \begin{align*}
        \lim_{n \to \infty}\int_{\Omega}\x{f}_{\x{\theta}}(\x{x})^{\top}\nabla\rho_{R_{n}}(\x{x})\d\x{x} = 0.
    \end{align*}

    For left hand side, since $0 \leq \rho_{R_{n}} \leq 1$, we have for every $n$
    \begin{align*}
        \abs{\rho_{R_{n}}(\x{x})(\nabla \cdot \x{f}_{\x{\theta}})(\x{x})} \leq \abs{(\nabla \cdot \x{f}_{\x{\theta}})(\x{x})},
    \end{align*}
    and $\nabla \cdot \x{f}_{\x{\theta}}$ is absolutely integrable on $\Omega$. For fixed $\x{x} \in \Omega$, once $R_{n} \geq \norm{\x{x}}$ we have $\rho_{R_{n}}(\x{x}) = 1$, so
    \begin{align*}
        \lim_{n \to \infty}\rho_{R_{n}}(\x{x})(\nabla \cdot \x{f}_{\x{\theta}})(\x{x}) = (\nabla \cdot \x{f}_{\x{\theta}})(\x{x}).
    \end{align*}
    By dominated convergence,
    \begin{align*}
        \lim_{n \to \infty}\int_{\Omega}\rho_{R_{n}}(\x{x})(\nabla \cdot \x{f}_{\x{\theta}})(\x{x})\d\x{x} = \int_{\Omega}(\nabla \cdot \x{f}_{\x{\theta}})(\x{x})\d\x{x}.
    \end{align*}

    So in the case of unbounded sets as well, we have
    \begin{align*}
        \int_{\Omega} (\nabla \cdot \x{f}_{\x{\theta}})(\x{x}) \d \x{x} = 0
    \end{align*}

    Integrating \autoref{eq:divergence-of-flux} over $\Omega$, we have therefore shown
    \begin{align*}
        \underset{\x{x} \sim p}{\E}\left[\nabla \log p(\x{x})^{\top}D(\x{x})\nabla \log p_{\x{\theta}}(\x{x})\right] = -\underset{\x{x} \sim p}{\E}\left[\nabla \cdot\left(D(\x{x})\nabla \log p_{\x{\theta}}(\x{x})\right)\right]
    \end{align*}
    Substituting into $\L(\x{\theta})$ and completing the square, using that $D$ is symmetric,
    \begin{align*}
        \L(\x{\theta}) = \L_{\text{GSM}}(\x{\theta}) - \frac{1}{2}\underset{\x{x} \sim p}{\E}\left[\nabla \log p(\x{x})^{\top}D(\x{x})\nabla \log p(\x{x})\right]
    \end{align*}
    The last term is finite by assumption and does not depend on $\x{\theta}$, so minimising $\L(\x{\theta})$ is equivalent to minimising $\L_{\text{GSM}}(\x{\theta})$.
\end{proof}

\section{Proof of \autoref{prop:linear-operator-is-complete}} \label{appendix:proof-prop-linear-operator-is-complete}
\begin{proof}
    This follows as $D(\x{x})^{\frac{1}{2}}\nabla \log p_{1}(\x{x}) = D(\x{x})^{\frac{1}{2}}\nabla \log p_{2}(\x{x})$ implies $\nabla \log p_{1}(\x{x}) = \nabla \log p_{2}(\x{x})$ almost everywhere. This is the same as $\nabla \log \frac{p_{1}(\x{x})}{p_{2}(\x{x})} = 0$. As $\Omega$ is convex and convex subsets of $\R^{d}$ are connected, we conclude that $\frac{p_{1}(\x{x})}{p_{2}(\x{x})} = c$ almost everywhere for some constant $c$. As both $p_{1}, p_{2}$ are densities, we conclude that $p_{1}(\x{x}) = p_{2}(\x{x})$ almost everywhere.
\end{proof}

\section{Proof of \autoref{prop:midpoint-connectivity-checks-main-theorem}} \label{appendix:proof-midpoint-connectivity-checks-main-theorem}
\begin{proof}
    We begin with \ref{cond:symmetric-g-and-c-bounded}. Interchanging $\x{x}$ and $\x{y}$ doesn't change $\bar{\x{x}}, \norm{\x{y} - \x{x}}_{A(\bar{\x{x}})}, \norm{\x{y} - \x{x}}$ which implies that $g_{r}$ is symmetric. Also, $g_{r}$ is measurable because it is the product of continuous and indicator functions.

    Fix a compact subset $K \subset \Omega$, then by \autoref{lemma:compact-cover-of-S}, there exists $\eta > 0$ such that
    \begin{align*}
        K' = \{\x{z} \in \R^{d} : \inf_{\x{x} \in K} \norm{\x{z} - \x{x}} \leq \eta\}
    \end{align*}
    is compact and $K \subset K' \subset \Omega$. Since $c$ is a continuous function on $K'$, implies $c_{K'} = \sup_{\x{z} \in K'} c(\x{z})$ is finite. As $\varepsilon(r) \to 0$, we have
    \begin{align*}
        \exists \: r_{K} > 0 \text{ such that } \forall \: 0 <r < r_{K} \text{ we have } \varepsilon(r) < \eta
    \end{align*}
    Take $0 < r < r_{K}$ and $\x{x} \in K$ and $\x{y} \in \Omega$ such that $g_{r}(\x{y}, \x{x}) > 0$. From \autoref{eq:midpoint-symmetric-connectivity}, we have $\norm{\x{y} - \x{x}} \leq \varepsilon(r)$ which implies $\norm{\x{y} - \x{x}} < \eta$ so $\x{y} \in K' \subset \Omega$. So the ball $\{\x{y} \in \R^{d} : \norm{\x{y} - \x{x}} \leq \varepsilon(r)\}$ lies inside $\Omega$. Further,
    \begin{align*}
        \norm{\bar{\x{x}} - \x{x}} \leq \frac{1}{2}\varepsilon(r) < \eta
    \end{align*}
    So $\bar{\x{x}} \in K'$ meaning $c(\bar{\x{x}}) \leq c_{K'}$. Consider
    \begin{align*}
        \C(\x{x}; r) &= \int_{\x{y} \in \Omega}g_{r}(\x{y}, \x{x}) \d \x{y} \\
        &= \underset{\substack{\x{y}: \norm{\x{y} - \x{x}} \leq \varepsilon(r)}}{\int}c(\bar{\x{x}})\exp\left(-\frac{\norm{\x{y} - \x{x}}^{2}_{A(\bar{\x{x}})}}{2r^{2}}\right) \d \x{y} \\
        &\leq \underset{\substack{\x{y}: \norm{\x{y} - \x{x}} \leq \varepsilon(r)}}{\int}c(\bar{\x{x}}) \d \x{y}
    \end{align*}
    where the first inequality follows because exponent in non-positive. From above, if $\x{x} \in K$ and $0 < r < r_{K}$, then
    \begin{align*}
        \C(\x{x}; r) &\leq c_{K'}\text{Vol}\left(\varepsilon(r)\right) \\
        &< c_{K'}\text{Vol}\left(2\eta\right)
    \end{align*}
    where $\text{Vol}(s)$ denotes the volume of a ball in $\R^{d}$ with radius $s$. This concludes \ref{cond:symmetric-g-and-c-bounded}.

    From \autoref{eq:midpoint-symmetric-connectivity}, it's easy to see that $g_{r}(\x{y}, \x{x}) = 0$ if $\norm{\x{y} - \x{x}} > \varepsilon(r)$. Therefore, \ref{cond:delta-go-to-0-and-g-0} holds for every compact $K$ with $\delta_{r}(K) = \varepsilon(r)$, which is positive and tends to $0$.

    We now turn to \ref{cond:uni-conv-second-moment} and \ref{cond:uni-conv-first-moment}. Since $A$ is positive-definite and continuous, the smallest eigenvalue is a continuous function of $A$. Therefore
    \begin{align*}
        \lambda = \inf_{\x{z} \in K'} \lambda_{\min}(A(\x{z})) > 0
    \end{align*}
    Observe that $\lambda$ is positive, because above infimum will be attained as $K'$ is compact. As $\varepsilon(r) \to 0$,
    \begin{align*}
        \exists \: r_{*} > 0 \text{ such that } \forall \: 0 < r < r_{*} \text{ we have } \varepsilon(r) < \eta
    \end{align*}
    From now on, we take $0 < r < r_{*}$. Following the same argument as we did for \ref{cond:symmetric-g-and-c-bounded}, the ball $\{\x{y} \in R^{d} : \norm{\x{y} - \x{x}} \leq \varepsilon(r)\}$ lies inside $\Omega$. As $g_{r}$ depends on $\x{y} - \x{x}$ and $\bar{\x{x}}$, we define
    \begin{align*}
        \gamma_{r}(\x{z}, \x{h}) = c(\x{z})\exp\left(-\frac{\norm{\x{h}}^{2}_{A(\x{z})}}{2r^{2}}\right)
    \end{align*}
    then $g_{r}(\x{x} + \x{h}, \x{x}) = \gamma_{r}\left(\x{x} + \frac{1}{2}\x{h}, \x{h}\right)\mathbbm{1}_{\norm{\x{h}} \leq \varepsilon(r)}(\x{h})$ where $\mathbbm{1}$ is indicator function. We observe the following
    \begin{quote}
        \textbf{Claim - 1} : If $\x{x} \in K, \norm{\x{h}} \leq \varepsilon(r)$, then $\x{x} + t\x{h} \in K'$ for all $t \in [0, 1]$. Further,
        \begin{align*}
            M_{2,r}(\x{x}) = \int_{\x{h} : \norm{\x{h}} \leq \varepsilon(r)} \gamma_{r}\left(\x{x} + \frac{1}{2}\x{h}, \x{h}\right)\x{h}\x{h}^{\top} d\x{h} \\ \x{m}_{1,r}(\x{x}) = \int_{\x{h} : \norm{\x{h}} \leq \varepsilon(r)}\gamma_{r}\left(\x{x} + \frac{1}{2}\x{h}, \x{h}\right) \x{h} \d \x{h}
        \end{align*}
        \begin{proof}
            Observe that
            \begin{align*}
                \norm{\x{x} + t\x{h} - \x{x}} = t\norm{\x{h}} \leq \varepsilon(r) < \eta
            \end{align*}
            which shows that $\x{x} + t\x{h} \in K'$ (where we used the fact that $\x{x} \in K$). Putting $\x{y} = \x{x} + \x{h}$, then the region on which $g_{r}$ is non zero is $\norm{\x{h}} \leq \varepsilon(r)$. So the expression of $M_{2,r}$ and $\x{m}_{1,r}$ trivially follows.
        \end{proof}
        \textbf{Claim - 2} : For any $\x{z} \in K', \x{h} \in \R^{d}$, we have
        \begin{align*}
            \abs{\gamma_{r}(\x{z}, \x{h})} &\leq L\exp\left(-\frac{\lambda}{2r^{2}}\norm{\x{h}}^{2}\right) \\
            \norm{\nabla_{\x{z}} \gamma_{r}(\x{z}, \x{h})} &\leq L_{1}\left(1 + \frac{\norm{\x{h}}^{2}}{r^{2}}\right)\exp\left(-\frac{\lambda}{2r^{2}}\norm{\x{h}}^{2}\right) \\
            \norm{H_{\gamma_{r}}(\x{z}, \x{h})} &\leq L_{2}\left(1 + \frac{\norm{\x{h}}^{2}}{r^{2}} + \frac{\norm{\x{h}}^{4}}{r^{4}}\right)\exp\left(-\frac{\lambda}{2r^{2}}\norm{\x{h}}^{2}\right)
        \end{align*}
        where $H_{\gamma_{r}}$ is Hessian of $\gamma_{r}$ with respect to $\x{z}$ for some real constants $L, L_{1}, L_{2}$
        \begin{proof}
            The bound follows because $A$ and $c$ are $C^{2}$ functions and $K'$ is compact. Observe that $\norm{\x{h}}_{A(\x{z})}^{2} \geq \lambda_{min}(A(\x{z}))\norm{\x{h}}^{2} \geq \lambda\norm{\x{h}}^{2}$ as $\x{z} \in K'$. Therefore,
            \begin{align*}
                \exp\left(-\frac{\norm{\x{h}}^{2}_{A(\x{z})}}{2r^{2}}\right) \leq \exp\left(-\frac{\lambda}{2r^{2}}\norm{\x{h}}^{2}\right)
            \end{align*}
            Let $L = \sup_{\x{z} \in K'}c(\x{z})$, then first bound follows. The gradient of $\gamma_{r}$ would be
            \begin{align*}
                \frac{\partial \gamma_{r}}{\partial z_{k}}(\x{z}) &= \left(\frac{\partial c}{\partial z_{k}}(\x{z}) - \frac{c(\x{z})}{2r^{2}}\x{h}^{\top}\frac{\partial A}{\partial z_{k}}(\x{z})\x{h}\right)\exp\left(-\frac{\norm{\x{h}}^{2}_{A(\x{z})}}{2r^{2}}\right)
            \end{align*}
            $\frac{\partial c}{\partial z_{k}}(\x{z})$, $\frac{\partial A}{\partial z_{k}}(\x{z})$ are continuous, so they are bounded. Put $L_{k} = \sup_{\x{z} \in K'}\abs{\frac{\partial c}{\partial z_{k}}(\x{z})}, L'_{k} = \frac{L}{2}\sup_{\x{z} \in K'}\norm{\frac{\partial A}{\partial z_{k}}(\x{z})}$, then we have
            \begin{align*}
                \abs{\frac{\partial \gamma_{r}}{\partial z_{k}}(\x{z})} &\leq \left(L_{k} + L'_{k}\frac{\norm{\x{h}}^{2}}{r^{2}}\right)\exp\left(-\frac{\lambda}{2r^{2}}\norm{\x{h}}^{2}\right) \\
                &\leq \max(L_{k},L'_{k})\left(1 + \frac{\norm{\x{h}}^{2}}{r^{2}}\right)\exp\left(-\frac{\lambda}{2r^{2}}\norm{\x{h}}^{2}\right)
            \end{align*}
            Cumulating over all $1 \leq k \leq d$, we will have a second bound with a constant $L_{1}$. Differentiating once again produces similar terms and following the same argument as $\frac{\partial \gamma}{\partial z_{k}}$, one can arrive at the final bound as well.

            An important point to note is that constants $L, L_{1}, L_{2}$ doesn't depend on $r$.
        \end{proof}
    \end{quote}
    Take any $\x{x} \in K$, from \textbf{Claim - 1}
    \begin{align*}
        M_{2, r}(\x{x}) &= \int_{\x{h} : \norm{\x{h}} \leq \varepsilon(r)}\gamma_{r}\left(\x{x} + \frac{1}{2}\x
        {h}, \x{h}\right) \x{h}\x{h}^{\top} \d \x{h}
    \end{align*}
    We compare this with the second moment of $\gamma_{r}(\x{x}, \x{h})$, by writing
    \begin{align} \label{eq:m2-split}
        M_{2,r}(\x{x}) = \int_{\x{h} : \norm{\x{h}} \leq \varepsilon(r)}\gamma_{r}(\x{x}, \x{h})\x{h}\x{h}^{\top}\d\x{h} + \int_{\x{h} : \norm{\x{h}} \leq \varepsilon(r)}\left[\gamma_{r}\left(\x{x} + \frac{1}{2}\x{h}, \x{h}\right) - \gamma_{r}(\x{x}, \x{h})\right]\x{h}\x{h}^{\top}\d\x{h}
    \end{align}
    Consider the first integral in the above equality
    \begin{align*}
        \int_{\x{h} : \norm{\x{h}} \leq \varepsilon(r)}\gamma_{r}(\x{x}, \x{h})\x{h}\x{h}^{\top} \d \x{h} &= \int_{\R^{d}} \gamma_{r}(\x{x}, \x{h})\x{h}\x{h}^{\top} \d \x{h} - \int_{\x{h} : \norm{\x{h}} > \varepsilon(r)}\gamma_{r}(\x{x}, \x{h})\x{h}\x{h}^{\top} \d \x{h} \\
        &= \int_{\R^{d}}c(\x{x})\exp\left(-\frac{\norm{\x{h}}_{A(\x{x})}^{2}}{2r^{2}}\right)\x{h} \x{h}^{\top} \d \x{h} - \int_{\x{h} : \norm{\x{h}} > \varepsilon(r)}\gamma_{r}(\x{x}, \x{h})\x{h}\x{h}^{\top} \d \x{h} \\
        &= c(\x{x})\left(2\pi\right)^{\frac{d}{2}}r^{d+2}\det A(\x{x})^{-\frac{1}{2}}A(\x{x})^{-1} - \int_{\x{h} : \norm{\x{h}} > \varepsilon(r)}\gamma_{r}(\x{x}, \x{h})\x{h}\x{h}^{\top} \d \x{h}
    \end{align*}
    From the proposition statement, we have $q(r) = (2\pi)^{\frac{d}{2}}r^{d+2}$ and $D(\x{x}) = c(\x{x})\det A(\x{x})^{-\frac{1}{2}}A(\x{x})^{-1}$. So by \autoref{eq:m2-split},
    \begin{align*}
        \frac{1}{q(r)}M_{2,r}(\x{x}) - D(\x{x}) &= \frac{1}{q(r)}\int_{\x{h} : \norm{\x{h}} \leq \varepsilon(r)}\left[\gamma_{r}\left(\x{x} + \frac{1}{2}\x{h}, \x{h}\right) - \gamma_{r}(\x{x}, \x{h})\right]\x{h}\x{h}^{\top} \d \x{h}  \\
        &\quad - \frac{1}{q(r)}\int_{\x{h} : \norm{\x{h}} > \varepsilon(r)}\gamma_{r}(\x{x}, \x{h})\x{h}\x{h}^{\top} \d \x{h}
    \end{align*}
    The idea now will be to bound $\sup_{\x{x} \in K} \norm{\frac{1}{q(r)}M_{2,r}(\x{x}) - D(\x{x})}$. From above equation, we have
    \begin{align*}
        \norm{\frac{1}{q(r)}M_{2,r}(\x{x}) - D(\x{x})} &\leq \frac{1}{q(r)}\underbrace{\int_{\x{h} : \norm{\x{h}} \leq \varepsilon(r)}\abs{\gamma_{r}\left(\x{x} + \frac{1}{2}\x{h}, \x{h}\right) - \gamma_{r}(\x{x}, \x{h})}\norm{\x{h}}^{2} \d \x{h}}_{\tau_{1}(\x{x})} \\
        &\quad + \frac{1}{q(r)}\underbrace{\int_{\x{h} : \norm{\x{h}} > \varepsilon(r)}\abs{\gamma_{r}(\x{x}, \x{h})}\norm{\x{h}}^{2} \d \x{h}}_{\tau_{2}(\x{x})}
    \end{align*}
    Before going ahead, we will define an integral that we will use soon. Let $I_{k} = \int_{\x{u} \in \R^{d}} \norm{\x{u}}^{k}\exp\left(-\frac{\lambda}{2}\norm{\x{u}}^{2}\right) \d \x{u}$, then applying the transformation $\x{h} = r\x{u}$, we have
    \begin{align*}
        I_{k} = \frac{1}{r^{d + k}}\int_{\x{h} \in \R^{d}}\norm{\x{h}}^{k}\exp\left(-\frac{\lambda}{2r^{2}}\norm{\x{h}}^{2}\right) \d \x{h}
    \end{align*}
    We now move ahead to bound $\tau_{1}, \tau_{2}$.
    \begin{quote}
        For each $\x{h}$, the mean value theorem applied to $t \mapsto \gamma_{r}(\x{x} + t\x{h}, \x{h})$ on $[0, \frac{1}{2}]$ gives a point $\x{z}' = \x{x} + t'\x{h}$, with $t' \in (0, \frac{1}{2})$ depending on $\x{h}$, such that
        \begin{align*}
            \gamma_{r}\left(\x{x} + \frac{1}{2}\x{h}, \x{h}\right) - \gamma_{r}(\x{x}, \x{h}) = \frac{1}{2}\nabla_{\x{z}}\gamma_{r}(\x{z}', \x{h})^{\top}\x{h}
        \end{align*}
        Since $\x{z}' \in K'$ by \textbf{Claim - 1}, \textbf{Claim - 2} bounds the right-hand side, and hence
        \begin{align*}
            \abs{\gamma_{r}\left(\x{x} + \frac{1}{2}\x{h}, \x{h}\right) - \gamma_{r}(\x{x}, \x{h})} \leq \frac{L_{1}}{2}\norm{\x{h}}\left(1 + \frac{\norm{\x{h}}^{2}}{r^{2}}\right)\exp\left(-\frac{\lambda}{2r^{2}}\norm{\x{h}}^{2}\right)
        \end{align*}
        Putting this bound in $\tau_{1}$, we have
        \begin{align*}
            \frac{1}{q(r)}\tau_{1}(\x{x}) &\leq \frac{L_{1}}{2q(r)}\int_{\x{h} \in \R^{d}}\left(\norm{\x{h}}^{3} + \frac{\norm{\x{h}}^{5}}{r^{2}}\right)\exp\left(-\frac{\lambda}{2r^{2}}\norm{\x{h}}^{2}\right) \d \x{h} \\
            &= \frac{L_{1}}{2(2\pi)^{\frac{d}{2}}r^{d+2}}\left(r^{d+3}I_{3} + r^{d+3}I_{5}\right) \\
            &= \frac{L_{1}r}{2(2\pi)^{\frac{d}{2}}}(I_{3} + I_{5})
        \end{align*}
    \end{quote}
    Similarly for $\tau_{2}$
    \begin{quote}
        Using the bound from \textbf{Claim - 2}, we have
        \begin{align*}
            \frac{1}{q(r)}\tau_{2}(\x{x}) \leq \frac{L}{q(r)}\int_{\x{h} : \norm{\x{h}} > \varepsilon(r)}\norm{\x{h}}^{2}\exp\left(-\frac{\lambda}{2r^{2}}\norm{\x{h}}^{2}\right) \d \x{h}
        \end{align*}
        Applying the transformation $\x{h} = r\x{u}$, we have
        \begin{align*}
            \frac{1}{q(r)}\tau_{2}(\x{x}) \leq \frac{L}{(2\pi)^{\frac{d}{2}}}\int_{\x{u}: \norm{\x{u}} > \frac{\varepsilon(r)}{r}}\norm{\x{u}}^{2}\exp\left(-\frac{\lambda}{2}\norm{\x{u}}^{2}\right) \d \x{u}
        \end{align*}
    \end{quote}
    Putting it all together,
    \begin{align*}
        \sup_{\x{x} \in K} \norm{\frac{1}{q(r)}M_{2,r}(\x{x}) - D(\x{x})} &\leq \frac{L_{1}r}{2(2\pi)^{\frac{d}{2}}}(I_{3} + I_{5}) + \frac{L}{(2\pi)^{\frac{d}{2}}}\int_{\x{u}: \norm{\x{u}} > \frac{\varepsilon(r)}{r}}\norm{\x{u}}^{2}\exp\left(-\frac{\lambda}{2}\norm{\x{u}}^{2}\right) \d \x{u}
    \end{align*}
    As $r \to 0$ we have $\frac{\varepsilon(r)}{r} \to \infty$, so the integral is tail of Gaussian, implying it will go to $0$. First term has a factor of $r$, therefore it will also go to zero. Further, the limit $D$ is $C^{1}$, indeed $C^{2}$, since $c$ and $A$ are $C^{2}$, and $D(\x{x})$ is symmetric positive-definite since $c(\x{x}) > 0$ and $A(\x{x})$ is symmetric positive-definite. This concludes \ref{cond:uni-conv-second-moment}.

    We now turn to \ref{cond:uni-conv-first-moment}. Take any $\x{x} \in K$. From \textbf{Claim - 1},
    \begin{align*}
        \x{m}_{1,r}(\x{x}) = \int_{\x{h} : \norm{\x{h}} \leq \varepsilon(r)}\gamma_{r}\left(\x{x} + \frac{1}{2}\x{h}, \x{h}\right)\x{h}\d\x{h}
    \end{align*}
    For \ref{cond:uni-conv-second-moment} a first-order expansion of $\gamma_{r}$ in its first argument was enough. Here we need one more order. Fix $\x{h}$ with $\norm{\x{h}} \leq \varepsilon(r)$ and put $f(t) = \gamma_{r}(\x{x} + t\x{h}, \x{h})$ for $t \in [0, \frac{1}{2}]$. By \textbf{Claim - 1}, $\x{x} + t\x{h} \in K'$ for these $t$, and
    \begin{align*}
        f'(t) = \nabla_{\x{z}}\gamma_{r}(\x{x} + t\x{h}, \x{h})^{\top}\x{h}, \qquad f''(t) = \x{h}^{\top}H_{\gamma_{r}}(\x{x} + t\x{h}, \x{h})\x{h}
    \end{align*}
    Taylor's theorem in integral form on $[0, \frac{1}{2}]$ gives
    \begin{align*}
        \gamma_{r}\left(\x{x} + \frac{1}{2}\x{h}, \x{h}\right) = \gamma_{r}(\x{x}, \x{h}) + \frac{1}{2}\nabla_{\x{z}}\gamma_{r}(\x{x}, \x{h})^{\top}\x{h} + \rho(\x{x}, \x{h})
    \end{align*}
    where
    \begin{align*}
        \rho(\x{x}, \x{h}) = \int_{0}^{\frac{1}{2}}\left(\frac{1}{2} - t\right)f''(t)\d t
    \end{align*}
    Since $\int_{0}^{1/2}(\frac{1}{2} - t)\d t = \frac{1}{8}$, \textbf{Claim - 2} gives
    \begin{align*}
        \abs{\rho(\x{x}, \x{h})} \leq \frac{L_{2}}{8}\left(\norm{\x{h}}^{2} + \frac{\norm{\x{h}}^{4}}{r^{2}} + \frac{\norm{\x{h}}^{6}}{r^{4}}\right)\exp\left(-\frac{\lambda}{2r^{2}}\norm{\x{h}}^{2}\right)
    \end{align*}
    Putting $\gamma_{r}\left(\x{x} + \frac{1}{2}\x{h}, \x{h}\right)$ in $\x{m}_{1,r}$, we have
    \begin{align*}
        \x{m}_{1,r}(\x{x}) &= \underbrace{\int_{\x{h} : \norm{\x{h}} \leq \varepsilon(r)}\gamma_{r}(\x{x}, \x{h})\x{h}\d\x{h}}_{\x{a}_{1}} + \underbrace{\frac{1}{2}\int_{\x{h} : \norm{\x{h}} \leq \varepsilon(r)}\left(\nabla_{\x{z}}\gamma_{r}(\x{x}, \x{h})^{\top}\x{h}\right)\x{h}\d\x{h}}_{\x{a}_{2}} \\
        &\quad  + \underbrace{\int_{\x{h} : \norm{\x{h}} \leq \varepsilon(r)}\rho(\x{x}, \x{h})\x{h}\d\x{h}}_{\x{a}_{3}}
    \end{align*}
    We treat the three terms separately.
    \begin{quote}
        Since $\norm{-\x{h}}_{A(\x{x})} = \norm{\x{h}}_{A(\x{x})}$, the function $\gamma_{r}(\x{x}, \cdot)$ is even, so the integrand of the first term is odd in $\x{h}$. The ball $\norm{\x{h}} \leq \varepsilon(r)$ is symmetric about $\x{0}$, hence $\x{a}_{1} = \x{0}$.
    \end{quote}
    \begin{quote}
        Using the bound on $\rho$ and the definition of $I_{k}$,
        \begin{align*}
            \frac{1}{q(r)}\norm{\x{a}_{3}} &\leq \frac{L_{2}}{8q(r)}\int_{\x{h} \in \R^{d}}\left(\norm{\x{h}}^{3} + \frac{\norm{\x{h}}^{5}}{r^{2}} + \frac{\norm{\x{h}}^{7}}{r^{4}}\right)\exp\left(-\frac{\lambda}{2r^{2}}\norm{\x{h}}^{2}\right)\d\x{h} \\
            &= \frac{L_{2}r}{8(2\pi)^{\frac{d}{2}}}(I_{3} + I_{5} + I_{7})
        \end{align*}
    \end{quote}
    \begin{quote}
        The term $\x{a}_{2}$ produces the divergence. It can be written as
        \begin{align*}
            \x{a}_{2} = \frac{1}{2}\int_{\x{h} : \norm{\x{h}} \leq \varepsilon(r)}(\x{h}\x{h}^{\top})\nabla_{\x{z}} \gamma_{r}(\x{x}, \x{h}) \d \x{h}
        \end{align*}
        Likewise in \ref{cond:uni-conv-second-moment}, we will split this in $\R^{d}$ and $\norm{\x{h}} > \varepsilon(r)$. Accordingly, for $\x{z} \in K'$ define,
        \begin{align*}
            \x{\mu}(\x{z}) = \int_{\norm{\x{h}} \leq \varepsilon(r)}\gamma_{r}(\x{z}, \x{h})\x{h}\x{h}^{\top}\d\x{h}, \qquad \zeta(\x{z}) = \int_{\norm{\x{h}} > \varepsilon(r)}\gamma_{r}(\x{z}, \x{h})\x{h}\x{h}^{\top}\d\x{h}
        \end{align*}
        Exactly as in \ref{cond:uni-conv-second-moment}, we have $\x{\mu}(\x{z}) + \zeta(\x{z}) = q(r)D(\x{z})$. The $i$-th component of $\x{a}_{2}$ is
        \begin{align*}
            \frac{1}{2}\sum_{k=1}^{d}\int_{\norm{\x{h}} \leq \varepsilon(r)}\frac{\partial\gamma_{r}}{\partial z_{k}}(\x{x}, \x{h})\,h_{k}h_{i}\d\x{h} = \frac{1}{2}\sum_{k=1}^{d}\frac{\partial \x{\mu}_{ki}}{\partial z_{k}}(\x{x}) = \frac{1}{2}(\nabla \cdot \x{\mu})_{i}(\x{x})
        \end{align*}
        where differentiating under the integral sign is allowed because the domain is a fixed compact ball and $\partial\gamma_{r}/\partial z_{k}$ is continuous. Hence
        \begin{align*}
            \x{a}_{2} = \frac{1}{2}q(r)(\nabla \cdot D)(\x{x}) - \frac{1}{2}(\nabla \cdot \zeta)(\x{x})
        \end{align*}
        It remains to show that $\nabla \cdot \zeta$ is small. The closed ball of radius $\eta$ around $\x{x}$ lies in $K'$, and for $\x{z}$ in it, \textbf{Claim - 2} gives
        \begin{align*}
            \abs{\frac{\partial\gamma_{r}}{\partial z_{k}}(\x{z}, \x{h})}\norm{\x{h}}^{2} \leq L_{1}\left(\norm{\x{h}}^{2} + \frac{\norm{\x{h}}^{4}}{r^{2}}\right)\exp\left(-\frac{\lambda}{2r^{2}}\norm{\x{h}}^{2}\right)
        \end{align*}
        which is integrable in $\x{h}$ and does not depend on $\x{z}$. So we may differentiate $\zeta$ under the integral sign at $\x{x}$, and applying $\x{h} = r\x{u}$,
        \begin{align*}
            \norm{\frac{\partial\zeta}{\partial z_{k}}(\x{x})} \leq \int_{\norm{\x{h}} > \varepsilon(r)}\abs{\frac{\partial\gamma_{r}}{\partial z_{k}}(\x{x}, \x{h})}\norm{\x{h}}^{2}\d\x{h} \leq L_{1}r^{d+2}J(r)
        \end{align*}
        where
        \begin{align*}
            J(r) := \int_{\x{u} : \norm{\x{u}} > \frac{\varepsilon(r)}{r}}\left(\norm{\x{u}}^{2} + \norm{\x{u}}^{4}\right)\exp\left(-\frac{\lambda}{2}\norm{\x{u}}^{2}\right)\d\x{u}
        \end{align*}
        Since $(\nabla \cdot \zeta)_{i} = \sum_{k}(\partial\zeta/\partial z_{k})_{ki}$ and each entry of a matrix is at most its norm, $\abs{(\nabla \cdot \zeta)_{i}(\x{x})} \leq dL_{1}r^{d+2}J(r)$ for each $i$, and therefore
        \begin{align*}
            \frac{1}{q(r)}\norm{(\nabla \cdot \zeta)(\x{x})} \leq \frac{d^{\frac{3}{2}}L_{1}}{(2\pi)^{\frac{d}{2}}}J(r)
        \end{align*}
    \end{quote}
    Putting the three terms together,
    \begin{align*}
        \frac{1}{q(r)}\x{m}_{1,r}(\x{x}) - \frac{1}{2}(\nabla \cdot D)(\x{x}) = -\frac{1}{2q(r)}(\nabla \cdot \zeta)(\x{x}) + \frac{1}{q(r)}\x{a}_{3}
    \end{align*}
    and so
    \begin{align*}
        \sup_{\x{x} \in K}\norm{\frac{1}{q(r)}\x{m}_{1,r}(\x{x}) - \frac{1}{2}(\nabla \cdot D)(\x{x})} \leq \frac{d^{\frac{3}{2}}L_{1}}{2(2\pi)^{\frac{d}{2}}}J(r) + \frac{L_{2}r}{8(2\pi)^{\frac{d}{2}}}(I_{3} + I_{5} + I_{7})
    \end{align*}
    As $r \to 0$, $J(r) \to 0$, because $\varepsilon(r)/r \to \infty$ and $J(r)$ is the tail of a convergent integral, and the second term has a factor $r$. This proves \ref{cond:uni-conv-first-moment} with $\x{b} = \frac{1}{2}\nabla \cdot D$, which is continuous since $D$ is $C^{2}$.
\end{proof}

\section{Generalisation Beyond MPF} \label{appendix:generalization-of-main-theorems}
In this section, we derive the GSM objective from a more general starting objective than the one considered in the main text (\autoref{eq:mpf-objective}). Motivated by the discussion in Section \ref{appendix:sec-cond-nce-discussion}, consider for a convex function $f: (0, \infty) \to \R$,
\begin{align} \label{eq:mpf-f-nce-objective}
    \K_{f}(\x{\theta}) &= \underset{\x{x} \sim p}{\E}\left[\int_{\x{y} \in \Omega} g_{r}(\x{y}, \x{x})\H_{f}\left(\frac{p_{\x{\theta}}(\x{y})}{p_{\x{\theta}}(\x{x})}\right) \d \x{y}\right]
\end{align}
where $\H_{f} : (0, \infty) \to \R$ defined as
\begin{align*}
    \H_{f}(z) = -f'\left(\frac{1}{z}\right) + zf'(z) - f(z)
\end{align*}
As noted in Section \ref{appendix:sec-cond-nce-discussion}, if we choose $f(z) = -\sqrt{z}$, yields MPF objective. We now state the following lemma which shows that $\H_{f}$ agrees with $\sqrt{z}$ near $z = 1$ up to an affine change.
\begin{lemma} \label{lemma:mpf-f-nce-is-mpf-near-1}
    Let $f$ be $C^{3}$, then $\varphi_{f} : (0, \infty) \to \R$ defined as
    \begin{align*}
        \varphi_{f}(z) = \H_{f}(z) + f(1) - 4f''(1)\left(\sqrt{z} - 1\right)
    \end{align*}
    is $C^{2}$ and $\varphi_{f}(1) = \varphi'_{f}(1) = \varphi''_{f}(1) = 0$.
\end{lemma}
\begin{proof}
    As $f$ is $C^{3}$, it follows that $\varphi_{f}$ is $C^{2}$. $\varphi_{f}(1) = 0$ trivially follows. Also
    \begin{align*}
        \varphi'_{f}(z) &= \frac{1}{z^{2}}f''\left(\frac{1}{z}\right) + zf''(z) - \frac{2f''(1)}{\sqrt{z}} \\
        \varphi''_{f}(z) &= -\frac{1}{z^{4}}f'''\left(\frac{1}{z}\right) - \frac{2}{z^{3}}f''\left(\frac{1}{z}\right) + f''(z) + zf'''(z) + \frac{f''(1)}{z^{\frac{3}{2}}}
    \end{align*}
    Putting $z = 1$, the second conclusion also follows.
\end{proof}
For the family of connectivity functions $g_{r}$, define the analogues of $\I, \C$ and $\Psi$ as
\begin{align*}
    \I_{f}(\x{x}, \x{\theta}; r) = \int_{\x{y}}g_{r}(\x{y}, \x{x})\H_{f}\left(\frac{p_{\x{\theta}}(\x{y})}{p_{\x{\theta}}(\x{x})}\right) \d \x{y}, \quad \C_{f}(\x{x}; r) = -f(1) \C(\x{x}; r)
\end{align*}
and
\begin{align*}
    \Psi_{f}(\x{x}, \x{\theta}; r) = \frac{\I_{f}(\x{x}, \x{\theta}; r) - \C_{f}(\x{x}; r)}{\Q(r)}
\end{align*}
\begin{theorem}
    Suppose the assumptions of \autoref{theorem:sm-convex} hold, and let $f$ be $C^{3}$ with $f''(1) > 0$. Then for every $\x{x} \in \Omega$, $\Psi_{f}(\x{x}, \x{\theta}; r)$ has the same pointwise limit as $\Psi(\x{x}, \x{\theta}; r)$ when $\Q(r) = f''(1)q(r)$. Further, if $\Psi_{f}$ satisfies \ref{ass:dominating-function-for-psi}, then $\lim_{r \to 0} \E_{\x{x} \sim p}\left[\Psi_{f}(\x{x}, \x{\theta}; r)\right]$ equals to objective in \autoref{eq:main-objective-in-compact-form}.
\end{theorem}
\begin{proof}
    From \autoref{lemma:mpf-f-nce-is-mpf-near-1}, we have
    \begin{align*}
        \H_{f}(z) = -f(1) + 4f''(1)(\sqrt{z} - 1) + \varphi(z)
    \end{align*}
    Let $R(\x{x}, \x{y}) = \frac{p_{\x{\theta}}(\x{y})}{p_{\x{\theta}}(\x{x})} = \exp\left(E_{\x{\theta}}(\x{x}) - E_{\x{\theta}}(\x{y})\right)$, then
    \begin{align*}
        \I_{f}(\x{x}; \x{\theta}; r) &= -f(1)\C(\x{x}; r) + 4f''(1)\int_{\x{y} \in \Omega}g_{r}(\x{y}, \x{x})\left(\exp\left(\frac{1}{2}[E_{\x{\theta}}(\x{x}) - E_{\x{\theta}}(\x{y})]\right) - 1\right) \d \x{y} \\
        &\quad + \int_{\x{y} \in \Omega}g_{r}(\x{y}, \x{x})\varphi_{f}(R(\x{x}, \x{y})) \d \x{y}
    \end{align*}
    Since the normalisation of $\frac{q(r)}{4}$ is used in \autoref{theorem:sm-convex}, we have
    \begin{align*}
        \Psi_{f}(\x{x}, \x{\theta}; r) &= \Psi(\x{x}, \x{\theta}; r) + \frac{1}{f''(1)q(r)}\int_{\x{y} \in \Omega}g_{r}(\x{y}, \x{x})\varphi_{f}(R(\x{x}, \x{y})) \d \x{y}
    \end{align*}
    It suffices to show that last term goes to $0$. Fix a $\x{x}$, then exists $\rho(\x{x}) > 0$ such that $N_{\rho(\x{x})}(\x{x}) = \{\x{y} : \norm{\x{y} - \x{x}} \leq \rho(\x{x})\} \subset \Omega$. As $\nabla E_{\x{\theta}}$ is continuous and $N_{\rho(\x{x})}(\x{x})$ is compact, define
    \begin{align*}
        G_{\x{x}} = \sup_{\x{p} \in N_{\rho(\x{x})}(\x{x})}\norm{\nabla E_{\x{\theta}}(\x{p})}
    \end{align*}
    Take $K = \{\x{x}\}$ and by \ref{cond:delta-go-to-0-and-g-0}, $\delta_{r}(K) \to 0$ as $r \to 0$ implies
    \begin{align*}
        \exists \: r_{1} > 0 \text{ such that } \forall \: 0 < r < r_{1} \text{ we have } \delta_{r}(K) < \rho(\x{x}) \text{ and } G_{\x{x}}\exp(G_{\x{x}}\rho(\x{x}))\delta_{r}(K) < 1
    \end{align*}
    From now on, we take $0 < r < r_{1}$ then $N_{\delta_{r}(K)}(\x{x}) \subset N_{\rho(\x{x})}(\x{x}) \subset \Omega$. Therefore
    \begin{align*}
        \abs{\int_{\x{y} \in \Omega}g_{r}(\x{y}, \x{x})\varphi_{f}(R(\x{x}, \x{y})) \d \x{y}} &\leq \int_{\x{y} \in \Omega}g_{r}(\x{y}, \x{x})\abs{\varphi_{f}(R(\x{x}, \x{y}))} \d \x{y} \\
        &= \int_{\x{y} : \norm{\x{y} - \x{x}} \leq \delta_{r}(K)}g_{r}(\x{y}, \x{x})\abs{\varphi_{f}(R(\x{x}, \x{y}))} \d \x{y}
    \end{align*}
    where we have used the fact that $g_{r}$ is non-negative. Therefore, we only have to consider $\x{y}$ which satisfies $\norm{\x{y} - \x{x}} \leq \delta_{r}(K)$. As $\varphi_{f}$ is $C^{2}$, expanding $\varphi$ around 1 (because of \autoref{lemma:mpf-f-nce-is-mpf-near-1}), we have
    \begin{align} \label{eq:varphi-f-relation-with-second-derivative}
        \varphi_{f}(R(\x{x}, \x{y})) &= \frac{1}{2}\varphi_{f}''\left(1 + \lambda(R(\x{x}, \x{y}) - 1)\right)\left(R(\x{x}, \x{y}) - 1\right)^{2}
    \end{align}
    where $\lambda \in (0, 1)$. The next step is to bound $R(\x{x}, \x{y}) - 1$ when $\norm{\x{y} - \x{x}} \leq \delta_{r}(K)$.
    \begin{quote}
        For any $s \in \R$, applying Taylor's theorem on $\exp(s)$ we have
        \begin{align*}
            \exp(s) = 1 + \exp(q)s
        \end{align*}
        where $q$ is in between $0$ and $s$ which means $q \leq \abs{q} \leq \abs{s}$. Therefore $\abs{\exp(s) - 1} = \exp(q)\abs{s} \leq \exp(\abs{s})\abs{s}$. Taking $s = E_{\x{\theta}}(\x{x}) - E_{\x{\theta}}(\x{y})$, we have
        \begin{align*}
            \abs{R(\x{x}, \x{y}) - 1} \leq \exp\left(\abs{E_{\x{\theta}}(\x{x}) - E_{\x{\theta}}(\x{y})}\right)\abs{E_{\x{\theta}}(\x{x}) - E_{\x{\theta}}(\x{y})}
        \end{align*}
        Since $E_{\x{\theta}} \in C^{1}$ (actually $C^{2}$), applying Taylor's theorem, we have
        \begin{align*}
            &E_{\x{\theta}}(\x{y}) = E_{\x{\theta}}(\x{x}) + \nabla E_{\x{\theta}}(\x{x} + \lambda'(\x{y} - \x{x}))^{\top}(\x{y} - \x{x}) \\
            \implies & \abs{E_{\x{\theta}}(\x{x}) - E_{\x{\theta}}(\x{y})} \leq \norm{\nabla E_{\x{\theta}}(\x{x} + \lambda'(\x{y} - \x{x}))} \norm{\x{y} - \x{x}}
        \end{align*}\
        where $\lambda' \in (0, 1)$. The point $\x{v} = \x{x} + \lambda'(\x{y} - \x{x})$ satisfies $\norm{\x{v} - \x{x}} \leq \norm{\x{y} - \x{x}} \leq \delta_{r}(K) < \rho(\x{x})$ which is finite. Hence
        \begin{align*}
            \exp\left(\abs{E_{\x{\theta}}(\x{x}) - E_{\x{\theta}}(\x{y})}\right) \leq \exp(G_{\x{x}}\rho(\x{x}))
        \end{align*}
        Putting it back, we have
        \begin{align} \label{eq:bound-on-R}
            \abs{R(\x{x}, \x{y}) - 1} \leq G_{\x{x}}\exp(G_{\x{x}}\rho(\x{x}))\norm{\x{y} - \x{x}}
        \end{align}
    \end{quote}
    As $\varphi_{f}''$ is continuous (because $\varphi$ is $C^{2}$), define $\omega : (0, 1) \to \R$ as
    \begin{align*}
        \omega(t) = \sup_{s : \abs{s - 1} \leq t} \abs{\varphi_{f}''(s)}
    \end{align*}
    which is finite, since for $t < 1$ the interval $[1 - t, 1 + t]$ is a compact subset of $(0, \infty)$ on which $\varphi_{f}''$ is continuous. From \autoref{eq:varphi-f-relation-with-second-derivative}, $\varphi_{f}''$ is evaluated at $1 + \lambda(R(\x{x}, \x{y}) - 1)$. We also have
    \begin{align*}
        \abs{1 + \lambda(R(\x{x}, \x{y}) - 1) - 1} = \lambda\abs{R(\x{x}, \x{y}) - 1} &\leq G_{\x{x}}\exp(G_{\x{x}}\rho(\x{x}))\delta_{r}(K) \\
        &< 1
    \end{align*}
    So
    \begin{align*}
        \abs{\varphi_{f}(R(\x{x}, \x{y}))} &\leq \underbrace{\frac{1}{2}G^{2}_{\x{x}}\exp\left(2G_{\x{x}}\rho(\x{x})\right)}_{L_{\x{x}}}\omega\left(G_{\x{x}}\exp(G_{\x{x}}\rho(\x{x}))\delta_{r}(K)\right)\norm{\x{y} - \x{x}}^{2}
    \end{align*}
    With this, the integral bound becomes
    \begin{align*}
        \abs{\frac{1}{q(r)}\int_{\x{y} \in \Omega}g_{r}(\x{y}, \x{x})\varphi_{f}(R(\x{x}, \x{y})) \d \x{y}} &\leq \frac{1}{q(r)}\int_{\x{y} : \norm{\x{y} - \x{x}} \leq \delta_{r}(K)}g_{r}(\x{y}, \x{x})\abs{\varphi_{f}(R(\x{x}, \x{y}))} \d \x{y} \\
        &\leq \frac{1}{q(r)}L_{\x{x}}\omega\left(G_{\x{x}}\exp(G_{\x{x}}\rho(\x{x}))\delta_{r}(K)\right)\int_{\x{y} : \norm{\x{y} - \x{x}} \leq \delta_{r}(K)}g_{r}(\x{y}, \x{x}) \norm{\x{y} - \x{x}}^{2} \d \x{y} \\
        &\leq L_{\x{x}}\omega\left(G_{\x{x}}\exp(G_{\x{x}}\rho(\x{x}))\delta_{r}(K)\right)\frac{\tr\left(M_{2, r}(\x{x})\right)}{q(r)}
    \end{align*}
    \begin{quote}
        We will now show that $\lim_{t \to 0} \omega(t) = 0$. Let $\epsilon > 0$, as $\varphi_{f}''$ is continuous at $1$ and $\varphi_{f}''(1) = 0$ (from \autoref{lemma:mpf-f-nce-is-mpf-near-1}), we have
        \begin{align*}
            \exists \: \eta > 0 \text{ such that } \forall \: s \in (0, \infty), \abs{s - 1} < \eta &\text{ we have } \abs{\varphi_{f}''(s)} < \epsilon
        \end{align*}
        Take any $t \in (0, 1), t < \eta$, then $\{s : \abs{s - 1} \leq t\} \subset \{s : \abs{s - 1} < \eta\}$. So
        \begin{align*}
            \sup_{s : \abs{s - 1} \leq t} \abs{\varphi_{f}''(s)} &\leq \sup_{s : \abs{s - 1} < \eta} \abs{\varphi_{f}''(s)} \\
            \omega(t) &\leq \epsilon
        \end{align*}
        So we showed
        \begin{align*}
            \forall \: t \in (0, 1), t < \eta \text{ we have } \omega(t) \leq \epsilon
        \end{align*}
        As $\epsilon$ is arbitrary, the conclusion follows (Observe that $\omega$ is non-negative by definition).
    \end{quote}
    From above, as $r \to 0$ by \ref{cond:delta-go-to-0-and-g-0}, $\delta_{r}(K) \to 0$ implies $\omega\left(G_{\x{x}}\exp(G_{\x{x}}\rho(\x{x}))\delta_{r}(K)\right) \to 0$. Also by \ref{cond:uni-conv-second-moment}, $\frac{\tr\left(M_{2, r}(\x{x})\right)}{q(r)} \to \tr(D(\x{x}))$. So we can conclude
    \begin{align*}
        \lim_{r \to 0}\frac{1}{q(r)}\int_{\x{y} \in \Omega}g_{r}(\x{y}, \x{x})\varphi_{f}(R(\x{x}, \x{y})) \d \x{y} = 0
    \end{align*}
    concluding that $\lim_{r \to 0} \Psi_{f}(\x{x}, \x{\theta}; r) = \lim_{r \to 0} \Psi(\x{x}, \x{\theta}; r)$.

    The statement about the expectation then follows from \ref{ass:dominating-function-for-psi} and the dominated convergence theorem, exactly as in \autoref{theorem:sm-convex}.
\end{proof}

\section{Proper Scoring Rule} \label{appendix:proper-scoring-rules}
Let $\mathcal{P}$ be a class of probability distributions on $\mathcal{X}$. A scoring rule $S(\x{x}, Q)$ is the loss incurred when $Q \in \mathcal{P}$ is quoted and $\x{x}$ is observed~\citep{parry2012proper}; we write $S(P, Q) = \E_{\x{x} \sim P}[S(\x{x}, Q)]$. The rule is \emph{proper} relative to $\mathcal{P}$ if $S(P, Q) \geq S(P, P)$ for all $P, Q \in \mathcal{P}$, and \emph{strictly proper} if, in addition, equality holds only for $Q = P$. It is \emph{local of order $m$} if $S(\x{x}, Q)$ depends on the density $q$ of $Q$ only through $q$ and its derivatives up to order $m$ at $\x{x}$.

For distributions on the real line, \citet{parry2012proper} recall that the logarithmic score is, up to affine transformation, the only proper local scoring rule of order zero, and show that proper local scoring rules exist for every even order and for no odd order, and that those of order two or higher can be computed without the normalising constant; see also \citet{local-proper-scoring-rule-ehm} for order two. Multivariate local scoring rules are studied by \citet{extensive_scoring_rule}. Proper scoring rules are also used to evaluate probabilistic forecasts~\citep{forecaster-lerch} and to train generative models~\citep{bortoli2025distributional}.

Let $D : \Omega \to \R^{d \times d}$ be $C^{1}$ with $D(\xx)$ symmetric positive-definite, as in \autoref{theorem:sm-convex}. For a density $q \in C^{2}$ that is positive on $\Omega$, define
\begin{align} \label{eq:main-objective-scoring-rule}
    S(\xx; q) = \frac{1}{2}\nabla \log q(\xx)^{\top}D(\xx)\nabla \log q(\xx) + \nabla \cdot \big(D(\xx)\nabla \log q(\xx)\big), \qquad \xx \in \Omega.
\end{align}
For $q = p_{\ttheta}$, we have $S(\xx; p_{\ttheta}) = \frac{1}{2}\nabla E_{\ttheta}(\xx)^{\top}D(\xx)\nabla E_{\ttheta}(\xx) - \nabla \cdot \big(D(\xx)\nabla E_{\ttheta}(\xx)\big)$, which is the pointwise limit in \autoref{eq:psi-pointwise-limit}, the integrand of \autoref{eq:main-objective-in-compact-form}, and the per-sample loss in \autoref{eq:exp-fam-per-sample-loss}. Write $q_{i} = \partial q/\partial x_{i}$ and $q_{ij} = \partial^{2}q/\partial x_{i}\partial x_{j}$. Since $\nabla \cdot (D\x{v}) = (\nabla \cdot D)^{\top}\x{v} + \tr(DJ_{\x{v}})$ for a vector field $\x{v}$ with Jacobian $J_{\x{v}}$, and $J_{\nabla \log q} = H_{\log q} = H_{q}/q - \nabla q\nabla q^{\top}/q^{2}$,
\begin{align} \label{eq:parry_eq4}
    S(\xx; q) = \sum_{i,j=1}^{d}D_{ij}(\xx)\left(\frac{q_{ij}(\xx)}{q(\xx)} - \frac{1}{2}\frac{q_{i}(\xx)q_{j}(\xx)}{q(\xx)^{2}}\right) + \sum_{j=1}^{d}(\nabla \cdot D)_{j}(\xx)\frac{q_{j}(\xx)}{q(\xx)}.
\end{align}
Hence $S$ is local of order two, and it is unchanged when $q$ is multiplied by a positive constant, so it does not involve the normalising constant of $q$. For $D = \mathbb{I}$, it is the Hyv\"{a}rinen score~\citep{Hyvarinen05a}. For general $D$, \autoref{eq:parry_eq4} is the order-two rule generated by the key function $-\frac{1}{2}q^{-1}\sum_{i,j}D_{ij}q_{i}q_{j}$ in the construction of \citet{extensive_scoring_rule}. On a domain with a boundary, however, the form of $S$ does not by itself make it proper: as in \autoref{prop:main-objective-in-gsm-form}, propriety requires a boundary condition, and it holds relative to the class of densities that satisfy it.

\begin{proposition} \label{prop:gsm-proper-scoring-rule}
    Let $\ttheta'$ be such that, for every $\x{\theta}$, $\E_{\xx \sim p_{\ttheta'}}\abs{S(\xx; p_{\ttheta})} < \infty$ and conditions \ref{cond:bounded-quad-form-model}--\ref{cond:finite-norm-for-unbounded-sets} of \autoref{prop:main-objective-in-gsm-form} hold with $p = p_{\ttheta'}$. Then, for every $\ttheta$,
    \begin{align} \label{eq:gsm-scoring-rule-gap}
        \E_{\xx \sim p_{\ttheta'}}[S(\xx; p_{\ttheta})] - \E_{\xx \sim p_{\ttheta'}}[S(\xx; p_{\ttheta'})] = \frac{1}{2}\E_{\xx \sim p_{\ttheta'}}\Big[\norm{\nabla \log p_{\ttheta}(\xx) - \nabla \log p_{\ttheta'}(\xx)}^{2}_{D(\xx)}\Big] \geq 0,
    \end{align}
    with equality if and only if $p_{\ttheta} = p_{\ttheta'}$ almost everywhere. Consequently, if these conditions hold for every $\ttheta'$, then $S$ is a strictly proper local scoring rule of order two relative to the model $\{p_{\ttheta}\}$.
\end{proposition}
\begin{proof}
    Since $S(\cdot\,; p_{\ttheta'})$ is defined, $p_{\ttheta'}$ is positive and $C^{2}$ on $\Omega$, in particular $C^{1}$. By \ref{cond:bounded-quad-form-model} with $\ttheta = \ttheta'$, $\E_{\xx \sim p_{\ttheta'}}[\nabla \log p_{\ttheta'}(\xx)^{\top}D(\xx)\nabla \log p_{\ttheta'}(\xx)] < \infty$. The proof of \autoref{prop:main-objective-in-gsm-form} in Appendix~\ref{appendix:proof-prop-main-objective-in-gsm-form} therefore applies with $p_{\ttheta'}$ in place of $p$, $\E_{\xx \sim p_{\ttheta'}}[S(\xx; p_{\ttheta})]$ in place of $\L(\ttheta)$, and the assumed integrability of $S(\cdot\,; p_{\ttheta})$ in place of \ref{ass:dominating-function-for-psi}. It shows that, for every $\ttheta$,
    \begin{align*}
        \E_{\xx \sim p_{\ttheta'}}[S(\xx; p_{\ttheta})] &= \frac{1}{2}\E_{\xx \sim p_{\ttheta'}}\Big[\norm{\nabla \log p_{\ttheta}(\xx) - \nabla \log p_{\ttheta'}(\xx)}^{2}_{D(\xx)}\Big] \\
        &\quad - \frac{1}{2}\E_{\xx \sim p_{\ttheta'}}\big[\nabla \log p_{\ttheta'}(\xx)^{\top}D(\xx)\nabla \log p_{\ttheta'}(\xx)\big]
    \end{align*}
    Substituting $\x{\theta} = \x{\theta}'$, we have
    \begin{align*}
        \E_{\x{x} \sim p_{\x{\theta}'}}[S(\x{x}; p_{\x{\theta}'})] &= -\frac{1}{2}\E_{\xx \sim p_{\ttheta'}}\big[\nabla \log p_{\ttheta'}(\xx)^{\top}D(\xx)\nabla \log p_{\ttheta'}(\xx)\big]
    \end{align*}
    Subtracting the above two equations, we have
    \begin{align*}
        \E_{\xx \sim p_{\ttheta'}}[S(\xx; p_{\ttheta})] - \E_{\x{x} \sim p_{\x{\theta}'}}[S(\x{x}; p_{\x{\theta}'})] = \frac{1}{2}\E_{\xx \sim p_{\ttheta'}}\Big[\norm{\nabla \log p_{\ttheta}(\xx) - \nabla \log p_{\ttheta'}(\xx)}^{2}_{D(\xx)}\Big]
    \end{align*}
    giving \autoref{eq:gsm-scoring-rule-gap}. The $\geq 0$ also follows because $D$ is positive-definite matrix and it induces a valid norm.

    If its right-hand side is zero, then $D^{\frac{1}{2}}\nabla \log p_{\ttheta} = D^{\frac{1}{2}}\nabla \log p_{\ttheta'}$ almost everywhere on $\Omega$, because $p_{\ttheta'} > 0$ on $\Omega$. This is the hypothesis of \autoref{prop:linear-operator-is-complete} with $p_{1} = p_{\ttheta}$ and $p_{2} = p_{\ttheta'}$, since $(\O p_{i})/p_{i} = D^{\frac{1}{2}}\nabla \log p_{i}$, so $p_{\ttheta} = p_{\ttheta'}$ almost everywhere. Conversely, if $p_{\ttheta} = p_{\ttheta'}$ almost everywhere, then the two densities, being continuous, coincide on $\Omega$, and the right-hand side is zero. The last statement follows from the definitions at the beginning of this section, since $S$ is local of order two by \autoref{eq:parry_eq4}.
\end{proof}

In particular, if the model is well specified, $p = p_{\ttheta^{*}}$, then \autoref{prop:gsm-proper-scoring-rule} with $\ttheta' = \ttheta^{*}$ shows that the population objective $\L(\ttheta) = \E_{\xx \sim p}[S(\xx; p_{\ttheta})]$ satisfies $\L(\ttheta) \geq \L(\ttheta^{*})$, with equality if and only if $p_{\ttheta} = p$ almost everywhere. This requires the conditions of \autoref{prop:gsm-proper-scoring-rule} only for the data density, not for every $\ttheta'$.

\paragraph{The boundary condition cannot be dropped.} We illustrate this with two one-dimensional examples, one on a bounded and one on an unbounded domain. In both, the unweighted rule $D \equiv 1$ is not proper, whereas a weight $D_{\phi}$ induced by a barrier function satisfies the conditions of \autoref{prop:gsm-proper-scoring-rule}.

First, let $\Omega = (0, 1)$, let the data density be uniform, $p \equiv 1$, and consider the truncated Gaussian densities $q_{k}(x) \propto \exp\big(-k(x - \frac{1}{2})^{2}/2\big)$ on $(0, 1)$, $k > 0$. Since $-k(x - \frac{1}{2})^{2}/2 = \frac{k}{2}x - \frac{k}{2}x^{2} - \frac{k}{8}$, the density $q_{k}$ is the member $\ttheta = (k/2, k)^{\top}$ of the exponential family with sufficient statistic $t(x) = (x, -x^{2}/2)$ considered in Section \ref{sec:exp-family}, and $p$ is its member $\ttheta = \x{0}$. The scores are $\partial_{x}\log p = 0$ and $\partial_{x}\log q_{k}(x) = -k(x - \frac{1}{2})$. With $D \equiv 1$, which is $D_{\phi}$ for $\phi(x) = x^{2}/2$, this gives $S(x; p) = 0$ and $S(x; q_{k}) = \frac{1}{2}k^{2}(x - \frac{1}{2})^{2} - k$, so that
\begin{align*}
    \E_{x \sim p}[S(x; q_{k})] = \frac{k^{2}}{24} - k < 0 = \E_{x \sim p}[S(x; p)] \qquad \text{for } 0 < k < 24.
\end{align*}
Hence $S$ is not proper relative to any class that contains $p$ and some $q_{k}$ with $0 < k < 24$. The left-hand side is smallest at $k = 12$, that is, at $\ttheta = (6, 12)^{\top}$, which is the limit of the estimator with $D \equiv 1$ in Appendix~\ref{appendix:exp-fam-comments}. Condition~\ref{cond:regularity-conditions-convex-set} fails here, because $p(x)\,\partial_{x}\log q_{k}(x) = -k(x - \frac{1}{2})$ tends to $k/2$ as $x \to 0$ and to $-k/2$ as $x \to 1$. In contrast, the entropic weight $D(x) = x(1 - x)$, which is $D_{\phi}$ for $\phi(x) = x\log x + (1 - x)\log(1 - x)$, vanishes at both endpoints. Every density in the family is bounded and positive on $[0, 1]$, and the first two derivatives of its log-density are bounded, so the family satisfies the conditions of \autoref{prop:gsm-proper-scoring-rule}; indeed, $\E_{x \sim p}[S(x; q_{k})] = k^{2}/240 \geq 0 = \E_{x \sim p}[S(x; p)]$.

Second, let $\Omega = (0, \infty)$ and consider the exponential densities $p_{\lambda}(x) = \lambda e^{-\lambda x}$, $\lambda > 0$, whose score is $\partial_{x}\log p_{\lambda}(x) = -\lambda$; let $\lambda^{*}$ denote the true rate. With $D \equiv 1$, $S(x; p_{\lambda}) = \frac{1}{2}\lambda^{2}$, so that
\begin{align*}
    \E_{x \sim p_{\lambda^{*}}}[S(x; p_{\lambda})] = \frac{1}{2}\lambda^{2} < \frac{1}{2}(\lambda^{*})^{2} = \E_{x \sim p_{\lambda^{*}}}[S(x; p_{\lambda^{*}})] \qquad \text{for } 0 < \lambda < \lambda^{*},
\end{align*}
and $S$ is not proper relative to the exponential densities. The empirical objective is also $\frac{1}{2}\lambda^{2}$, whatever the data, and it is minimised over $\lambda \in \R$ at $\hat{\lambda} = 0$, outside the parameter space. Condition~\ref{cond:regularity-conditions-convex-set} fails at $0$, because $p_{\lambda^{*}}(x)\,\partial_{x}\log p_{\lambda}(x) = -\lambda^{*}\lambda e^{-\lambda^{*}x}$ tends to $-\lambda^{*}\lambda \neq 0$ as $x \to 0$. In contrast, the weight $D(x) = x$, which is $D_{\phi}$ for $\phi(x) = x\log x$, vanishes at $0$: the boundary term becomes $-\lambda^{*}\lambda\,x e^{-\lambda^{*}x} \to 0$, and since all the moments involved are finite, the exponential densities satisfy the conditions of \autoref{prop:gsm-proper-scoring-rule}. Now $S(x; p_{\lambda}) = \frac{1}{2}\lambda^{2}x - \lambda$, so that
\begin{align*}
    \E_{x \sim p_{\lambda^{*}}}[S(x; p_{\lambda})] - \E_{x \sim p_{\lambda^{*}}}[S(x; p_{\lambda^{*}})] = \frac{(\lambda - \lambda^{*})^{2}}{2\lambda^{*}} \geq 0,
\end{align*}
and the empirical objective $\frac{1}{2}\lambda^{2}\bar{x} - \lambda$, where $\bar{x}$ is the sample mean, is minimised at $\hat{\lambda} = 1/\bar{x}$, which is the maximum likelihood estimator. The weight $D(x) = x^{2}$ of non-negative score matching~\citep{Hyvaerinen2007}, which is $D_{\phi}$ for $\phi(x) = -\log x$, also satisfies the conditions; it gives $S(x; p_{\lambda}) = \frac{1}{2}\lambda^{2}x^{2} - 2\lambda x$ and $\hat{\lambda} = 2\bar{x}/\overline{x^{2}}$, where $\overline{x^{2}}$ is the sample mean of $x^{2}$.

\section{Proof of \autoref{theorem:convexity}} \label{appendix:convexity_proof}
 \begin{proof}
    \begin{enumerate}[label=(\alph*)]
        \item Fix $\xx \in \Omega$ and $\ttheta \in \R^{r}$, then
        \begin{align*}
            \nabla \log p_{\x{\theta}}(\x{x}) = J_{\x{t}}^{\top}(\x{x})\x{\theta} + \nabla \beta(\x{x})
        \end{align*}
        Let $\x{a}(\x{x}) = J_{\x{t}}^{\top}(\x{x})\x{\theta}$, so $\nabla \log p_{\x{\theta}}(\x{x}) = \x{a}(\x{x}) + \nabla \beta(\x{x})$. Since $\x{t}$ is $C^{2}$, we have $\x{a} \in C^{1}$. For brevity, we suppress the argument $\xx \in \Omega$. Substituting $\nabla \log p_{\x{\theta}}$ in \autoref{eq:exp-fam-per-sample-loss}, we have
        \begin{align*}
            S(\x{x}; \x{\theta}) &= \frac{1}{2}(\x{a} + \nabla \beta)^{\top}D(\x{a} + \nabla \beta) + \nabla\cdot\left(D\left(\x{a} + \nabla \beta\right)\right) \\
            &= \frac{1}{2}\x{a}^{\top}D\x{a} + \x{a}^{\top}D\nabla \beta + \frac{1}{2}\nabla \beta^{\top}D\nabla \beta + \nabla \cdot \left(D\x{a}\right) + \nabla \cdot (D\nabla \beta)
        \end{align*}
        Using $\x{a} = \sum_{\ell=1}^{r}\theta_{\ell}\nabla t_{l}$, observe
        \begin{align*}
            \nabla \cdot (D\x{a}) &= \nabla \cdot (DJ_{\x{t}}^{\top}\x{\theta}) \\
            &= \sum_{\ell = 1}^{r} \theta_{\ell} \nabla \cdot (D \nabla t_{l})
        \end{align*}
        and
        \begin{align*}
            \x{a}^{\top}D\nabla \beta = \sum_{\ell=1}^{r}\theta_{\ell} \nabla t_{\ell}^{\top}D\nabla \beta
        \end{align*}
        Putting it back, we have
        \begin{align*}
            S(\x{x}; \x{\theta}) &= \frac{1}{2}\x{\theta}^{\top}J_{\x{t}}DJ_{\x{t}}^{\top}\x{\theta} + \sum_{\ell = 1}^{r}\theta_{\ell}\left[\nabla \cdot (D\nabla t_{l}) + \nabla t_{\ell}^{\top} D\nabla\beta\right] + \left(\frac{1}{2}\nabla \beta^{\top} D \nabla \beta + \nabla \cdot (D \nabla \beta)\right) \\
            &= \frac{1}{2}\x{\theta}^{\top}\Gamma\x{\theta} + \x{\theta}^{\top}\x{g} + c
        \end{align*}
        where we have used \autoref{eq:exp-fam-Gamma-g-c} in the last step. Now,
        \begin{align*}
            \hat{\L}_{N}(\x{\theta}) &= \frac{1}{N}\sum_{i=1}^{N}S(\x{x}_{i}; \x{\theta}) \\
            &= \frac{1}{2}\x{\theta}^{\top}\left(\frac{1}{N}\sum_{i=1}^{N}\Gamma(\x{x}_{i})\right)\x{\theta} + \x{\theta}^{\top}\left(\frac{1}{N}\sum_{i=1}^{N}\x{g}(\x{x}_{i})\right) + \frac{1}{N}\sum_{i=1}^{N}c(\x{x}_{i}) \\
            &= \frac{1}{2}\ttheta^{\top}\Gamma_{N}\ttheta + \x{g}_{N}^{\top}\ttheta + c_{N}
        \end{align*}
        which concludes the first part. For a single sample $\x{x}$, consider for any $\ell \in \{1,\dots,d\}$
        \begin{align*}
            g_{\ell} &= \nabla \cdot\ (D\nabla t_{\ell}) + \nabla t_{\ell}^{\top}D\nabla \beta \\
            &= \underbrace{\tr(DH_{t_{\ell}})}_{=\tau_{\ell}} + \nabla t_{\ell}^{\top}(\nabla \cdot D) + \nabla t_{\ell}^{\top}D\nabla \beta
        \end{align*}
        So for a single $\x{x}$, we have
        \begin{align*}
            \x{g}(\x{x}) = \x{\tau}(\x{x}) + J_{\x{t}}(\x{x})(\nabla \cdot D)(\x{x}) + J_{\x{t}}(\x{x})D(\x{x})\nabla \beta(\x{x})
        \end{align*}
        Therefore, the expression of $g_{N}$ given in proposition statement follows.
        \item Each $\Gamma(\x{x}_{i})$ is symmetric positive-semidefinite, so $\Gamma_{N}$ is symmetric positive-semidefinite. Consider for any $\x{v} \in \R^{r}$, we have
        \begin{align*}
            \x{v}^{\top}\Gamma_{N}\x{v} &= \frac{1}{N}\sum_{i=1}^{N}\left(\x{v}^{\top}J_{\x{t}}(\x{x}_{i})\right)D(\x{x}_{i})\left(J_{\x{t}}(\x{x}_{i})^{\top}\x{v}\right) \\
            &= \frac{1}{N}\sum_{i=1}^{N}\x{w}^{\top}_{i}D(\x{x}_{i})\x{w}_{i}
        \end{align*}
        where $\x{w}_{i} = J_{\x{t}}(\x{x}_{i})^{\top}\x{v}$. As $D$ is symmetric positive-definite, we have
        \begin{align*}
            \x{v}^{\top}\Gamma_{N}\x{v} = 0
            \iff \x{w}_{i} = \x{0} \ \text{ for } i = 1, \dots, N
            \iff M_{N}\x{v} = \begin{bmatrix} J_{t}(\xx_{1})^{\top}\x{v} \\ \vdots \\ J_{t}(\xx_{N})^{\top}\x{v} \end{bmatrix} = \x{0}
        \end{align*}
        Hence $\Gamma_{N} \succ 0$ if and only if null space of $M_{N}$ only contains $\{\x{0}\}$, that is, $\rank M_{N} = r$.
        \item As $\hat{\L}_{N}$ is quadratic in $\x{\theta}$ and $\Gamma_{N}$ is symmetric, its Hessian is $H_{\hat{\L}_{N}}(\x{\theta}) = \Gamma_{N}$. Convexity trivially follows as $\Gamma_{N}$ is positive-semidefinite from part (b). Similarly, if $\Gamma_{N} \succ 0$, then strict convexity of $\hat{\L}$ also follows. Let $\bar{\ttheta}$ satisfy $\Gamma_{N}\bar{\ttheta} = -\x{g}_{N}$. Completing the square, using the symmetry of $\Gamma_{N}$, gives
        \begin{align} \label{eq:exp-fam-complete-square}
            \hat{\L}_{N}(\ttheta) = \hat{\L}_{N}(\bar{\ttheta}) + \frac{1}{2}(\ttheta - \bar{\ttheta})^{\top}\Gamma_{N}(\ttheta - \bar{\ttheta}), \qquad \ttheta \in \R^{r}.
        \end{align}
        If $\Gamma_{N} \succ 0$, then $\bar{\ttheta} = -\Gamma_{N}^{-1}\x{g}_{N}$ is the only such point, and \autoref{eq:exp-fam-complete-square} shows that it is the unique minimiser.
    \end{enumerate}
\end{proof}

\section{Proof of \autoref{theorem:consistency}}
\begin{proof}
     From \autoref{theorem:convexity}(a), we have
    \begin{align*}
        \hat{\L}_{N}(\ttheta) = \frac{1}{2}\ttheta^{\top}\Gamma_{N}\ttheta + \x{g}_{N}^{\top}\ttheta + c_{N},
    \end{align*}
    and since $\Gamma_{N}$ is symmetric (\autoref{theorem:convexity}(b)), the gradient of $\hat{\L}_{N}$ with respect to $\ttheta$ is $\nabla_{\ttheta}\hat{\L}_{N}(\ttheta) = \Gamma_{N}\ttheta + \x{g}_{N}$. Likewise, $\Gamma_{0}$ is symmetric, because every $\Gamma(\x{x})$ is, so by \autoref{eq:exp-fam-population}, the gradient of $\L$ is $\Gamma_{0}\ttheta + \x{g}_{0}$. Since $\Gamma_{0} \succ 0$ by \ref{ass:exp-fam-identifiable}, $\L$ is strictly convex, and setting its gradient to zero shows that its unique minimiser $\ttheta_{0} = -\Gamma_{0}^{-1}\x{g}_{0}$ satisfies
    \begin{align} \label{eq:exp-fam-population-foc}
        \Gamma_{0}\ttheta_{0} = -\x{g}_{0}.
    \end{align}
    Since $\Gamma_{N} \succeq 0$, it fails to be positive-definite exactly when it is singular, and in that case we set $\Gamma_{N}^{-1} = 0$. With this convention, let $\hat{\ttheta}_{N} = -\Gamma_{N}^{-1}\x{g}_{N}$ for every $N$.

    \begin{enumerate}[label=(\alph*)]
        \item Setting the gradient of $\hat{\L}_{N}$ to zero yields the first-order condition $\Gamma_{N}\ttheta = -\x{g}_{N}$. If $\Gamma_{N} \succ 0$, this equation has the unique solution $\hat{\ttheta}_{N} = -\Gamma_{N}^{-1}\x{g}_{N}$, which is the unique minimiser of $\hat{\L}_{N}$ over $\R^{r}$ by \autoref{theorem:convexity}(c).

        It remains to show that, almost surely, $\Gamma_{N} \succ 0$ for all sufficiently large $N$. Since $\Gamma_{N}$ and $\x{g}_{N}$ are sample averages of the i.i.d.\ random matrices $\Gamma(\x{x}_{i})$ and vectors $\x{g}(\x{x}_{i})$, we apply the strong law of large numbers \citep[Theorem~2.4.1]{durrett2019} to each entry. The entries are integrable by \ref{ass:exp-fam-moments}, since they are bounded in absolute value by $\norm{\Gamma(\x{x}_{i})}$ and $\norm{\x{g}(\x{x}_{i})}$. Therefore,
        \begin{align} \label{eq:exp-fam-slln}
            \Gamma_{N} \asto \Gamma_{0} \quad \text{and} \quad \x{g}_{N} \asto \x{g}_{0} \quad \text{as } N \to \infty.
        \end{align}
        The determinant is a polynomial in the entries of a matrix, so the continuous mapping theorem \citep[Theorem~2.3]{vanderVaart1998} gives $\det\Gamma_{N} \asto \det\Gamma_{0}$, and $\det\Gamma_{0} > 0$ because $\Gamma_{0} \succ 0$. Hence, almost surely, there is an $N_{0}$, which depends on the sample, such that $\det\Gamma_{N} > 0$ for all $N \geq N_{0}$. For these $N$, the eigenvalues of $\Gamma_{N}$ are nonnegative, because $\Gamma_{N} \succeq 0$ (\autoref{theorem:convexity}(b)), and their product $\det\Gamma_{N}$ is positive. So all eigenvalues of $\Gamma_{N}$ are positive, that is, $\Gamma_{N} \succ 0$.

        \item The map $A \mapsto A^{-1}$ is continuous on the open set of invertible matrices, which contains $\Gamma_{0}$. Setting $A^{-1} = 0$ for singular $A$ does not change this. From \autoref{eq:exp-fam-slln} and the continuous mapping theorem \citep[Theorem~2.3]{vanderVaart1998}, we get
        \begin{align*}
            \Gamma_{N}^{-1} \asto \Gamma_{0}^{-1} \quad \text{and} \quad \hat{\ttheta}_{N} = -\Gamma_{N}^{-1}\x{g}_{N} \asto -\Gamma_{0}^{-1}\x{g}_{0} = \ttheta_{0}.
        \end{align*}
        \item Whenever $\Gamma_{N} \succ 0$, we have $\Gamma_{N}\hat{\ttheta}_{N} = -\x{g}_{N}$, and $\Gamma_{0}\ttheta_{0} = -\x{g}_{0}$ by \autoref{eq:exp-fam-population-foc}. Subtracting these equations gives
        \begin{align*}
            \Gamma_{N}(\hat{\ttheta}_{N} - \ttheta_{0}) = -(\x{g}_{N} - \x{g}_{0}) - (\Gamma_{N} - \Gamma_{0})\ttheta_{0}.
        \end{align*}
        Multiplying both sides by $\sqrt{N}\,\Gamma_{N}^{-1}$, we obtain
        \begin{align} \label{eq:exp-fam-clt-expansion}
            \sqrt{N}(\hat{\ttheta}_{N} - \ttheta_{0}) = -\Gamma_{N}^{-1}\sqrt{N}\big[(\x{g}_{N} - \x{g}_{0}) + (\Gamma_{N} - \Gamma_{0})\ttheta_{0}\big] = -\Gamma_{N}^{-1}\,\frac{1}{\sqrt{N}}\sum_{i=1}^{N}\x{Z}_{i},
        \end{align}
        where
        \begin{align*}
            \x{Z}_{i} = \big[\x{g}(\x{x}_{i}) - \x{g}_{0}\big] + \big[\Gamma(\x{x}_{i}) - \Gamma_{0}\big]\ttheta_{0} = \Gamma(\x{x}_{i})\ttheta_{0} + \x{g}(\x{x}_{i}),
        \end{align*}
        and the last equality uses \autoref{eq:exp-fam-population-foc}. The $\x{Z}_{i}$ are i.i.d., and by construction $\E[\x{Z}_{i}] = \x{0}$, since $\Gamma_{0} = \E[\Gamma(\x{x})]$ and $\x{g}_{0} = \E[\x{g}(\x{x})]$. Hence $\Cov(\x{Z}_{i}) = \E[\x{Z}_{i}\x{Z}_{i}^{\top}] = \Sigma_{0}$. Since $\Sigma_{0}$ is finite, the multivariate central limit theorem \citep[Example~2.18]{vanderVaart1998} gives
        \begin{align*}
            \frac{1}{\sqrt{N}}\sum_{i=1}^{N}\x{Z}_{i} \dto \N(\x{0}, \Sigma_{0}).
        \end{align*}
        Furthermore, $\Gamma_{N}^{-1} \asto \Gamma_{0}^{-1}$ by (b), and hence in probability and in distribution~\citep[Theorem~2.7(i)]{vanderVaart1998}. Then, by applying Slutsky's theorem~\citep[Lemma~2.8]{vanderVaart1998}, we get
        \begin{align*}
            -\Gamma_{N}^{-1}\,\frac{1}{\sqrt{N}}\sum_{i=1}^{N}\x{Z}_{i} \dto \N(\x{0}, \Gamma_{0}^{-1}\Sigma_{0}\Gamma_{0}^{-1}),
        \end{align*}
        since $-\Gamma_{0}^{-1}\x{U} \sim \N(\x{0}, \Gamma_{0}^{-1}\Sigma_{0}\Gamma_{0}^{-1})$ for $\x{U} \sim \N(\x{0}, \Sigma_{0})$, as $\Gamma_{0}^{-1}$ is symmetric. Finally, \autoref{eq:exp-fam-clt-expansion} holds whenever $\Gamma_{N} \succ 0$, so by (a) it holds, almost surely, for all $N \geq N_{0}$. Hence the difference between $\sqrt{N}(\hat{\ttheta}_{N} - \ttheta_{0})$ and the right-hand side of \autoref{eq:exp-fam-clt-expansion} is zero for all $N \geq N_{0}$, almost surely, so it converges to zero almost surely, and hence in probability. Hence, $\sqrt{N}(\hat{\ttheta}_{N} - \ttheta_{0})$ has the same limit in distribution, $\N(\x{0}, \Gamma_{0}^{-1}\Sigma_{0}\Gamma_{0}^{-1})$~\citep[Theorem~2.7(iv)]{vanderVaart1998}.

        \item Since $\ttheta_{0} \in \intr\Theta$, there is an $\epsilon > 0$ such that every $\ttheta$ with $\norm{\ttheta - \ttheta_{0}} < \epsilon$ lies in $\Theta$. By (b), almost surely $\norm{\hat{\ttheta}_{N} - \ttheta_{0}} < \epsilon$, and hence $\hat{\ttheta}_{N} \in \Theta$, for all sufficiently large $N$.

        \item Write $\hat{\theta}_{N,a}$ for the $a$-th entry of $\hat{\ttheta}_{N}$, and $\Gamma_{ja}(\xx)$ and $g_{j}(\xx)$ for the entries of $\Gamma(\xx)$ and $\x{g}(\xx)$. Since $\big(\Gamma(\x{x}_{i})\hat{\ttheta}_{N} + \x{g}(\x{x}_{i})\big)_{j} = \sum_{a=1}^{r}\Gamma_{ja}(\x{x}_{i})\hat{\theta}_{N,a} + g_{j}(\x{x}_{i})$, expanding the product in the definition of $\hat{\Sigma}_{N}$ gives, for $j, k = 1, \dots, r$,
        \begin{align}\label{eq:exp-fam-sandwich-expansion}
            [\hat{\Sigma}_{N}]_{jk} &= \sum_{a,b=1}^{r}\hat{\theta}_{N,a}\hat{\theta}_{N,b}\,\frac{1}{N}\sum_{i=1}^{N}\Gamma_{ja}(\x{x}_{i})\Gamma_{kb}(\x{x}_{i})  \\
            &\quad + \sum_{a=1}^{r}\hat{\theta}_{N,a}\,\frac{1}{N}\sum_{i=1}^{N}\big[\Gamma_{ja}(\x{x}_{i})g_{k}(\x{x}_{i}) + g_{j}(\x{x}_{i})\Gamma_{ka}(\x{x}_{i})\big] + \frac{1}{N}\sum_{i=1}^{N}g_{j}(\x{x}_{i})g_{k}(\x{x}_{i}). \nonumber
        \end{align}
        Each entry of $\Gamma(\x{x})$ and $\x{g}(\x{x})$ is bounded in absolute value by $\norm{\Gamma(\x{x})}$ and $\norm{\x{g}(\x{x})}$, respectively, and $\abs{uv} \leq \frac{1}{2}(u^{2} + v^{2})$. So, when $\E\norm{\Gamma(\x{x})}^{2} < \infty$ and $\E\norm{\x{g}(\x{x})}^{2} < \infty$, every product averaged in \autoref{eq:exp-fam-sandwich-expansion} is integrable. In particular, $\Sigma_{0}$ is finite, because $[\Sigma_{0}]_{jk}$ equals the right-hand side of \autoref{eq:exp-fam-sandwich-expansion} with $\hat{\ttheta}_{N}$ replaced by $\ttheta_{0}$ and each sample average replaced by the corresponding expectation. By the strong law of large numbers \citep[Theorem~2.4.1]{durrett2019}, each sample average in \autoref{eq:exp-fam-sandwich-expansion} converges almost surely to that expectation, and $\hat{\ttheta}_{N} \asto \ttheta_{0}$ by (b). The right-hand side of \autoref{eq:exp-fam-sandwich-expansion} is a polynomial in $\hat{\ttheta}_{N}$ and the sample averages, so the continuous mapping theorem \citep[Theorem~2.3]{vanderVaart1998} gives $\hat{\Sigma}_{N} \asto \Sigma_{0}$. Finally, $\Gamma_{N}^{-1} \asto \Gamma_{0}^{-1}$ by (b), and matrix multiplication is continuous, so $\Gamma_{N}^{-1}\hat{\Sigma}_{N}\Gamma_{N}^{-1} \asto \Gamma_{0}^{-1}\Sigma_{0}\Gamma_{0}^{-1}$.
    \end{enumerate}
\end{proof}

\subsection{Proof of \autoref{cor:exp-fam-true-parameter}} \label{appendix:identification_proof}

\begin{proof}
    Let $U$ be an open neighbourhood of $\ttheta^{*}$ on which the assumptions of \autoref{prop:main-objective-in-gsm-form} hold. The proof of \autoref{prop:main-objective-in-gsm-form} in Appendix~\ref{appendix:proof-prop-main-objective-in-gsm-form} shows that, for every $\ttheta \in U$,
    \begin{align*}
        \L(\ttheta) = \frac{1}{2}\E\Big[\norm{\nabla \log p_{\ttheta}(\x{x}) - \nabla \log p(\x{x})}^{2}_{D(\x{x})}\Big] - \frac{1}{2}\E\big[\nabla \log p(\x{x})^{\top}D(\x{x})\nabla \log p(\x{x})\big].
    \end{align*}
    The second term does not depend on $\ttheta$. The first term is nonnegative, and it vanishes at $\ttheta = \ttheta^{*}$ because $p = p_{\ttheta^{*}}$. Hence $\L(\ttheta) \geq \L(\ttheta^{*})$ for every $\ttheta \in U$, that is, $\ttheta^{*}$ minimises $\L$ over the open set $U$. The gradient of $\L$ therefore vanishes at $\ttheta^{*}$, and by \autoref{eq:exp-fam-population} this gradient is $\Gamma_{0}\ttheta^{*} + \x{g}_{0}$. So $\Gamma_{0}\ttheta^{*} = -\x{g}_{0}$, and by \ref{ass:exp-fam-identifiable}, $\ttheta^{*} = -\Gamma_{0}^{-1}\x{g}_{0} = \ttheta_{0}$, the unique minimiser of $\L$. The remaining claims follow from \autoref{theorem:consistency}.
\end{proof}
\subsection{Comments on \texorpdfstring{\autoref{theorem:consistency}}{the consistency theorem}} \label{appendix:exp-fam-comments}
    \begin{enumerate}[label=(\arabic*)]\label{remark:exp-fam-scope}
        \item \emph{Invertibility of $\Gamma_{N}$.} Positive semidefiniteness of $\Gamma_{N}$ is automatic (\autoref{theorem:convexity}(b)), but positive-definiteness is not: $\Gamma_{N}$ is singular whenever $N < r/d$ (\autoref{theorem:convexity}(b)), so it cannot be assumed for every $N$ when $r > d$. \autoref{theorem:consistency} therefore assumes $\Gamma_{0} \succ 0$ instead and derives $\Gamma_{N} \succ 0$ for large $N$. For the truncated Gaussian ($r = d(d+3)/2$) and the Dirichlet distribution in reduced coordinates ($r = d + 1$), $\Gamma_{N} \succ 0$ almost surely for every $N \geq d + 1$ and every $N \geq 2$, respectively, and \ref{ass:exp-fam-moments} implies \ref{ass:exp-fam-identifiable}.
        \item \emph{The plug-in sandwich.} On $\{\Gamma_{N} \succ 0\}$, $\hat{\Sigma}_{N}$ in \autoref{theorem:consistency}(e) is exactly centred: the vectors $\Gamma(\x{x}_{i})\hat{\ttheta}_{N} + \x{g}(\x{x}_{i})$ average to $\Gamma_{N}\hat{\ttheta}_{N} + \x{g}_{N} = \x{0}$, so $\hat{\Sigma}_{N}$ is their sample covariance (with divisor $N$). The sandwich $\Gamma_{N}^{-1}\hat{\Sigma}_{N}\Gamma_{N}^{-1}$ estimates the asymptotic covariance of $\sqrt{N}(\hat{\ttheta}_{N} - \ttheta_{0})$, that is, the spread of $\hat{\ttheta}_{N}$ about $\ttheta_{0}$. It measures the error about $\ttheta^{*}$ only when $\ttheta_{0} = \ttheta^{*}$, for example under Corollary~\ref{cor:exp-fam-true-parameter}.
        \item \emph{Relation to prior work.} The proof of \autoref{theorem:consistency} is the standard argument for a linear estimating equation. For $\Omega = \R^{d}$, $\hat{\ttheta}_{N}$ is a diffusion score matching estimator, whose consistency and asymptotic normality \citet{min_stein_disc_est} establish under their assumptions (see also \citet{generalized_score, gsm_xu}). What is specific to a constrained domain is the boundary condition of \autoref{prop:main-objective-in-gsm-form}, which \autoref{cor:exp-fam-true-parameter} needs to identify $\ttheta^{*}$.
    \end{enumerate}

\section{Experiments} \label{appendix:more-experiments}
In the subsequent sections, we provide additional experimental results. Specific implementation details are outlined in Appendix \ref{appendix:implementation-details}.

\begin{table*}[t]
    \centering
    \caption{Performance comparison of parameter estimators ($\x{\mu}$ and $K$) for the truncated Gaussian model on $\mathcal{S}^{10}$ at $N=800$ across 50 independent trials, where $\phi_{1}(\x{x}) = \sum_{i}x_{i}\log(x_{i}) + \left(1 - \sum_{i}x_{i}\right)\log\left(1 - \sum_{i}x_{i}\right)$ and $\phi_{2}(\x{x}) = -\sum_{i}\log(x_{i}) - \log\left(1 - \sum_{i}x_{i}\right)$. Baselines include Truncated SM \cite{truncated_sm} and GSM-GD \cite{gsm_general_domain} evaluated with $\x{h}_{1}(\x{x}) = \x{x}$, $\x{h}_{2}(\x{x}) = \x{x}^2$, and $\x{h}_{3}(\x{x}) = \sqrt{\x{x}}$. We report the Mean, Median, and Standard Deviation of the MSE.}
    \label{tab:trunc-gm-polytope-10d-results}
    \resizebox{\textwidth}{!}{
        \begin{sc}
            \begin{tabular}{@{}clcccccc@{}}
                \toprule
                & & \multicolumn{3}{c}{$\text{MSE}_{\x{\mu}}$} & \multicolumn{3}{c}{$\text{MSE}_{K}$} \\
                \cmidrule(lr){3-5} \cmidrule(lr){6-8}
                Methodology & $\phi(\x{x})$; $\x{h}(\x{x})$ & Mean & Median & Std. & Mean & Median & Std. \\
                \midrule
                OURS & $\phi_{1}$ & $\mathbf{0.0089}$ & $\mathbf{0.0045}$ & $\mathbf{0.0162}$ & $\mathbf{1.90 \times 10^4}$ & $\mathbf{1.65 \times 10^4}$ & $\mathbf{8.77 \times 10^3}$ \\
                OURS & $\phi_{2}$ & $3.2791$ & $0.0119$ & $15.3060$ & $2.73 \times 10^4$ & $2.31 \times 10^4$ & $1.22 \times 10^4$ \\
                Truncated SM \cite{truncated_sm} & --- & $170.8484$ & $0.1050$ & $1192.0491$ & $5.19 \times 10^4$ & $4.77 \times 10^4$ & $2.35 \times 10^4$ \\
                GSM-GD \cite{gsm_general_domain} & $\x{h}_{1}$ & $2.1755$ & $0.0070$ & $14.8702$ & $2.27 \times 10^4$ & $1.98 \times 10^4$ & $1.03 \times 10^4$ \\
                GSM-GD \cite{gsm_general_domain} & $\x{h}_{2}$ & $1.9612$ & $0.0720$ & $9.4792$ & $4.01 \times 10^4$ & $3.53 \times 10^4$ & $1.89 \times 10^4$ \\
                GSM-GD \cite{gsm_general_domain} & $\x{h}_{3}$ & $1.9876$ & $0.0285$ & $8.6418$ & $7.00 \times 10^4$ & $3.83 \times 10^4$ & $1.64 \times 10^5$ \\
                \bottomrule
            \end{tabular}
        \end{sc}
    }
\end{table*}

\begin{table}
    \centering
    \caption{Performance comparison of parameter estimators for the Dirichlet distribution concentration parameter $\x{\alpha}$ on the simplex $\Delta^9$ ($d=10$) at $N=800$ across 50 independent trials, where $\phi_{1}(\x{x}) = \sum_{i}x_{i}\log(x_{i}) + \left(1 - \sum_{i}x_{i}\right)\log\left(1 - \sum_{i}x_{i}\right)$ and $\phi_{2}(\x{x}) = -\sum_{i}\log(x_{i}) - \log\left(1 - \sum_{i}x_{i}\right)$. Baselines include Truncated SM \citep{truncated_sm} and GSM-CD \citep{gsm-on-compositional-data} (adapting GSM-GD \cite{gsm_general_domain} to simplex domains) evaluated with $\x{h}_{1}(\x{x}) = \x{x}$, $\x{h}_{2}(\x{x}) = \x{x}^2$, and $\x{h}_{3}(\x{x}) = \sqrt{\x{x}}$. We report the Mean, Median, and Standard Deviation of the MSE.}
    \label{tab:dirichlet-simplex-10d-results}
    \resizebox{0.95\textwidth}{!}{
        \begin{sc}
            \begin{tabular}{@{}clccc@{}}
                \toprule
                Methodology & $\phi(\x{x})$; $\x{h}(\x{x})$ & Mean MSE & Median MSE & Std. MSE \\
                \midrule
                OURS & $\phi_{1}$ & $0.2905$ & $0.2356$ & $0.2417$ \\
                OURS & $\phi_{2}$ & $\mathbf{0.2366}$ & $\mathbf{0.1527}$ & $\mathbf{0.2274}$ \\
                Truncated SM \cite{truncated_sm} & --- & $1.1192$ & $0.7314$ & $1.0901$ \\
                GSM-CD \cite{gsm-on-compositional-data} & $\x{h}_{1}(\x{x})$ & $0.3504$ & $0.2299$ & $0.3333$ \\
                GSM-CD \cite{gsm-on-compositional-data} & $\x{h}_{2}(\x{x})$ & $0.4295$ & $0.2366$ & $0.5366$ \\
                GSM-CD \cite{gsm-on-compositional-data} & $\x{h}_{3}(\x{x})$ & $5.7546$ & $5.7375$ & $2.1430$ \\
                \bottomrule
            \end{tabular}
        \end{sc}
    }
\end{table}

\subsection{Ablation Studies}
\subsubsection{Robustness Across Diverse Ground Truth Parameter Initialisations} \label{appendix:ablation-robustness}
The primary experiments performed for Truncated Gaussian model on $\mathcal{S}^{d}$ and Dirichlet model on $\Delta^{d-1}$ were for fixed choices of ground truth parameters. To verify that the performance advantages of the proposed estimators are not artifacts of a specific parameter initialisation, we systematically evaluate estimator robustness across $50$ distinct ground truth parameter configurations.

To aggregate performance across diverse parameter initialisations, we evaluate estimators using three complementary summary metrics: \emph{Win Rate}, \emph{Average Rank}, and the median MSE across configurations. For each ground-truth configuration $c \in \{1, \dots, C\}$, we compute the median MSE across 50 independent Monte Carlo trials for every candidate estimator. An estimator is assigned a win for configuration $c$ if it achieves the strictly lowest trial-median MSE, yielding a cumulative win rate $\frac{1}{C}\sum_{c=1}^C \mathbbm{1}_{\text{rank}_c = 1}$. Similarly, candidate methods are ranked from $1$ (best) to $M$ (worst) on each configuration based on trial-median MSE, from which we report the average rank $\frac{1}{C}\sum_{c=1}^C \text{rank}_c$. Evaluating average rank alongside win rate prevents misleading conclusions where an estimator wins narrowly in select regimes but suffers catastrophic numerical degradation in others.

\autoref{tab:trunc-gm-polytope-10d-ablation-results} and \autoref{tab:dirichlet-simplex-10d-ablation-results} report quantitative ablation metrics assessing estimator robustness across 50 diverse ground-truth parameter configurations ($N=800$, 50 trials per configuration) for the truncated Gaussian model on $\mathcal{S}^{10}$ and the Dirichlet model on $\Delta^{9}$, respectively. We observe the following:

\begin{enumerate}
    \item For the truncated Gaussian model on $\mathcal{S}^{10}$, our proposed barrier potential $\phi_{1}$ achieves the strongest performance across all 50 ground-truth regimes. On $\x{\mu}$, $\phi_{1}$ attains an $84\%$ win rate, an average rank of $1.22$, and a median MSE of $0.0081$, outperforming GSM-GD($\x{h}_{1}$) ($12\%$ win rate, average rank $2.12$, median MSE $0.0121$). For $K$, $\phi_{1}$ achieves a $100\%$ win rate across all 50 regimes, securing a rank of $1.00$ and a median MSE of $3.23 \times 10^4$. GSM-GD($\x{h}_{1}$) consistently follows as the runner-up with an average rank of $2.00$ and a median MSE of $3.90 \times 10^4$. Both estimators substantially outperform $\phi_{2}$, GSM-GD($\x{h}_{3}$), GSM-GD($\x{h}_{2}$), and Truncated SM \cite{truncated_sm}, where Truncated SM exhibits an average rank of $5.18$ on $\x{\mu}$ and $5.82$ on $K$ (median MSE of $7.95 \times 10^4$).
    \item For the Dirichlet model on $\Delta^9$, $\phi_{2}$ exhibits superior parameter recovery across diverse configurations, securing an $86\%$ win rate, an average rank of $1.14$, and the lowest median MSE of $0.0535$. $\phi_1$ follows with a $14\%$ win rate, an average rank of $2.80$, and a median MSE of $0.3617$. Among the baseline formulations, GSM-CD($\x{h}_{2}$) and GSM-CD($\x{h}_{1}$) yield average ranks of $3.08$ (median MSE $0.1258$) and $3.70$ (median MSE $1.2486$), respectively, while Truncated SM \cite{truncated_sm} ranks lower with an average rank of $4.52$ and a median MSE of $1.5181$. GSM-CD($\x{h}_{3}$) consistently performs worst, attaining an average rank of $5.76$ and a median MSE of $7.2267$.
\end{enumerate}

\begin{table}
    \centering
    \caption{Ablation study evaluating the robustness of parameter estimators for the truncated Gaussian model on the standard simplex polytope $\mathcal{S}^{10}$ across $50$ diverse ground-truth configurations ($N=800$, $50$ trials per configuration), where $\phi_{1}(\x{x}) = \sum_{i}x_{i}\log(x_{i}) + \left(1 - \sum_{i}x_{i}\right)\log\left(1 - \sum_{i}x_{i}\right)$ and $\phi_{2}(\x{x}) = -\sum_{i}\log(x_{i}) - \log\left(1 - \sum_{i}x_{i}\right)$. Baselines include Truncated SM \cite{truncated_sm} and GSM-GD \cite{gsm_general_domain} evaluated under $\x{h}_{1}(\x{x}) = \x{x}$, $\x{h}_{2}(\x{x}) = \x{x}^2$, and $\x{h}_{3}(\x{x}) = \sqrt{\x{x}}$. We report the Win Rate, Average Rank, and Median MSE across all configurations.}
    \label{tab:trunc-gm-polytope-10d-ablation-results}
    \resizebox{0.98\textwidth}{!}{
        \begin{sc}
            \begin{tabular}{@{}clcccccc@{}}
                \toprule
                & & \multicolumn{3}{c}{$\x{\mu}$} & \multicolumn{3}{c}{$K$} \\
                \cmidrule(lr){3-5} \cmidrule(lr){6-8}
                Methodology & $\phi(\x{x})$; $\x{h}(\x{x})$ & Win Rate & Avg. Rank & Median MSE & Win Rate & Avg. Rank & Median MSE \\
                \midrule
                OURS & $\phi_{1}$ & $\mathbf{0.84}$ & $\mathbf{1.22}$ & $\mathbf{0.0081}$ & $\mathbf{1.00}$ & $\mathbf{1.00}$ & $\mathbf{3.23 \times 10^4}$ \\
                OURS & $\phi_{2}$ & $0.00$ & $3.50$ & $0.0225$ & $0.00$ & $3.14$ & $5.14 \times 10^4$ \\
                Truncated SM \cite{truncated_sm} & --- & $0.02$ & $5.18$ & $0.0400$ & $0.00$ & $5.82$ & $7.95 \times 10^4$ \\
                GSM-GD \cite{gsm_general_domain} & $\x{h}_{1}$ & $0.12$ & $2.12$ & $0.0121$ & $0.00$ & $2.00$ & $3.90 \times 10^4$ \\
                GSM-GD \cite{gsm_general_domain} & $\x{h}_{2}$ & $0.00$ & $5.48$ & $0.0530$ & $0.00$ & $4.88$ & $7.06 \times 10^4$ \\
                GSM-GD \cite{gsm_general_domain} & $\x{h}_{3}$ & $0.02$ & $3.50$ & $0.0224$ & $0.00$ & $4.16$ & $6.85 \times 10^4$ \\
                \bottomrule
            \end{tabular}
        \end{sc}
    }
\end{table}

\begin{table}
    \centering
    \caption{Ablation study evaluating the robustness of parameter estimators for the Dirichlet distribution concentration parameter $\x{\alpha}$ on the simplex $\Delta^9$ ($d=10$) across $50$ diverse ground-truth configurations ($N=800$, $50$ trials per configuration), where $\phi_{1}(\x{x}) = \sum_{i}x_{i}\log(x_{i}) + \left(1 - \sum_{i}x_{i}\right)\log\left(1 - \sum_{i}x_{i}\right)$ and $\phi_{2}(\x{x}) = -\sum_{i}\log(x_{i}) - \log\left(1 - \sum_{i}x_{i}\right)$. Baselines include Truncated SM \cite{truncated_sm} and GSM-CD \cite{gsm-on-compositional-data} evaluated with $\x{h}_{1}(\x{x}) = \x{x}$, $\x{h}_{2}(\x{x}) = \x{x}^2$, and $\x{h}_{3}(\x{x}) = \sqrt{\x{x}}$. We report the Win Rate, Average Rank, and Median MSE across all configurations.}
    \label{tab:dirichlet-simplex-10d-ablation-results}
    \resizebox{0.95\textwidth}{!}{
        \begin{sc}
            \begin{tabular}{@{}clccc@{}}
                \toprule
                Methodology & $\phi(\x{x})$; $\x{h}(\x{x})$ & Win Rate & Avg. Rank & Median MSE \\
                \midrule
                OURS & $\phi_{1}$ & $0.14$ & $2.80$ & $0.3617$ \\
                OURS & $\phi_{2}$ & $\mathbf{0.86}$ & $\mathbf{1.14}$ & $\mathbf{0.0535}$ \\
                Truncated SM \cite{truncated_sm} & --- & $0.00$ & $4.52$ & $1.5181$ \\
                GSM-CD \cite{gsm-on-compositional-data} & $\x{h}_{1}$ & $0.00$ & $3.70$ & $1.2486$ \\
                GSM-CD \cite{gsm-on-compositional-data} & $\x{h}_{2}$ & $0.00$ & $3.08$ & $0.1258$ \\
                GSM-CD \cite{gsm-on-compositional-data} & $\x{h}_{3}$ & $0.00$ & $5.76$ & $7.2267$ \\
                \bottomrule
            \end{tabular}
        \end{sc}
    }
\end{table}

\subsubsection{Effect of Power Barrier Exponent} \label{appendix:ablation-subsec-power-barrier}
As noted in Section \ref{sec:experiments}, for a general convex polytope $\Omega = \{\x{x} \in \R^{d} : \x{a}_{k}^{\top}\x{x} < b_{k}, 1 \leq k \leq m\}$, we considered entropic barrier $\phi_{1}(\x{x}) = \sum_{k=1}^{m} s_{k}\log (s_{k})$ and log barrier $\phi_{2}(\x{x}) = -\sum_{k=1}^{m}\log(s_{k})$ where $s_{k} = b_{k} - \x{a}_{k}^{\top}\x{x}$ for our main experiments. The Hessians of $\phi_{1},\phi_{2}$ are given by
\begin{align*}
    H_{\phi_{1}}(\x{x}) &= \sum_{k=1}^{m}\frac{1}{s_{k}}\x{a}_{k}\x{a}_{k}^{\top} \quad \& \quad H_{\phi_{2}}(\x{x}) = \sum_{k=1}^{m}\frac{1}{s_{k}^{2}}\x{a}_{k}\x{a}_{k}^{\top}
\end{align*}
Observe that they have similar form but the power of $s_{k}$ is different. Motivated by this, we consider power barrier $\phi_{p}$ defined as
\begin{align*}
    \phi(\x{x}; p) = \frac{1}{p(p-1)}\sum_{k=1}^{m}s_{k}^{p}
\end{align*}
where $p \in (0, 2) \setminus \{1\}$. The corresponding Hessian is
\begin{align*}
    H_{\phi(\cdot; p)}(\x{x}) &= \sum_{k=1}^{m}s_{k}^{p-2}\x{a}_{k}\x{a}_{k}^{\top}
\end{align*}
Because $D_{\phi}$ scales inversely with the Hessian, enforcing boundary attenuation requires elements of the Hessian to diverge near boundary, which implies $p < 2$. We evaluate this sweep on the $10$ dimensional truncated Gaussian model on $\mathcal{S}^{10}$ at sample size $N = 800$, across exponents $p \in \{0.2, 0.5, 0.8, 1.2, 4.0/3.0, 1.4, 1.5, 1.6, 1.8\}$ over 50 independent trials. \autoref{fig:trunc-gauss-polytope-10d-power-barrier-exponent-ablation} summarises the results, reporting the median MSE alongside the inter-quartile range (IQR).

As shown in \autoref{fig:trunc-gauss-polytope-10d-power-barrier-exponent-ablation}, the choice of power barrier exponent $p \in (0, 2) \setminus \{1\}$ markedly influences estimation performance across both parameters. For smaller exponents ($p \in \{0.2, 0.5\}$), the estimator exhibits elevated error, yielding median MSEs of $0.0140$ and $0.0078$ for $\x{\mu}$, and $2.01 \times 10^4$ and $1.71 \times 10^4$ for $K$, respectively. As $p$ increases towards the intermediate regime, performance steadily improves, reaching a minimum near $p = 1.2$, where the power barrier attains the lowest median MSE across all evaluated exponents on both parameters ($0.0048$ on $\x{\mu}$ and $1.39 \times 10^4$ on $K$, with interquartile ranges of $0.0068$ and $7.70 \times 10^3$, respectively). In this optimal band, the power barrier outperforms all baselines, including GSM-GD($\x{h}_{1}$) (median MSE $0.0086$ on $\x{\mu}$ and $1.64 \times 10^4$ on $K$), Truncated SM ($0.0815$ on $\x{\mu}$ and $4.11 \times 10^4$ on $K$), GSM-GD($\x{h}_{3}$) ($0.0455$ on $\x{\mu}$ and $3.53 \times 10^4$ on $K$), and GSM-GD($\x{h}_{2}$) ($0.1046$ on $\x{\mu}$ and $3.07 \times 10^4$ on $K$).

However, as $p$ approaches $2$, estimation accuracy degrades sharply on both parameters: at $p = 1.6$, the median MSE for $\x{\mu}$ increases to $0.0335$ accompanied by a noticeable widening of the IQR ($0.1242$), while at $p = 1.8$, precision matrix estimation suffers substantial degradation, with median MSE escalating to $1.68 \times 10^5$ (IQR: $1.88 \times 10^5$).

\begin{figure}
    \centering
    \includegraphics[width=\linewidth]{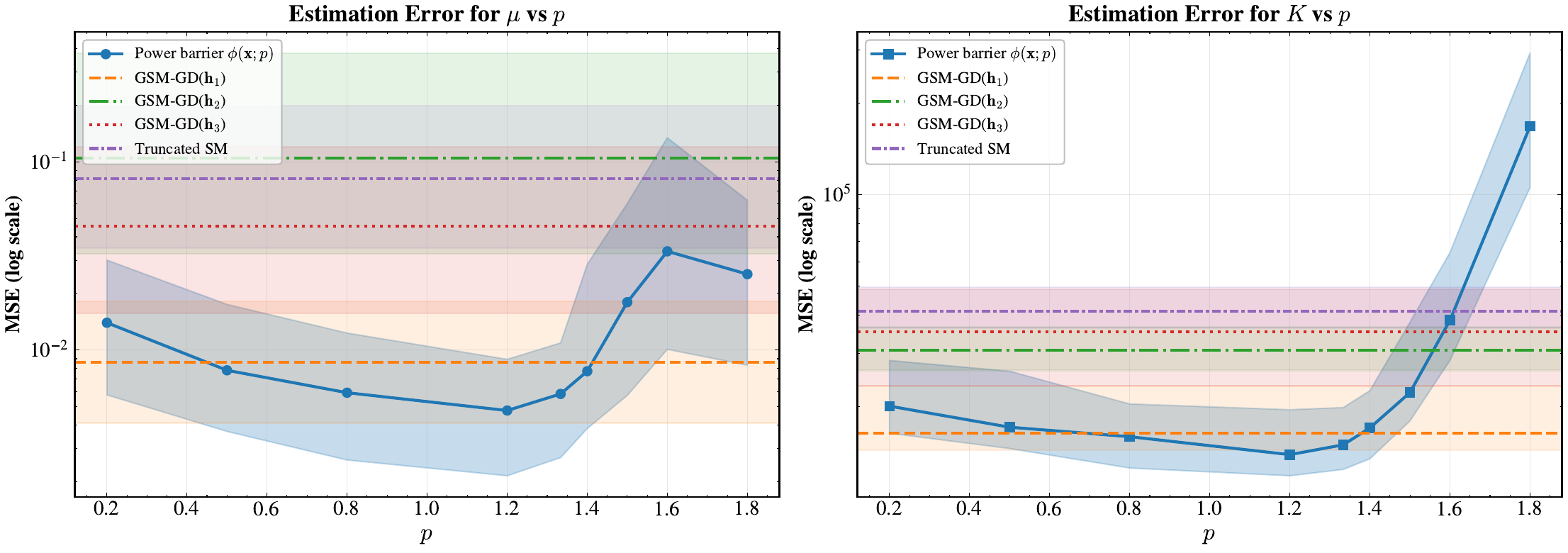}
    \caption{Estimation error for $\x{\mu}$ (left) and $K$ (right) as a function of the power barrier exponent $p$ on $\mathcal{S}^{10}$ ($d=10, N=800$). Solid curves and shaded regions display the median MSE and inter-quartile range (IQR) across 50 trials, respectively.}
    \label{fig:trunc-gauss-polytope-10d-power-barrier-exponent-ablation}
\end{figure}

\subsubsection{Evaluation Under Non-Convex Support} \label{appendix:ablation-subsec-non-convex-support}
In practical applications, observed data may be supported on non-convex domains. Although our theoretical framework assumes a convex support, we investigate the empirical behaviour of our estimator when applied to non-convex geometries via a convex relaxation. Specifically, we consider the non-convex domain
\begin{align*}
    \Omega = \{\x{x} \in \mathbb{R}^{d} : \norm{\x{x}} < a\} \cup \{\x{x} \in \mathbb{R}^{d} : b < \norm{\x{x}} < c\}, \quad 0 < a < b < c.
\end{align*}
Because our methodology requires a convex domain, a natural heuristic is to evaluate the estimator on the convex hull of $\Omega$, which corresponds to the open ball
\begin{align*}
    \Omega' = \{\x{x} \in \mathbb{R}^{d} : \norm{\x{x}} < c\}.
\end{align*}
We construct a logarithmic barrier potential directly on $\Omega'$:
\begin{align*}
    \phi(\x{x}) = -\log\left(c^{2} - \norm{\x{x}}^{2}\right).
\end{align*}
We generate samples from a standard Gaussian density ($\x{\mu}_0 = \x{0}, K_0 = I_d$) truncated to $\Omega$ with radii $a=0.2$, $b=0.5$, and $c=1.0$ in dimension $d=3$. We compare our convex-hull estimator against Truncated Score Matching \citep{truncated_sm}, which accounts for domain boundaries directly by weighting the objective.

\begin{figure}[ht]
    \centering
    \includegraphics[width=\linewidth]{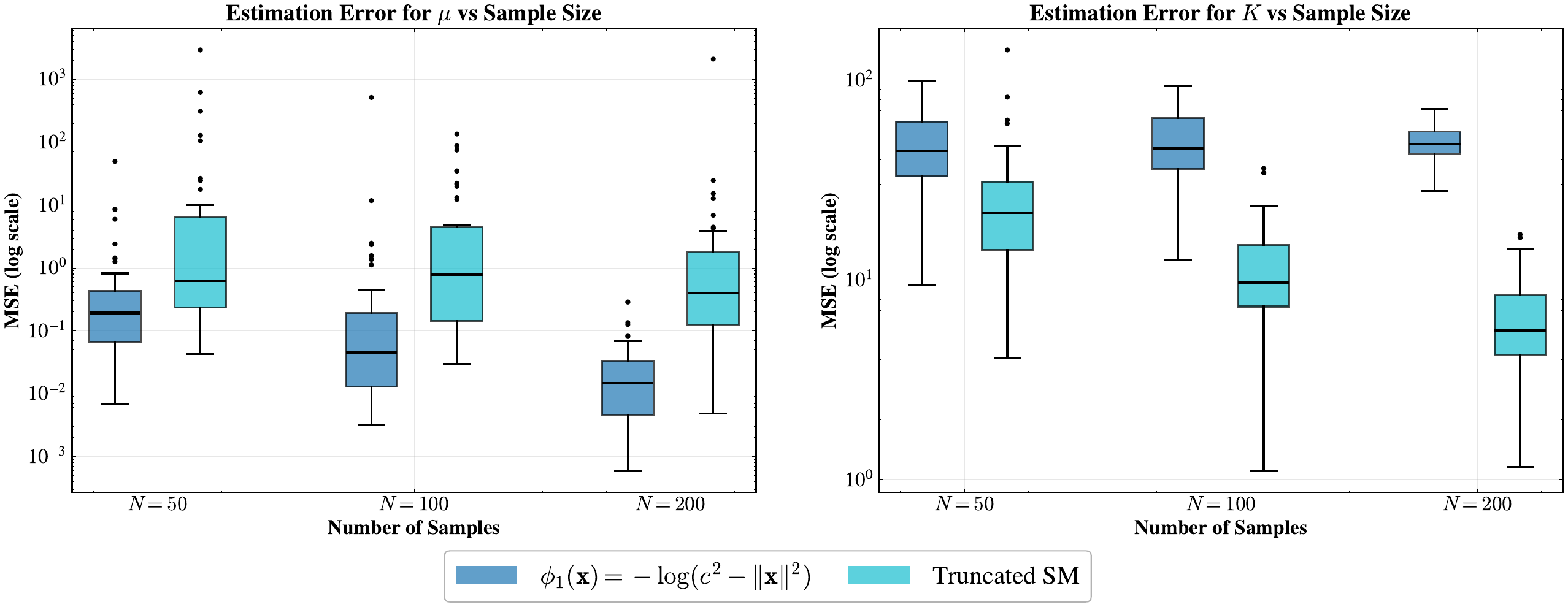}
    \caption{Parameter estimation error for $\x{\mu}$ (left) and $K$ (right) on the non-convex domain $\Omega$ with $a = 0.2$, $b = 0.5$, and $c = 1.0$ across sample sizes $N \in \{50, 100, 200\}$ over 50 independent trials.}
    \label{fig:trunc-gauss-non-convex-annulus-3d-ablation}
\end{figure}
Because the barrier $\phi$ is constructed with respect to the relaxed domain $\Omega'$, its generator $D_{\phi}(\x{x})$ does not vanish on the internal boundaries ($\norm{\x{x}} = a$ and $\norm{\x{x}} = b$), causing standard boundary conditions to fail. Consequently, we expect estimation quality to degrade. As shown in \autoref{fig:trunc-gauss-non-convex-annulus-3d-ablation}, estimation for $K$ degrades substantially, with the MSE plateauing around $150$ across all sample sizes. Surprisingly, estimation for $\x{\mu}$ remains accurate: the median MSE decreases from $0.0959$ at $N=50$ to $0.0079$ at $N=200$, outperforming Truncated Score Matching. To explain this asymmetric performance, we examine the integration by parts step that connects the tractable score matching loss to the weighted Fisher divergence:
\begin{align*}
    \mathcal{L}_{GSM}(\x{\theta}) &= \frac{1}{2}\E_{\x{x} \sim p}\left[\norm{\nabla \log p_{\x{\theta}}(\x{x}) - \nabla \log p(\x{x})}^{2}_{D_{\phi}(\x{x})}\right] \\
    &= \frac{1}{2}\E_{\x{x} \sim p}\left[\nabla \log p_{\x{\theta}}(\x{x})^{\top}D_{\phi}(\x{x})\nabla \log p_{\x{\theta}}(\x{x})\right] - \int_{\x{x} \in \Omega} \nabla p(\x{x})^{\top} D_{\phi}(\x{x}) \nabla \log p_{\x{\theta}}(\x{x}) \d \x{x} \\
    &\quad + C \\
    &= \frac{1}{2}\E_{\x{x} \sim p}\left[\nabla \log p_{\x{\theta}}(\x{x})^{\top}D_{\phi}(\x{x})\nabla \log p_{\x{\theta}}(\x{x})\right] + \E_{\x{x} \sim p}\left[\nabla \cdot \left(D_{\phi}(\x{x})\nabla \log p_{\x{\theta}}(\x{x})\right)\right] \\
    &\quad - \mathcal{B}(\x{\theta}) + C,
\end{align*}
where $C$ is a parameter-independent constant and $\mathcal{B}(\x{\theta})$ denotes the boundary integral over $\partial \Omega$:
\begin{align*}
    \mathcal{B}(\x{\theta}) = \int_{\partial \Omega} p(\x{x}) \, \x{n}(\x{x})^{\top}D_{\phi}(\x{x})\nabla \log p_{\x{\theta}}(\x{x}) \, \mathrm{d}s.
\end{align*}
Here $\partial \Omega$ comprises three concentric spheres: $\norm{\x{x}} = a$, $\norm{\x{x}} = b$, and $\norm{\x{x}} = c$. On the outer boundary $\norm{\x{x}} = c$, $D_{\phi}(\x{x}) \to \mathbf{0}$ by construction, so its contribution vanishes. The remaining boundary integral evaluates over the internal surfaces:
\begin{align*}
    \mathcal{B}(\x{\theta}) = \int_{\norm{\x{x}} = a} p(\x{x})\left(\frac{\x{x}}{a}\right)^{\top}D_{\phi}(\x{x})\nabla \log p_{\x{\theta}}(\x{x}) \, \mathrm{d}s - \int_{\norm{\x{x}} = b} p(\x{x})\left(\frac{\x{x}}{b}\right)^{\top}D_{\phi}(\x{x})\nabla \log p_{\x{\theta}}(\x{x}) \, \mathrm{d}s,
\end{align*}
where the sign inversion on $\norm{\x{x}} = b$ reflects the inward-pointing normal relative to the outer annulus. For a standard Gaussian centered at the origin, the density $p(\x{x}) \propto \exp\left(-\frac{1}{2}\norm{\x{x}}^{2}\right)$ is constant on any sphere $\norm{\x{x}} = r$. Furthermore, $\x{x}^{\top}D_{\phi}(\x{x}) = \gamma(r)\x{x}^{\top}$ for a scalar constant $\gamma(r)$. Substituting the score $\nabla \log p_{\x{\theta}}(\x{x}) = -K(\x{x} - \x{\mu})$, the integrand on each spherical shell becomes proportional to
\begin{align*}
    \x{x}^{\top}K(\x{x} - \x{\mu}) = \x{x}^{\top}K\x{x} - \x{x}^{\top}K\x{\mu}.
\end{align*}
By symmetry ($\x{x} \mapsto -\x{x}$), the linear term integrates to zero:
\begin{align*}
    \int_{\norm{\x{x}} = r} \x{x}^{\top}K\x{\mu} \, \mathrm{d}s = 0.
\end{align*}
In contrast, the quadratic term $\x{x}^{\top}K\x{x}$ is strictly positive and non-zero:
\begin{align*}
    \int_{\norm{\x{x}} = r} \x{x}^{\top}K\x{x} \, \mathrm{d}s = \frac{r^{2}\mathrm{Tr}(K)}{d} \mathrm{Area}(\mathbb{S}_{r}^{d-1}) \neq 0.
\end{align*}
Consequently, the boundary term simplifies to
\begin{align*}
    \mathcal{B}(\x{\theta}) = - C_K \mathrm{Tr}(K),
\end{align*}
for a positive constant $C_K$ independent of $\x{\mu}$. This yields an important insight --- the boundary discrepancy does not depend on $\x{\mu}$, meaning $\nabla_{\x{\mu}}\mathcal{B}(\x{\theta}) = \mathbf{0}$. Therefore, the gradient of our sample objective with respect to $\x{\mu}$ remains an asymptotically unbiased estimator of the true Fisher divergence gradient, allowing $\hat{\x{\mu}}_N \to \x{\mu}_0$ as $N \to \infty$. Conversely, $\nabla_K \mathcal{B}(\x{\theta}) \neq \mathbf{0}$, which introduces an irreducible bias into the loss for $K$. This explains why $K$ exhibits non-convergent error, while $\x{\mu}$ converges stably.

In comparison, Truncated Score Matching \citep{truncated_sm} uses the exact boundary distance weighting to ensure boundary vanishing across all bounding surfaces, which preserves theoretical consistency for both $\x{\mu}$ and $K$. While this consistency is reflected in the steady decrease of its precision error with sample size, Truncated Score Matching exhibits noticeably higher estimation error and wider spread for $\x{\mu}$ at small sample sizes compared to our approach. We hypothesise that this could be due to finite-sample gradient fluctuations arising from the non-smooth distance function; however, a thorough investigation of this behaviour is beyond the scope of this work.

\begin{remark} \label{remark:boundary-conditions-sufficient-and-decoupling-mu}
    This decoupling between $\x{\mu}$ and $K$ relies on the joint spherical symmetry of the domain cuts and the centered density ($\x{\mu}_0 = \mathbf{0}$). If the true mean were translated away from the origin, $p(\x{x})$ would vary across the internal boundaries, breaking the cancellation and coupling $\x{\mu}$ to the non-zero boundary integral. Nonetheless, this experiment highlights that boundary attenuation is a sufficient, rather than strictly necessary, condition for recovering subset parameters in constrained settings.
\end{remark}

\subsubsection{Evaluation Under Model Neumann Boundary Conditions} \label{subsec:neumann_ablation}
The boundary attenuation condition on the weighting matrix ($D_\phi(\x{x}) \to \mathbf{0}$ as $\x{x} \to \partial \Omega$) ensures that the boundary surface integrals vanish during integration by parts (\autoref{prop:main-objective-in-gsm-form}). However, boundary attenuation is a sufficient rather than necessary condition (see \autoref{remark:boundary-conditions-sufficient-and-decoupling-mu}). When the model family itself satisfies Neumann boundary conditions \citep{neumann_boundary}, the outward normal component of the model score vanishes identically along the entire boundary:
\begin{align*}
    \x{n}(\x{x})^\top \nabla \log p_\theta(\x{x}) = 0 \quad \text{for all } \x{x} \in \partial \Omega.
\end{align*}
In this case, the boundary term $\int_{\partial \Omega} p(\x{x}) \x{n}(\x{x})^\top D_{\phi}(\x{x}) \nabla \log p_\theta(\x{x}) \d s$ evaluates to zero for $D_{\phi}(\x{x}) = \id$. Consequently, original score matching \citep{Hyvarinen05a} remains theoretically valid on bounded domains without requiring boundary damping. To assess the behaviour of our framework when boundary attenuation is not required, we investigate model families designed to satisfy Neumann boundary conditions.

We first examine a one-dimensional setting to illustrate the setup. Let $\Omega = (0, 1)$ with model score $s_\theta(x) = -\theta \pi \sin(\pi x)$. Because $s_\theta(0) = s_\theta(1) = 0$, the Neumann condition holds at both endpoints. Integrating this score yields the model density
\begin{align*}
    p_\theta(x) = \frac{1}{\mathcal{Z}(\theta)} \exp\left(\theta \cos(\pi x)\right), \quad \mathcal{Z}(\theta) = \int_0^1 \exp\left(\theta \cos(\pi x)\right) \mathrm{d}x.
\end{align*}
Given i.i.d.\ samples $x_1, \dots, x_N \sim p_{\theta_0}$ with ground truth $\theta_0 = 2.0$, we compare Maximum Likelihood Estimation (MLE, solved numerically), original score matching (SM), and Generalised Score Matching (GSM) instantiated with two choices: $\phi_{1}(x) = x \log x + (1 - x) \log(1 - x)$ and $\phi_{2}(x) = -\log x - \log(1 - x)$.

We next extend the setup to two dimensions on the unit square $\Omega = (0, 1)^2$ with the exponential family density
\begin{align*}
    \log p_{\x{\theta}}(\x{x}) = \theta_{1} \cos(\pi x_{1}) + \theta_{2} \cos(\pi x_{2}) + \theta_{3} \cos(\pi x_{1}) \cos(\pi x_{2}) - \log \mathcal{Z}(\x{\theta})
\end{align*}
where the ground truth parameter is $\x{\theta}_0 = [1.2, -0.8, 1.0]^\top$. This density satisfies $\x{n}(\x{x})^\top \nabla \log p_{\x{\theta}}(\x{x}) = 0$ along all four boundary facets of $\partial \Omega$. We solve the MLE via gradient ascent, while SM and GSM admit closed-form estimators. For GSM, we use $\phi_{1}(\x{x}) = \sum_{j=1}^2 (x_{j} \log x_{j} + (1 - x_{j})\log(1 - x_{j}))$ and $\phi_{2}(\x{x}) = -\sum_{j=1}^2 (\log x_{j} + \log(1 - x_{j}))$.

\begin{table*}
    \caption{Parameter estimation Mean MSE across 50 independent trials for model densities satisfying Neumann boundary conditions in $d=1$ $(\theta_0 = 2.0)$ and $d=2$ $\left(\x{\theta}_0 = [1.2, -0.8, 1.0]^\top\right)$.}
    \label{tab:neumann-boundary-ablation}
    \begin{center}
        \resizebox{0.95\textwidth}{!}{
            \begin{sc}
                \begin{tabular}{@{}lcccc|cccc@{}}
                    \toprule
                    & \multicolumn{4}{c|}{$d = 1$} & \multicolumn{4}{c}{$d = 2$} \\
                    \cmidrule(lr){2-5} \cmidrule(lr){6-9}
                    $N$ & MLE & SM & $\phi_{1}$ & $\phi_{2}$ & MLE & SM & $\phi_{1}$ & $\phi_{2}$ \\
                    \midrule
                    $200$ & $0.0282$ & $0.0286$ & $0.0300$ & $0.0416$ & $0.0874$ & $0.0977$ & $0.1076$ & $0.1359$ \\
                    $500$ & $0.0121$ & $0.0132$ & $0.0139$ & $0.0218$ & $0.0324$ & $0.0358$ & $0.0444$ & $0.0623$ \\
                    $800$ & $0.0073$ & $0.0108$ & $0.0113$ & $0.0144$ & $0.0205$ & $0.0235$ & $0.0232$ & $0.0307$ \\
                    \bottomrule
                \end{tabular}
            \end{sc}
        }
    \end{center}
\end{table*}

\autoref{tab:neumann-boundary-ablation} summarises the results. In both dimensions, original score matching achieves lower estimation error than GSM in all settings except $d = 2$, $N = 800$, where $\phi_{1}$ is marginally better. Furthermore, within GSM, the entropic barrier consistently outperforms the logarithmic barrier. We hypothesise that this ordering reflects the severity of boundary sample attenuation. When the score function already satisfies Neumann boundary conditions, samples near the boundary carry valid parameter information and require no attenuation. Original score matching preserves all boundary samples ($D = \id$), the entropic barrier ($\phi_{1}$) attenuates them mildly, and the logarithmic barrier ($\phi_{2}$) suppresses them aggressively. Consequently, stronger attenuation discards more informative boundary data. While this ordering aligns with intuition, all evaluated estimators remain consistent, with the error across all methods monotonically decreasing as the sample size increases.

\subsection{Data Driven Ranking of Candidate Choices of \texorpdfstring{$\phi$}{phi}} \label{appendix:subsec-data-driven-ranking}
In a practical setting, a practitioner must select a suitable function $\phi$ without access to the ground-truth parameter. To address this, we propose a diagnostic procedure to rank candidate choices of $\phi$ purely based on the observed data.  A natural starting point is \autoref{theorem:consistency}(c, e), which establishes that the asymptotic variance of our estimator is given by $\Gamma_{0}^{-1}\Sigma_{0}\Gamma_{0}^{-1}$ and sample version $\Gamma_{N}^{-1}\hat{\Sigma}_{N}\Gamma_{N}^{-1}$ converges to $\Gamma_{0}^{-1}\Sigma_{0}\Gamma_{0}^{-1}$ almost surely.

Let $\hat{\Upsilon} = \Gamma_{N}^{-1}\hat{\Sigma}_{N}\Gamma_{N}^{-1}$ denote the sample estimated asymptotic covariance matrix. We propose scoring candidate $\phi$ using the trace of this matrix, $\text{Tr}(\hat{\Upsilon}$, which we refer to as the \textit{estimated total variance}. A lower estimated total variance suggests a more statistically efficient choice of $\phi$. As an alternative heuristic, we consider the geometric curvature of the objective characterised by the condition number of its Hessian. Since the empirical loss is quadratic in $\x{\theta}$ (in case of exponential family, cf. \autoref{eq-convex-gsm-loss-exp}), its Hessian is exactly $\Gamma_N$. We propose ranking the candidates using the condition number of this Hessian, denoted as $\kappa(\Gamma_N)$. A lower condition number implies a better conditioned optimisation landscape. To evaluate these heuristics, we ran ranking experiments on our primary settings: the truncated Gaussian model on $\mathcal{S}^{10}$, and the Dirichlet model on the simplex $\Delta^9$. We fixed the sample size to $N=800$ for these diagnostic tests; because our proposed methodology evaluates gradients at the empirical estimate $\hat{\theta}_N$, ranking criteria computed on very small sample sizes may be excessively noisy and unreliable.

To quantify the quality of our proposed scores, we compare the data driven rankings against the true rankings of $\phi$, which are determined by the actual parameter Mean Squared Error (MSE) computed using the ground truth. We report Top-1 Accuracy (the proportion of trials where the heuristic successfully identifies the candidate with the lowest MSE) alongside Worst Regret and Mean Regret to measure the relative error penalty incurred when a suboptimal candidate is selected. We also report both Spearman ($\rho$) and Kendall ($\tau$) rank correlation coefficients between the predicted and true ranks. For the truncated Gaussian model, we separate the ranking evaluation for $\x{\mu}$ and $K$. To assess the effectiveness of the proposed diagnostic heuristics, we evaluate the candidate $\phi$ criterion over a curated family comprising the entropic and logarithmic barriers ($\phi_{1}, \phi_{2}$), alongside power barrier $\phi(\cdot; p)$ (Appendix \ref{appendix:ablation-subsec-power-barrier}) evaluated at exponents $p \in \{0.2, 0.8, 1.2, 1.8\}$.

\begin{table*}[ht]
    \caption{Evaluation of data-driven candidate ranking heuristics across 50 independent trials at $N=800$ over the candidate set $\{\phi_{1}, \phi_{2}, \phi(\cdot; p=0.2), \phi(\cdot; p=0.8), \phi(\cdot; p=1.2), \phi(\cdot; p=1.8)\}$. Top-1 Accuracy measures how often the heuristic selects the MSE-optimal potential. Worst and Mean Regret quantify relative sub-optimality under misselection. Rank correlations ($\rho$ and $\tau$) evaluate monotonic agreement with ground-truth error ordering.}
    \label{tab:candidate-ranking-diagnostics}
    \begin{center}
        \resizebox{\textwidth}{!}{
            \begin{sc}
                \begin{tabular}{@{}lllccccc@{}}
                    \toprule
                    Domain & Parameter & Score & Top-1 Acc. & Worst Regret & Mean Regret & $\rho$ & $\tau$ \\
                    \midrule
                    $\mathcal{S}^{10}$ & $K$ & $\text{Tr}(\hat{\Upsilon})$ & $0.540$ & $0.4769$ & $0.0552$ & $0.8823$ & $0.8347$ \\
                    & & $\kappa(\Gamma_N)$ & $0.120$ & $1.4764$ & $0.2745$ & $0.2937$ & $0.2293$ \\
                    \cmidrule(lr){2-8}
                    & $\x{\mu}$ & $\text{Tr}(\hat{\Upsilon})$ & $0.520$ & $4.7510$ & $0.5831$ & $0.6571$ & $0.5973$ \\
                    & & $\kappa(\Gamma_N)$ & $0.120$ & $22.0204$ & $2.4052$ & $0.0697$ & $0.0613$ \\
                    \midrule
                    $\Delta^{9}$ & $\x{\alpha}$ & $\text{Tr}(\hat{\Upsilon})$ & $0.300$ & $6.0111$ & $0.7101$ & $0.6263$ & $0.5280$ \\
                    & & $\kappa(\Gamma_N)$ & $0.200$ & $6.4300$ & $0.7093$ & $0.7509$ & $0.6507$ \\
                    \bottomrule
                \end{tabular}
            \end{sc}
        }
    \end{center}
\end{table*}

\autoref{tab:candidate-ranking-diagnostics} summarises the diagnostic efficacy of the finite-sample total variance estimator $\text{Tr}(\hat{\Upsilon})$ and the empirical Hessian condition number $\kappa(\Gamma_N)$ across 50 independent trials at $N = 800$. On the simplex polytope $\mathcal{S}^{10}$, the asymptotic covariance trace $\text{Tr}(\hat{\Upsilon})$ consistently provides candidate selection and rank ordering for both model parameters. For precision matrix estimation ($K$), $\text{Tr}(\hat{\Upsilon})$ achieves a Top-1 accuracy of $54\%$ with a minimal mean regret of $0.0552$ (median regret is $0.0$, IQR: $0.0668$), alongside monotonic correlation with the true estimation error ($\rho = 0.8823$, $\tau = 0.8347$). For the location parameter $\x{\mu}$, $\text{Tr}(\hat{\Upsilon})$ similarly secures a $52\%$ Top-1 accuracy and a median regret of $0.0$, retaining positive rank correlation ($\rho = 0.6571$, $\tau = 0.5973$). In contrast, the condition number heuristic $\kappa(\Gamma_N)$ performs poorly on $\mathcal{S}^{10}$, yielding only $12\%$ Top-1 accuracy on both parameters and failing to track error rankings ($\rho = 0.2937$ on $K$, and near-zero correlation $\rho = 0.0697$ on $\x{\mu}$), while incurring substantially higher mean regret ($2.4052$ on $\x{\mu}$) and extreme worst-case penalties ($22.0204$).

On the simplex domain $\Delta^9$ for the Dirichlet model, both diagnostics attain modest Top-1 selection rates ($30\%$ for $\text{Tr}(\hat{\Upsilon})$ and $20\%$ for $\kappa(\Gamma_N)$), and the selected candidate has a median relative regret of $0.31$ and $0.22$, respectively. The mean relative regrets ($0.71$ for both) are driven by a minority of poor selections: for both scores, the regret exceeds $1$ in $24\%$ of the trials. Both diagnostics still agree with the ordering by true error ($\rho = 0.6263$ for $\text{Tr}(\hat{\Upsilon})$ and $\rho = 0.7509$ for $\kappa(\Gamma_N)$). Part of this agreement comes from the bottom of the ranking: both scores rank $\phi(\cdot; p=1.8)$ last in every trial, and it is the least accurate candidate in every trial, with a median error about $2400$ times that of the best. On $\Delta^{9}$, the diagnostics therefore reliably rule out the unstable candidate near the upper power boundary.

\begin{remark}
    It should be noted that this diagnostic approach becomes computationally expensive in very high-dimensional settings, as computing the estimated total variance $\text{Tr}(\hat{\Upsilon})$ requires computing the inverse of  $\Gamma_N$.
\end{remark}

\subsection{Implicit Variational Autoencoders with Score Matching} \label{appendix:subsec-implicit-vae}
To demonstrate that the proposed framework is applicable to generative modeling, beyond parameter estimation problems, we consider implicit variational autoencoders (VAEs)~\citep{song2019sliced}, where score matching naturally arises in the training procedure. We leverage the generalised score matching objective, in \autoref{theorem:sm-convex}, to train the implicit VAE and demonstrate how the proposed objective can be scaled to high-dimensional settings. Consider the objective in \autoref{eq:main-objective-in-compact-form}:
\begin{align*}
    \L(\x{\theta}) &= \underset{\x{x} \sim p}{\E}\left[\frac{1}{2}\nabla \log p_{\x{\theta}}(\x{x})^{\top}D(\x{x})\nabla \log p_{\x{\theta}}(\x{x}) + \tr\left(D(\x{x})H_{\log p_{\theta}}(\x{x})\right) + \nabla \log p_{\x{\theta}}(\x{x})^{\top}(\nabla \cdot D)(\x{x})\right]
\end{align*}
While this formulation circumvents the computation of the partition function $\Z(\x{\theta})$, it requires computing the trace of $D(\x{x})H_{\log p_{\theta}}(\x{x})$. A practical alternative is to consider an unbiased estimator of the trace term typically achieved using Hutchinson-style trace estimator~\citep{Hutchinson01011990, Epperlyetal2024, fdsm, Shetty-mcsm}. The key idea is to estimate the trace by considering projections onto random directions and computing the corresponding quadratic form. In particular, we consider
\begin{align} \label{eq:gsm-obj-in-hutchison}
    \hat{\L}(\x{\theta}) &= \underset{\substack{\x{x} \sim p, \\ \x{v} \sim p_{\x{v}}}}{\E}\Bigg[\frac{1}{2}\nabla \log p_{\x{\theta}}(\x{x})^{\top}D(\x{x})\nabla \log p_{\x{\theta}}(\x{x}) + \x{v}^{\top}D(\x{x})H_{\log p_{\theta}}(\x{x})\x{v}\nonumber \\
    &\qquad\qquad\qquad\qquad\qquad\qquad\qquad\qquad\qquad\qquad  + \nabla \log p_{\x{\theta}}(\x{x})^{\top}(\nabla \cdot D)(\x{x})\Bigg]
\end{align}
where $p_{\x{v}}$ denotes the density of the random direction vectors $\x{v}$, chosen independently of $\x{x}$ and satisfying $\mathbb{E}[\x{v}\x{v}^{\top}] = \mathbb{I}$. If we choose $D(\x{x}) = \mathbb{I}$ in \autoref{eq:gsm-obj-in-hutchison}, then we recover the original score matching objective, with the trace term replaced by corresponding Hutchinson estimator. This resulting objective coincides with the variance-reduced version of sliced score matching (SSM-VR) considered by \citet{song2019sliced}.

\begin{table}[ht]
\caption{NLL comparison of implicit VAE training objectives across different latent dimensions averaged over 5 runs. The expressions $\sigma(\x{x})$, $\exp(-\abs{\x{x}})$, and $\exp(-\x{x}^{2})$ are evaluated element-wise.}
\label{tab:mnist_implicit_vae_results_multi_runs_averaged}
\begin{center}
    \resizebox{0.99\columnwidth}{!}{
    \begin{sc}
        \begin{tabular}{@{}lcccc@{}}
            \toprule
            & \multicolumn{1}{c}{SSM-VR} & \multicolumn{3}{c}{H-GSM} \\
            \cmidrule(lr){2-2} \cmidrule(lr){3-5}
            Latent Dimension
            & -
            & $D_{1}(\x{x}) = \text{diag}(\sigma(\x{x}))$
            & $D_{2}(\x{x}) = \text{diag}(\exp(-\abs{\x{x}}))$
            & $D_{3}(\x{x}) = \text{diag}(\exp(-\x{x}^{2}))$ \\
            \midrule
            $8$  & $96.17 \pm 0.10$ & $96.72 \pm 0.39$ & $99.61 \pm 1.95$ & $130.02 \pm 6.05$ \\
            $32$ & $89.31 \pm 0.12$ & $89.81 \pm 0.20$ & $90.37 \pm 0.65$ & $95.91 \pm 2.34$ \\
            \bottomrule
        \end{tabular}
    \end{sc}
    }
\end{center}
\end{table}

For implicit VAEs, we follow the training methodology of \citet{song2019sliced}. We replace the SSM-VR objective used to train the score network with the GSM objective given by \autoref{eq:gsm-obj-in-hutchison}, which we refer to as Hutchinson-GSM (H-GSM). The model architecture, hyperparameter choices, and training procedure are identical to those followed by \citet{song2019sliced}. In the implementation, the prior of the latent is chosen to be a standard Gaussian supported on $\R^{d}$. We work with the GSM objective derived in \autoref{theorem:sm-convex} and explore the effect of different choices of $D$. Specifically, we consider three choices of $D$ with $D_{1}(\x{x}) = \text{diag}(\sigma(\x{x}))$, $D_{2}(\x{x}) = \text{diag}(\exp(-\abs{\x{x}}))$, and $D_3(\x{x}) = \text{diag}(\exp(-\x{x}^{2}))$, where $\sigma(\x{x})$ denotes the element-wise sigmoid function, and $|\x{x}|, \x{x}^{2}$ and $\exp(\x{x})$ denote element-wise absolute value, square and exponential operations, respectively.

\begin{table}[ht]
    \caption{FID comparison of implicit VAE training objectives on CelebA dataset at different training checkpoints averaged over 5 runs. The latent dimension is fixed at $32$. The expressions $\sigma(\x{x})$, $\exp(-\abs{\x{x}})$, and $\exp(-\x{x}^{2})$ are evaluated element-wise.}
    \label{tab:celeba_implicit_vae_fid_results_multi_runs_averaged}
    \begin{center}
        \resizebox{0.99\columnwidth}{!}{
        \begin{sc}
            \begin{tabular}{@{}lcccc@{}}
                \toprule
                & \multicolumn{1}{c}{SSM-VR} & \multicolumn{3}{c}{H-GSM} \\
                \cmidrule(lr){2-2} \cmidrule(lr){3-5}
                Iteration
                & -
                & $D_{1}(\x{x}) = \text{diag}(\sigma(\x{x}))$
                & $D_{2}(\x{x}) = \text{diag}(\exp(-\abs{\x{x}}))$
                & $D_{3}(\x{x}) = \text{diag}(\exp(-\x{x}^{2}))$ \\
                \midrule
                $10$k  & $101.07 \pm 1.54$ & $140.41 \pm 10.13$ & $122.01 \pm 6.03$ & $121.51 \pm 2.43$ \\
                $20$k  & $83.10 \pm 1.10$ & $119.66 \pm 11.49$ & $101.79 \pm 5.88$ & $99.75 \pm 5.68$ \\
                $30$k  & $76.34 \pm 1.45$ & $106.44 \pm 10.34$ & $87.57 \pm 5.91$ & $86.45 \pm 4.99$ \\
                $40$k  & $72.12 \pm 1.38$ & $100.20 \pm 11.57$ & $80.66 \pm 5.96$ & $79.82 \pm 3.94$ \\
                $50$k  & $68.89 \pm 0.69$ & $93.31 \pm 11.75$ & $75.77 \pm 6.27$ & $74.31 \pm 1.50$ \\
                $60$k  & $67.80 \pm 1.70$ & $87.77 \pm 10.27$ & $73.22 \pm 5.90$ & $70.90 \pm 2.14$ \\
                $70$k  & $66.94 \pm 2.06$ & $85.85 \pm 10.72$ & $71.21 \pm 5.74$ & $68.87 \pm 1.71$ \\
                $80$k  & $65.70 \pm 1.97$ & $82.20 \pm 9.48$ & $69.48 \pm 6.55$ & $66.97 \pm 1.13$ \\
                $90$k  & $64.49 \pm 1.48$ & $79.80 \pm 7.57$ & $67.29 \pm 4.67$ & $65.37 \pm 1.31$ \\
                $100$k & $64.10 \pm 1.82$ & $78.11 \pm 7.15$ & $66.60 \pm 5.06$ & $64.15 \pm 0.53$ \\
                \bottomrule
            \end{tabular}
        \end{sc}
        }
    \end{center}
\end{table}

We compare implicit VAEs trained using SSM-VR and the H-GSM objective on MNIST and CelebA datasets. Following the evaluation protocol of \citet{song2019sliced}, we report the negative log-likelihood (NLL) for MNIST and Fr\'echet Inception Distance (FID) for CelebA over various choices of $D$. The results for MNIST are summarised in \autoref{tab:mnist_implicit_vae_results_multi_runs_averaged}. As observed, H-GSM performs competitively with SSM-VR, in particular NLL for $D_{1}(\x{x}) = \text{diag}(\sigma(\x{x}))$ with latent dimension $8$ and $32$. In both cases, the H-GSM objective exhibits higher variance across runs than SSM-VR.

The results for CelebA are summarised in \autoref{tab:celeba_implicit_vae_fid_results_multi_runs_averaged}. We observe that H-GSM with $D_{3}(\x{x}) = \text{diag}(\exp(-\x{x}^{2}))$ performs competitively with the SSM-VR baseline, with its FID approaching that of SSM-VR as training progresses. The $D_{2}(\x{x}) = \text{diag}(\exp(-\abs{\x{x}}))$ variant is closer to the baseline SSM-VR than $D_{1}(\x{x}) = \text{diag}(\sigma(\x{x}))$. In terms of variability across runs, $D_{3}$ exhibits a variance comparable to or lower than that of SSM-VR at later checkpoints, whereas $D_{1}$ and $D_{2}$ generally exhibit higher variance. These results further highlight the sensitivity of H-GSM performance to the choice of the diagonal matrix $D(\x{x})$.
\begin{remark}
    The SSM-VR results reported in \autoref{tab:mnist_implicit_vae_results_multi_runs_averaged} and \autoref{tab:celeba_implicit_vae_fid_results_multi_runs_averaged} were obtained by re-running the original implementation provided by \citet{song2019sliced} under the same experimental settings.
\end{remark}

\section{Implementation Details} \label{appendix:implementation-details}
This section describes how the parameter estimation experiments of Section \ref{sec:experiments} and Appendix~\ref{appendix:more-experiments} were run. The implicit VAE follows the setup of \citet{song2019sliced} and is described in Appendix~\ref{appendix:subsec-implicit-vae}. We first describe the computation shared by all estimators and the weights we compare, then the two main models with their ground truths and samplers, the error metrics, and finally the settings of each ablation. \autoref{tab:implementation-settings} summarises the settings. The proposed estimators use $D_{\phi} = H_{\phi}^{-1}$ throughout, that is, $c = (\det H_{\phi})^{1/2}$ in \autoref{prop:midpoint-connectivity-checks-main-theorem}.

\begin{table}
    \centering
    \caption{Settings of the parameter estimation experiments. Every trial draws a new dataset, and all methods of an experiment are fitted to the same dataset within a trial. For Gibbs sampler, a burn-in of $2000$ sweeps was used for all the experiments.}
    \label{tab:implementation-settings}
    \resizebox{\textwidth}{!}{
        \begin{tabular}{@{}llllll@{}}
            \toprule
            Experiment & Domain & Ground truth & Sampler & $N$ & Trials \\
            \midrule
            Truncated Gaussian (\autoref{tab:trunc-gm-polytope-10d-results}) & $\mathcal{S}^{10}$ & $\x{\mu}^{*} = (10 + \sqrt{10})^{-1}\x{1}$, $K^{*} = 60\,\id + 20\,\x{1}\x{1}^{\top}$ & Gibbs, thinning $3$ & $200, 500, 800$ & $50$ \\
            Dirichlet (\autoref{tab:dirichlet-simplex-10d-results}) & $\Delta^{9}$ & $\x{\alpha}^{*} = (1, 2, \dots, 10)$ & exact & $200, 500, 800$ & $50$ \\
            Robustness, truncated Gaussian (\autoref{tab:trunc-gm-polytope-10d-ablation-results}) & $\mathcal{S}^{10}$ & $50$ random $(\x{\mu}, K)$ & Gibbs, thinning $5$ & $800$ & $50$ per configuration \\
            Robustness, Dirichlet (\autoref{tab:dirichlet-simplex-10d-ablation-results}) & $\Delta^{9}$ & $50$ random $\x{\alpha}$ & exact & $800$ & $50$ per configuration \\
            Power barrier sweep & $\mathcal{S}^{10}$ & as in \autoref{tab:trunc-gm-polytope-10d-results} & Gibbs, thinning $3$ & $800$ & $50$ \\
            Data-driven ranking (\autoref{tab:candidate-ranking-diagnostics}) & $\mathcal{S}^{10}$, $\Delta^{9}$ & as in the main experiments & as in the main experiments & $800$ & $50$ \\
            Non-convex support & ball $\cup$ shell in $\R^{3}$ & $\x{\mu}^{*} = \x{0}$, $K^{*} = \id_{3}$ & rejection & $50, 100, 200$ & $50$ \\
            Neumann, $d = 1$ (\autoref{tab:neumann-boundary-ablation}) & $(0, 1)$ & $\theta_{0} = 2$ & rejection & $200, 500, 800$ & $50$ \\
            Neumann, $d = 2$ (\autoref{tab:neumann-boundary-ablation}) & $(0, 1)^{2}$ & $\ttheta_{0} = (1.2, -0.8, 1.0)$ & rejection & $200, 500, 800$ & $50$ \\
            \bottomrule
        \end{tabular}
    }
\end{table}

\subsection{Computing the Estimator} \label{appendix:impl-estimator}
All models we fit are exponential families of the form \autoref{eq:exp-fam} with $\beta \equiv 0$ on $\Omega$. For every method, we compute $\Gamma_{N}$ and $\x{g}_{N}$, as in \autoref{theorem:convexity}(a). The estimate $\hat{\ttheta}_{N} = -\Gamma_{N}^{-1}\x{g}_{N}$ is obtained by solving this $r \times r$ linear system with an LU factorisation with no regularisation. The weight $D$ is the only choice that distinguishes one method from another. If the solver reports $\Gamma_{N}$ as exactly singular, the code falls back to the Moore--Penrose pseudo-inverse. The baselines are fitted by the same closed-form solution, since their objectives are also weighted Fisher divergences and hence quadratic in $\ttheta$. Within a trial, all methods are fitted to the same sample.

\subsection{Weights and Their Divergences} \label{appendix:impl-weights}

\paragraph{Barrier weights on polytopes.} Let $\Omega = \{\xx : \x{a}_{k}^{\top}\xx < b_{k},\ 1 \le k \le m\}$, let $A \in \R^{m \times d}$ have rows $\x{a}_{k}^{\top}$, and write $s_{k}(\xx) = b_{k} - \x{a}_{k}^{\top}\xx$. Every barrier we use has the form
\begin{align*}
    \phi(\xx) = \sum_{k=1}^{m}\omega(s_{k}(\xx))
\end{align*}
with $\omega(s) = s\log s$ for the entropic barrier $\phi_{1}$, $\omega(s) = -\log s$ for the logarithmic barrier $\phi_{2}$, and $\omega(s) = s^{p}/(p(p-1))$ for the power barrier $\phi(\cdot\,; p)$ of Appendix~\ref{appendix:ablation-subsec-power-barrier}. $\mathcal{S}^{d}$, and $\mathcal{S}^{d-1}$ in the reduced Dirichlet coordinates, correspond to $A = [-\mathbb{I}_{d};\ \x{1}^{\top}]$ and $\x{b} = (0, \dots, 0, 1)$ with $m = d + 1$, where $\x{1}$ is the all-ones vector; this gives the expressions for $\phi_{1}$ and $\phi_{2}$ in the captions of \autoref{tab:trunc-gm-polytope-10d-results} and \autoref{tab:dirichlet-simplex-10d-results}. In the Dirichlet experiments, we add $10^{-8}$ to every slack before evaluating $H_{\phi}$, as a safeguard against draws that lie within rounding error of a facet.

\paragraph{Truncated SM.} Following \citet{truncated_sm}, $D(\xx) = g_{0}(\xx)\id$, where $g_{0}(\xx)$ is the Euclidean distance from $\xx$ to $\partial\Omega$. On a polytope, $g_{0}(\xx) = \min_{k}s_{k}(\xx)/\norm{\x{a}_{k}}$ and $\nabla \cdot D = \nabla g_{0} = -\x{a}_{k^{*}}/\norm{\x{a}_{k^{*}}}$, where $k^{*}$ attains the minimum, which is unique for almost every $\xx$. The distance is used as is, without capping or smoothing.

\paragraph{GSM-GD and GSM-CD.} Following \citet{gsm_general_domain}, the weight is
\begin{align*}
    D(\xx) = \diag\big(h_{1}(\varphi_{1}(\xx)), \dots, h_{d}(\varphi_{d}(\xx))\big),
\end{align*}
where $\varphi_{j}(\xx)$ is the distance from $\xx$ to $\partial\Omega$ along the $j$th coordinate axis, so $h$ acts on $\varphi_{j}(\xx)$ and not on $x_{j}$. For $\mathcal{S}^{d}$, we have
\begin{align*}
    \varphi_{j}(\xx) = \min(x_{j}, 1 - \x{1}^{\top}\xx) \quad \forall \: j
\end{align*}
For the Dirichlet in reduced coordinates, $1 - \x{1}^{\top}\y = x_{d}$ is the dropped component and $\varphi_{j}(\y) = \min(y_{j}, x_{d})$, which is the weight of GSM-CD \citep{gsm-on-compositional-data}. For $h_{1}(u) = u$, $h_{2}(u) = u^{2}$ and $h_{3}(u) = \sqrt{u}$, the divergence is $(\nabla \cdot D)_{j} = h'(x_{j})$ if $x_{j} < 1 - \x{1}^{\top}\xx$ and $-h'(1 - \x{1}^{\top}\xx)$ otherwise. We use $h$ without truncation.

\subsection{Truncated Gaussian on \texorpdfstring{$\mathcal{S}^{10}$}{the simplex polytope}} \label{appendix:impl-truncated-gaussian}
\paragraph{Parametrisation.} We write $\log p_{\x{\mu}, K}(\xx) = -\frac{1}{2}\xx^{\top}K\xx + \x{\eta}^{\top}\xx + \text{const}$ with $\x{\eta} = K\x{\mu}$, so that
\begin{align*}
    \x{t}(\xx) &= \big(x_{1}, \dots, x_{d},\ -\tfrac{1}{2}x_{1}^{2}, \dots, -\tfrac{1}{2}x_{d}^{2},\ (-x_{i}x_{j})_{i < j}\big), \\
    \ttheta &= \big(\eta_{1}, \dots, \eta_{d},\ K_{11}, \dots, K_{dd},\ (K_{ij})_{i < j}\big),
\end{align*}
with $r = 2d + d(d - 1)/2 = 65$ for $d = 10$.



\paragraph{Single Ground truth.} We take $\x{\mu}^{*} = (d + \sqrt{d})^{-1}\x{1}$ and $K^{*} = 8d\big(0.75\id + 0.25\x{1}\x{1}^{\top}\big)$. For $d=10$, the point $\x{\mu}^{*}$ lies at distance $(d + \sqrt{d})^{-1} \approx 0.076$ from every facet and condition number of $K^{*}$ is $\approx 4.33$. For calibration, $\norm{\x{\mu}^{*}}^{2} = 0.0577$ and $\norm{K^{*}}^{2} = 10^{5}$. The median errors of $\phi_{1}$ in \autoref{tab:trunc-gm-polytope-10d-results}, $0.0045$ for $\x{\mu}$ and $1.65 \times 10^{4}$ for $K$, are therefore relative squared errors of $0.078$ and $0.165$.

\paragraph{Sampler.} We use a Gibbs sampler that updates one coordinate at a time. Given the other coordinates, $x_{j}$ follows $\N(\nu_{j}, \sigma_{j}^{2})$ truncated to an interval $(l_{j}, u_{j})$, where
\begin{align*}
    \nu_{j} &= \frac{1}{K_{jj}}\Big((K\x{\mu})_{j} - \sum_{i \ne j}K_{ji}x_{i}\Big), \qquad \sigma_{j}^{2} = \frac{1}{K_{jj}}, \\
    (l_{j}, u_{j}) &= \Big\{z \in \R : a_{kj}z < b_{k} - \sum_{i \ne j}a_{ki}x_{i},\ 1 \le k \le m\Big\},
\end{align*}
which on $\mathcal{S}^{d}$ is $(0, 1 - \sum_{i \ne j}x_{i})$. We use inverse CDF method to sample from $x_{j}$. A sweep updates $x_{1}, \dots, x_{d}$ in turn. Each dataset of size $N$ pools four independent chains, each started at the incentre $\x{\mu}^{*}$ and seeded from the global generator. After a burn-in of $2000$ sweeps, each chain keeps every third state (thinning $3$) until it has $N/4$ states, and we check that every retained state lies in $\Omega$. We use $N \in \{200, 500, 800\}$ and $50$ trials for each $N$. Every trial draws a new dataset with new chain seeds, and all six estimators are fitted to it.

\subsection{Dirichlet on \texorpdfstring{$\Delta^{9}$}{the 9-simplex}} \label{appendix:impl-dirichlet}
In the reduced coordinates $\y \in \mathcal{S}^{d-1}$ of Section~\ref{appendix:subsec-dirichlet-simplex-exp},
\begin{align*}
    \x{t}(\y) = \big(\log y_{1}, \dots, \log y_{d-1}, \log(1 - \x{1}^{\top}\y)\big) \quad \text{and} \quad \ttheta = \x{\alpha} - \x{1},
\end{align*}
so $r = d$. We report $\hat{\x{\alpha}} = \x{1} + \hat{\ttheta}_{N}$ without enforcing positivity; in the main experiment, every component of $\hat{\x{\alpha}}$ was positive for every method and trial. The ground truth is $\x{\alpha}^{*} = (1, 2, \dots, 10)$, so $\norm{\x{\alpha}^{*}}^{2} = 385$, and the median error of $\phi_{2}$ in \autoref{tab:dirichlet-simplex-10d-results}, $0.153$, is a relative squared error of $4 \times 10^{-4}$. Samples are i.i.d.\ draws from $\mathrm{Dir}(\x{\alpha}^{*})$, generated by NumPy. The last coordinate is dropped and enters only through $x_{d} = 1 - \x{1}^{\top}\y$. We use $N \in \{200, 500, 800\}$ and $50$ trials for each $N$, with all estimators fitted to the same draw.

\subsection{Error Metrics and Summary Statistics} \label{appendix:impl-metrics}
In each trial, we record the squared errors $\norm{\hat{\x{\mu}} - \x{\mu}^{*}}^{2}$ and $\norm{\hat{K} - K^{*}}^{2}$ (the squared Frobenius norm) for the truncated Gaussian, $\norm{\hat{\x{\alpha}} - \x{\alpha}^{*}}^{2}$ for the Dirichlet, and $\norm{\hat{\ttheta} - \ttheta_{0}}^{2}$ for the Neumann boundary experiment. They are not normalised by the size of the true parameter. In all tables and figures, MSE refers to these per-trial squared errors. The Mean column is their average over the $50$ trials, which estimates the mean squared error; the Median and Std.\ columns are their median and standard deviation over the trials. The box plots show their quartiles, with whiskers reaching the most extreme trial within $1.5$ interquartile ranges of the box and the remaining trials drawn as points.



\subsection{Robustness Across Ground Truths} \label{appendix:impl-robustness}
The $50$ configurations of each model are drawn once, before any data, and are saved with the results. For the truncated Gaussian on $\mathcal{S}^{10}$, the location $\x{\mu}$ is drawn in four groups.
\begin{enumerate}[label=(\roman*)]
    \item $10$ configurations on the facet $\{\x{1}^{\top}\xx = 1\}$: $\x{\mu} = \x{\gamma}/\x{1}^{\top}\x{\gamma}$, where $\gamma_{1}, \dots, \gamma_{d}$ are i.i.d.\ $\mathrm{Gamma}(a, 1)$ and $a \sim \mathrm{Unif}(0.5, 2.5)$ is drawn once per configuration.
    \item $15$ configurations with $\x{\mu}$ uniform on $\mathcal{S}^{d}$: $\x{\mu}$ consists of the first $d$ components of $\x{\gamma}/\x{1}^{\top}\x{\gamma}$, where $\x{\gamma} \in \R^{d+1}$ has i.i.d.\ $\mathrm{Gamma}(1, 1)$ components.
    \item $15$ configurations as in (ii), but with $\gamma_{k} \sim \mathrm{Gamma}(a_{k}, 1)$ and $a_{1}, \dots, a_{d+1}$ i.i.d.\ $\mathrm{Unif}(0.5, 3)$, which spreads $\x{\mu}$ unevenly across the coordinates.
    \item $10$ configurations close to the vertex at the origin, with $\mu_{1}, \dots, \mu_{d}$ i.i.d.\ $\mathrm{Unif}(0.01, 0.05)$.
\end{enumerate}
In every configuration, $K = Q\diag(\x{\lambda})Q^{\top}$, where $Q$ is the orthogonal factor of the QR decomposition of a $d \times d$ matrix with i.i.d.\ standard normal entries, and $\lambda_{1}, \dots, \lambda_{d}$ are i.i.d.\ $\mathrm{Unif}(\lambda_{\min}, \kappa\lambda_{\min})$ with $\lambda_{\min} \sim \mathrm{Unif}(30, 60)$ and $\kappa \sim \mathrm{Unif}(3, 7)$, so the condition number of $K$ is at most $7$. The sampler is that of Appendix~\ref{appendix:impl-truncated-gaussian}, still started at the incentre of $\mathcal{S}^{10}$, with every fifth instead of every third state kept.

For the Dirichlet, the $50$ configurations form five groups of $10$.
\begin{enumerate}[label=(\roman*)]
    \item All $\alpha_{i}$ i.i.d.\ $\mathrm{Unif}(0.5, 0.9)$, so the density is unbounded near every facet.
    \item $\alpha_{1}, \dots, \alpha_{d-1}$ i.i.d.\ $\mathrm{Unif}(1.5, 5)$ and $\alpha_{d} \sim \mathrm{Unif}(0.5, 0.9)$, so only the dropped component has a concentration below $1$.
    \item The reverse: $\alpha_{1}, \dots, \alpha_{d-1}$ i.i.d.\ $\mathrm{Unif}(0.5, 0.9)$ and $\alpha_{d} \sim \mathrm{Unif}(1.5, 5)$.
    \item All $\alpha_{i}$ i.i.d.\ $\mathrm{Unif}(1.5, 5)$.
    \item All $\alpha_{i}$ i.i.d.\ $\mathrm{Unif}(6, 12)$, which concentrates the distribution away from the boundary.
\end{enumerate}
Each configuration uses $N = 800$ and $50$ trials. For each configuration and method, we take the median of the per-trial squared errors over the $50$ trials and rank the methods by it, with rank $1$ for the lowest. We report the fraction of configurations in which a method has rank $1$ (the win rate), its average rank, and the median over configurations of its per-configuration medians. For the truncated Gaussian, the ranks for $\x{\mu}$ and $K$ are computed separately.

\subsection{Power Barrier Sweep} \label{appendix:impl-power}
The sweep of Appendix~\ref{appendix:ablation-subsec-power-barrier} uses the ground truth and sampler of Appendix~\ref{appendix:impl-truncated-gaussian} with $N = 800$ and $50$ trials. In each trial, one dataset is drawn, and the nine power barriers $\phi(\cdot\,; p)$, $p \in \{0.2, 0.5, 0.8, 1.2, 4/3, 1.4, 1.5, 1.6, 1.8\}$, and the four baselines are all fitted to it. The sweep is a separate run from that of \autoref{tab:trunc-gm-polytope-10d-results}, and its datasets are drawn independently of those behind the $N = 800$ results there. The baseline values in the two places are therefore independent replicates of the same experiment and differ by Monte Carlo error; for example, the median $\text{MSE}_{K}$ of GSM-GD($\x{h}_{1}$) is $1.64 \times 10^{4}$ in the sweep and $1.98 \times 10^{4}$ in \autoref{tab:trunc-gm-polytope-10d-results}. Comparisons between methods should therefore be made within one run, where all methods see the same data.

\subsection{Data-Driven Ranking} \label{appendix:impl-ranking}
The ranking study of Appendix~\ref{appendix:subsec-data-driven-ranking} uses the ground truths and samplers of the main experiments with $N = 800$ and $50$ trials, and the candidates $\phi_{1}$, $\phi_{2}$ and $\phi(\cdot\,; p)$ for $p \in \{0.2, 0.8, 1.2, 1.8\}$, all fitted to the same dataset within a trial. For each candidate, we compute $\hat{\Sigma}_{N}$ of \autoref{theorem:consistency}(e) at $\hat{\ttheta}_{N}$ and score the candidate either by $\tr(\hat{\Upsilon})$, with $\hat{\Upsilon} = \Gamma_{N}^{-1}\hat{\Sigma}_{N}\Gamma_{N}^{-1}$, or by the spectral condition number $\kappa(\Gamma_{N})$. The trace is taken over the full natural parameter, $\ttheta = (\x{\eta}, K)$ for the truncated Gaussian and $\ttheta = \x{\alpha} - \x{1}$ for the Dirichlet, so the same score ranks the candidates for $\x{\mu}$ and for $K$. In each trial, the selected candidate minimises the score and the best candidate minimises the squared error, and the Top-1 accuracy is the fraction of trials in which the two coincide. The regret of a trial is $\left[\text{MSE}(\text{selected}) - \text{MSE}(\text{best})\right] / \text{MSE}(\text{best})$. The Spearman and Kendall correlations between the six scores and the six errors are computed within each trial and averaged over the trials.

\subsection{Non-Convex Support} \label{appendix:impl-nonconvex}
For Appendix~\ref{appendix:ablation-subsec-non-convex-support}, we draw from $\N(\x{0}, \id_{3})$ and keep the draws that lie in $\Omega$, which is exact rejection sampling. Truncated SM \cite{truncated_sm} uses the distance to the nearest of the three spheres, $g_{0}(\xx) = \min_{r \in \{a, b, c\}}\abs{r - \norm{\xx}}$, whose gradient is $-\operatorname{sign}(r^{*} - \norm{\xx})\,\xx/\norm{\xx}$ for the minimising radius $r^{*}$.

\subsection{Neumann Models} \label{appendix:impl-neumann}
Both Neumann models are sampled exactly by rejection from the uniform distribution on $\Omega$. In $d = 1$, every score matching estimator has the closed form. In $d = 2$, we use the estimator of Appendix~\ref{appendix:impl-estimator} with the polytope barriers for $(0, 1)^{2}$, given by $A = [-\id_{2};\ \id_{2}]$ and $\x{b} = (0, 0, 1, 1)$, and with $D = \id$ for SM. In both dimensions, coordinates are clipped to $[10^{-9}, 1 - 10^{-9}]$ before the barrier weights are evaluated. The MLE in $d = 1$ uses bisection method on $[-60, 60]$. In $d = 2$, the MLE runs gradient ascent on the average log-likelihood, from $\ttheta = \x{0}$ with unit step size, and stops when the norm of the gradient on the average log-likelihood falls below $10^{-5}$; this took between $76$ and $164$ iterations in every trial. In both cases, the expectations are computed by the trapezoidal rule.


\end{document}